\documentclass{article}

    \usepackage[final]{neurips_2022}

\usepackage[utf8]{inputenc} 
\usepackage[T1]{fontenc}    
\usepackage{hyperref}       
\usepackage{bookmark}
\usepackage{url}            
\usepackage{booktabs}       
\usepackage{amsfonts}       
\usepackage{nicefrac}       
\usepackage{microtype}      
\usepackage{xcolor}         
\usepackage{subfigure}
\usepackage{graphicx}
\usepackage{wrapfig}
\usepackage{varwidth}
\usepackage{mathtools}
\usepackage{adjustbox}
\usepackage{enumitem}
\usepackage{mdwlist}
\usepackage{physics}
\usepackage{framed}
\usepackage{dsfont}
\usepackage{tikz}
\usepackage{pgfplots}
\pgfplotsset{compat=1.17}
\usepackage{multirow}
\usepackage{mdwtab}
\usepackage{caption}
\usepackage{algorithmic}
\usepackage[ruled,vlined,noend]{algorithm2e}

\usepackage{wrapfig}
\usetikzlibrary{fit,positioning,shapes,backgrounds,arrows.meta,positioning,calc}
\usepackage{amsthm}
\newtheorem*{theorem*}{Theorem}
\newtheorem{theorem}{Theorem}
\newtheorem{lemma}{Lemma}
\newtheorem{corollary}{Corollary}
\newtheorem*{corollary*}{Corollary}

\newcommand{\cS}{{\mathcal{S}}}
\newcommand{\cA}{\mathcal{A}}
\newcommand{\cD}{\mathcal{D}}

\newcommand{\cT}{\mathcal{T}}

\newcommand{\RR}{\mathbb{R}}

\title{RORL: Robust Offline Reinforcement Learning via Conservative Smoothing}

\author{
Rui Yang$^1$\thanks{Equal Contribution}, Chenjia Bai$^{2*}$, Xiaoteng Ma$^{3}$,
\textbf{Zhaoran Wang$^{4}$, Chongjie Zhang$^{3}$, Lei Han$^{5}$\thanks{Corresponding Author}}\\
$^{1}$Hong Kong University of Science and Technology, $^{2}$Shanghai AI Laboratory \\
$^{3}$Tsinghua University, $^{4}$Northwestern University, $^{5}$Tencent Robotics X\\
\texttt{ryangam@connect.ust.hk, baichenjia@pjlab.org.cn} \\
\texttt{ma-xt17@mails.tsinghua.edu.cn, zhaoranwang@gmail.com}\\
\texttt{chongjie@tsinghua.edu.cn, leihan.cs@gmail.com}
}

\begin{document}

\maketitle

\begin{abstract}
  Offline reinforcement learning (RL) provides a promising direction to exploit massive amount of offline data for complex decision-making tasks. Due to the distribution shift issue, current offline RL algorithms are generally designed to be conservative in value estimation and action selection. However, such conservatism can impair the robustness of learned policies when encountering observation deviation under realistic conditions, such as sensor errors and adversarial attacks. To trade off robustness and conservatism, we propose Robust Offline Reinforcement Learning (RORL) with a novel conservative smoothing technique. In RORL, we explicitly introduce regularization on the policy and the value function for states near the dataset, as well as additional conservative value estimation on these states. Theoretically, we show RORL enjoys a tighter suboptimality bound than recent theoretical results in linear MDPs. We demonstrate that RORL can achieve state-of-the-art performance on the general offline RL benchmark and is considerably robust to adversarial observation perturbations.
\end{abstract}

\section{Introduction}
\label{sec:intro}
Over the past few years, deep reinforcement learning (RL) has been a vital tool for various decision-making tasks~\cite{mnih2015human,silver2016mastering,schrittwieser2020mastering,ecoffet2021first} in a trial-and-error manner. A major limitation of current deep RL algorithms is that they require intense online interactions with the environment~\cite{levine2020offline,yang2021exploration}. These data collecting processes can be costly and even prohibitive in many real-world scenarios such as robotics and health care \cite{levine2020offline,sun2021safe}. Offline RL~\cite{fujimoto2019off,kumar2019stabilizing} is gaining more attention recently since it offers probabilities to learn reinforced decision-making strategies from fully offline datasets.

The main challenge of offline RL is the distribution shift between the offline dataset and the learned policy, which would lead to severe overestimation for the out-of-distribution (OOD) actions \cite{fujimoto2019off,kumar2019stabilizing}. To overcome such an issue, a series of model-free offline RL works~\cite{wang2018exponentially,fujimoto2019off,yang2021believe,kumar2020conservative,li2021focal,an2021uncertainty,yang2022rethinking,bai2021pessimistic} propose to celebrate conservatism, such as constraining the learned policy close to the supported distribution or penalizing the $Q$-values of OOD actions. 
Besides, another stream of works builds upon model-based algorithms \cite{yu2020mopo,yu2021combo,wang2021offline}, which leverages the ensemble dynamics models to enforce pessimism through uncertainty penalizing or data generation.

However, conservatism is not the only concern when applying offline RL to the real world. Due to the sensor errors and model mismatch, the robustness of offline RL is also crucial under the realistic engineering conditions, which has not been well studied yet. In online RL, a series of works has been studied to learn the optimal policy under worst-case perturbations of the observation \cite{zhang2020robust,pattanaik2017robust,huang2017adversarial} or environmental dynamics \cite{vinitsky2020robust,pinto2017robust,rajeswaran2016epopt,bai2021dynamic}. Yet, it is non-trivial to apply online robust RL techniques into the offline problems. The main challenge is that the perturbation of states may bring OOD observation and extra overestimation for the value function. New techniques are needed to tackle the conservatism and robustness simultaneously in the offline RL.

This paper studies robust offline RL against adversarial observation perturbations, where the agent needs to learn the policy conservatively while handling the potential OOD observation with perturbation. We first demonstrate that current value-based offline RL algorithms lack the necessary smoothness for the policy, which is visualized in Figure~\ref{fig:schematic_diagram}. As an illustration, we show that a famous baseline method CQL~\cite{kumar2020conservative} learns a non-smooth value function, leading to significant performance degradation for even a tiny scale perturbation on observation (see Section~\ref{sec:Motivating_Example} for details). In addition, simply adopting the smoothing technique for existing methods may result in extra overestimation at the boundary of supported distribution and lead the agent toward unsafe areas.

To this end, we propose Robust Offline Reinforcement Learning (RORL) with a novel conservative smoothing technique, which explicitly handles the overestimation of OOD state-action pairs. Specifically, we explicitly introduce smooth regularization on both the value functions and policies for states near the dataset support and conservatively estimate the values of these OOD states based on pessimistic bootstrapping. Furthermore, we theoretically prove that RORL yields a valid uncertainty quantifier in linear MDPs and enjoys a tighter suboptimality bound than previous work~\cite{bai2021pessimistic}.

In our experiments \footnote{Our code is available at \href{https://github.com/YangRui2015/RORL}{https://github.com/YangRui2015/RORL}}, we demonstrate that RORL can achieve state-of-the-art (SOTA) performance in the D4RL benchmark~\cite{d4rl-2020} with fewer ensemble $Q$ networks than the current SOTA approach~\cite{an2021uncertainty}. The results of the benchmark experiments imply that robust training can lead to performance improvement in non-perturbed environments. Meanwhile, compared with current ensemble-based baselines, RORL is considerably more robust to adversarial perturbations on observations. We conduct the adversarial experiments under different attack types, showing consistently superior performance on several continuous control tasks.

\begin{figure*}[t]
\centering
\includegraphics[width=0.73\textwidth]{ 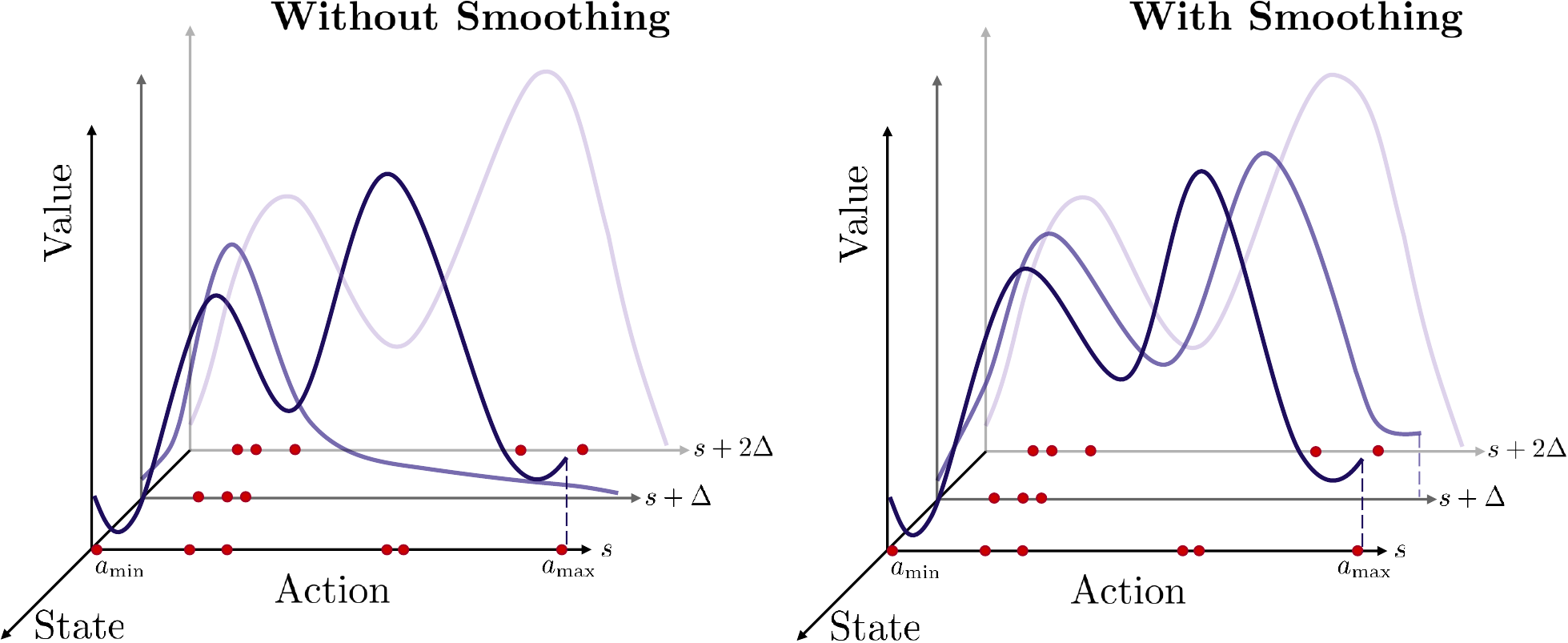}
\caption{A schematic diagram of smoothing in offline RL. The red spots represent the offline data samples. Without state smoothing, the value function would change drastically over neighboring states and induce an unstable policy. Yet, the smoothness may also lead to value overestimation of dangerous areas. RORL trades off smoothness and possible overestimation as discussed in Sec \ref{sec:method}.}
\vspace{-1em}
\label{fig:schematic_diagram}
\end{figure*}

\section{Preliminaries}

\paragraph{Offline RL}
Considering an episodic MDP $\mathcal{M}=(\mathcal{S},\mathcal{A},T,r,\gamma,\mathbb{P})$, where $\mathcal{S}$ is the state space, $\mathcal{A}$ is the action space, $T$ is the length of an episode, $r$ is the reward function, $\mathbb{P}$ is the dynamics, and $\gamma$ is the discount factor.
In offline RL, the objective of the agent is to find an optimal policy by sampling experiences from a fixed dataset $\mathcal{D}=\{(s_t^i,a_t^i,r_t^i,s_{t+1}^i)\}$. Nevertheless, directly applying off-policy algorithms in offline RL suffers from the distribution shift problem. In $Q$-learning, the value function evaluated on the greedy action $a'$ in Bellman operator $\mathcal{T}Q=r+\gamma\mathbb{E}_{s'}[\max_{a'} (s',a')]$ tends to have extrapolation error since $(s',a')$ has barely occurred in $\mathcal{D}$.

Pessimistic Bootstrapping for Offline RL (PBRL) \cite{bai2021pessimistic} is an uncertainty-based method that uses bootstrapped $Q$-functions for uncertainty quantification \cite{sun2022daux} and OOD sampling for regularization. Specifically, PBRL maintains $K$ bootstrapped $Q$ functions to quantify the epistemic uncertainty \cite{bai2021principled}
and performs pessimistic update
to penalize $Q$ functions with large uncertainties. The uncertainty is defined as the standard deviation among bootstrapped $Q$-functions. For each bootstrapped $Q$-function, the Bellman target is defined as $\widehat{\mathcal{T}} Q(s,a)=r(s,a)+\gamma \widehat{\mathbb{E}}_{s'\sim P(\cdot|s,a),a'\sim \pi(\cdot|s')}\big[Q(s',a')-\lambda u(s',a')\big]$.
Under linear MDP assumptions, this uncertainty is equivalent to the LCB penalty and is provably efficient \cite{jin2021pessimism}. Furthermore, PBRL incorporates OOD sampling by sampling OOD actions to form $(s, a^{\rm ood})$ pairs, where $a^{\rm ood}$ follows the learned policy. The detached learning target for $(s, a^{\rm ood})$ is $\widehat{\mathcal{T}}^{\rm ood} Q(s,a^{\rm ood}):=Q(s,a^{\rm ood})-\lambda u(s,a^{\rm ood})$, which introduces uncertainty penalization to enforce pessimistic $Q$-functions for OOD actions.

\paragraph{Smooth Regularized RL}
Robust RL aims to learn a robust policy against the adversarial perturbed environment in online RL. SR$^2$L \cite{shen2020deep} enforces smoothness in both the policy and $Q$-functions. Specifically, SR$^2$L encourages the outputs of the policy and value function to not change much when injecting small perturbations to the states. For state $s$, SR$^2$L constructs a perturbation set $\mathbb{B}_{d}(s,\epsilon)=\{\hat{s}:d(s,\hat{s})\leq \epsilon\}$ with a metric $d(·,·)$, which is chosen to be the $\ell_p$ distance, and introduces a smoothness regularizer for policy as
$\mathcal{R}^{\pi}_s=\mathbb{E}_{s\sim \rho^{\pi}} \max_{\hat{s}\in \mathbb{B}_{d}(s,\epsilon)} \mathcal{D}(\pi(\cdot|s)\|\pi(\cdot|\hat{s})),$
where $\mathcal{D}(\cdot\|\cdot)$ is a distance metric and the $\max$ operator gives an adversarial manner to choose $\hat{s}$. Similarly, the smoothness regularizer for the value function is defined as $\mathcal{R}^{V}_s=\mathbb{E}_{s\sim \rho^{\pi},a\sim \pi} \max_{\hat{s}\in \mathbb{B}_{d}(s,\epsilon)} (Q(s,a)-Q(\hat{s},a))^2$. SR$^2$L is shown to improve robustness against both random and adversarial perturbations.


\section{Robustness of Offline RL: A Motivating Example} \label{sec:Motivating_Example}

\begin{figure*}[t]
\centering
\subfigure[$Q$-function of CQL]{\includegraphics[width=0.33\textwidth]{ 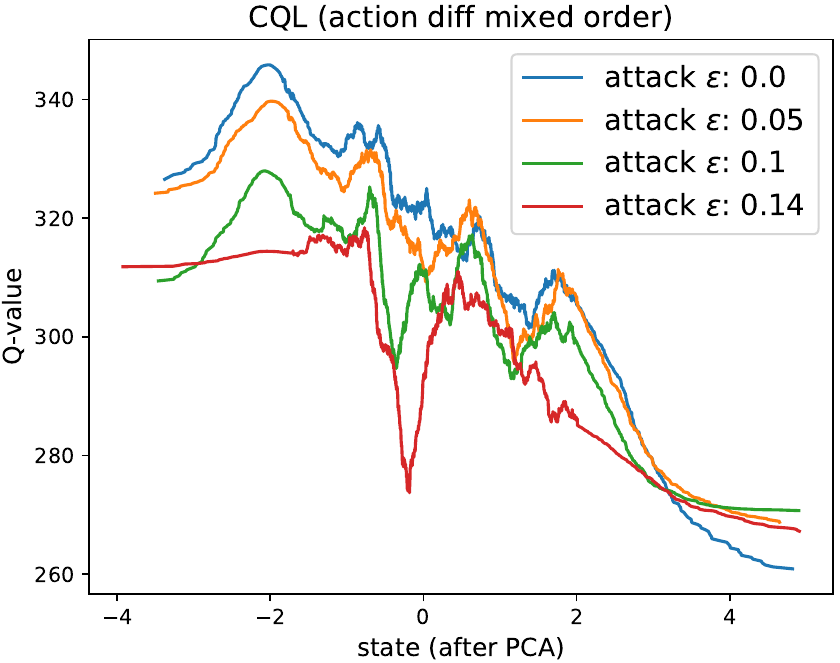}\label{fig:cql-non-smooth}}
\subfigure[$Q$-function of CQL-smooth]{\includegraphics[width=0.33\textwidth]{ 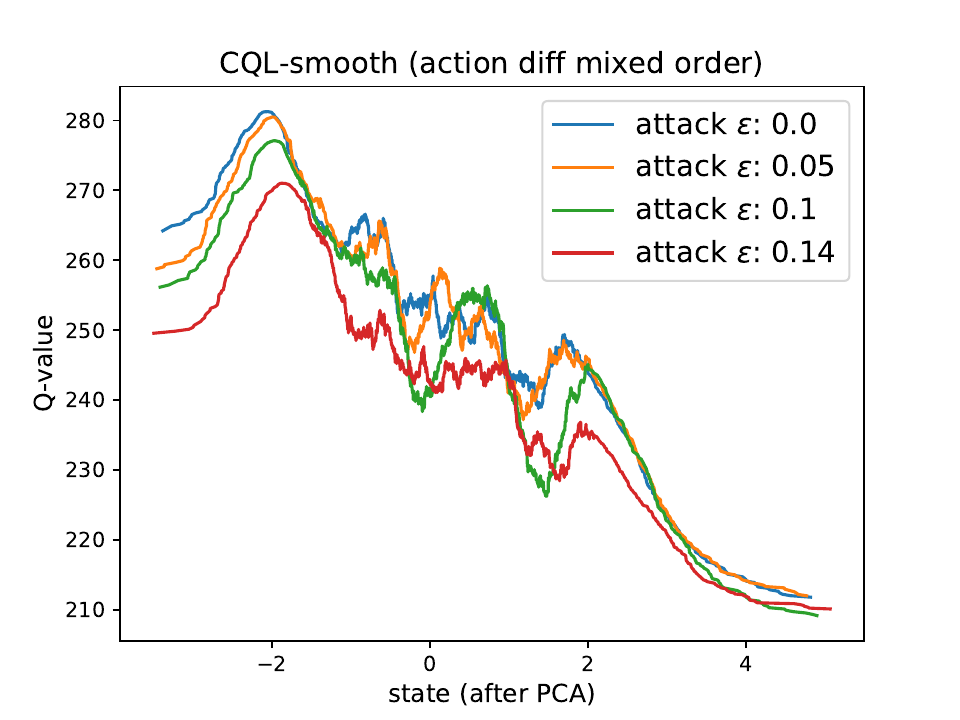}\label{fig:cql-smooth}}
\subfigure[Final performance]{\includegraphics[width=0.3\textwidth]{ 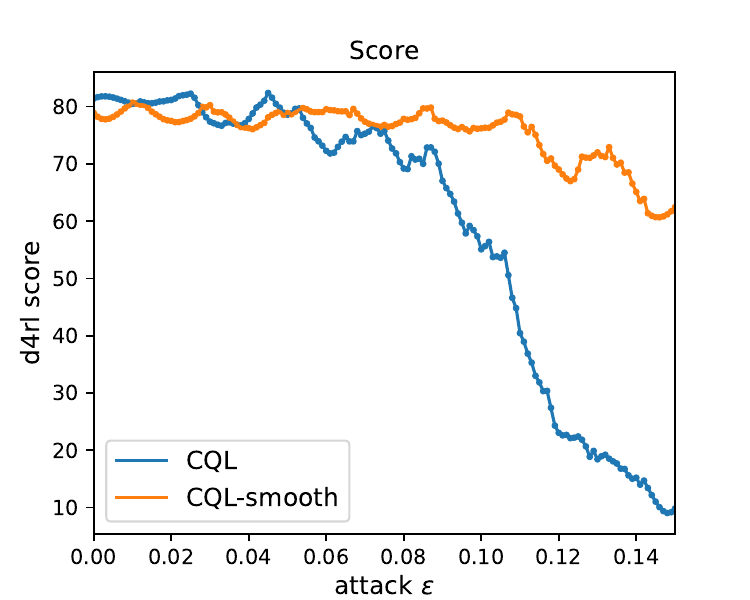}\label{fig:cql-score}}
\caption{(a) (b) The $Q$-functions of $\hat{s}$ with adversarial noises in CQL and CQL-smooth, respectively. The same moving average factor is used in plotting both figures. (c) The performance of CQL and CQL-smooth with different perturbation scales. We use 100 uniformly distributed $\epsilon\in[0.0,0.15]$ for the evaluation.}
\vspace{-1em}
\label{fig:cql-attack}
\end{figure*}

We give a motivating example to illustrate the robustness of the popular CQL \cite{kumar2020conservative} policies. We introduce an adversarial attack on state $s$ to obtain $\hat{s}=\arg\max_{\hat s\in \mathbb{B}_d(s,\epsilon)} D_{\rm J} (\pi_{\theta}(\cdot|s)\|\pi_{\theta}(\cdot|\hat s))$, where $\mathbb{B}_d(s,\epsilon)=\{\hat{s}:d(s,\hat{s})\leq \epsilon\}$ is the perturbation set and the metric $d(·,·)$ is chosen to be the $\ell_\infty$ norm. The Jeffrey’s divergence $D_{\rm J}$ for two distributions $P$, $Q$ is defined by:
$D_{\rm J}(P\|Q) = \frac{1}{2} [D_{\rm KL}(P\|Q) + D_{\rm KL}(Q\|P)]$. To obtain $\hat{s}$, we take gradient assent with respect to the loss function $D_{\rm J} (\pi_{\theta}(\cdot|s)\|\pi_{\theta}(\cdot|\hat s))$ and restrict the outputs to the $\mathbb{B}_d(s,\epsilon)$ set, where $\pi_\theta$ is a learned CQL policy. We remark that the the perturbation is applied on normalized observations following prior work \cite{zhang2020robust}.

In the \emph{walker-medium-v2} task from D4RL \cite{d4rl-2020}, we use various $\epsilon$ for adversarial attack to evaluate the robustness of CQL policies. Specifically, we use $\epsilon\in\{0, 0.05, 0.1, 0.14\}$ to control the strengths of the attack, where we have $\hat{s}=s$ if $\epsilon=0$. Given a specific $\epsilon$, we sample $N$ state-action pairs $\{(s_i,a_i)\}$ from the offline dataset, and then perform adversarial attack to obtain $\{(\hat{s}_i,a_i)\}$ and the corresponding $Q$-values $\{Q_i(\hat{s}_i,a_i)\}$, where the $Q$-function is the trained critic of CQL.

Figure \ref{fig:cql-non-smooth} shows the relationship between $\hat{s}_i$ and the corresponding $Q_i$ with different $\epsilon$. To visualize $\hat{s}_i$, we perform PCA dimensional reduction \cite{pca-1999} and choose one of the reduced dimensions to represent $\hat{s}_i$. More details can be found in Appendix \ref{ap:visual_cql}. With the increase of $\epsilon$ in the adversarial attack, the $Q$-curve has greater deviation compared to the curve with $\epsilon=0$. The result signifies that the $Q$-function of CQL is not smooth in the state space, which makes the adversarial noises easily affect the $Q$ values. As a comparison, we apply the proposed conservative smoothing loss in CQL training (i.e., \emph{CQL-smooth}) and use the same evaluation method to obtain $\hat{s}_i$ and $Q_i$. According to the result in Figure \ref{fig:cql-smooth}, the value function becomes smoother.

In addition, we show how the adversarial attack affects the final performance of offline RL policies. We use $\epsilon\in[0,0.15]$ to evaluate both the original CQL policies (i.e., \emph{CQL}) and CQL with conservative smoothing loss (i.e., \emph{CQL-smooth}) in adversarial attack. Figure \ref{fig:cql-score} shows the performance with different settings of $\epsilon$. We find that our smooth constraints significantly improve the robustness of CQL, especially for large adversarial noises.

\section{Robust Offline RL via Conservative Smoothing}
\label{sec:method}

\begin{figure}[t]
    \centering
    \begin{minipage}{0.42\linewidth}
        \includegraphics[width=\linewidth]{ 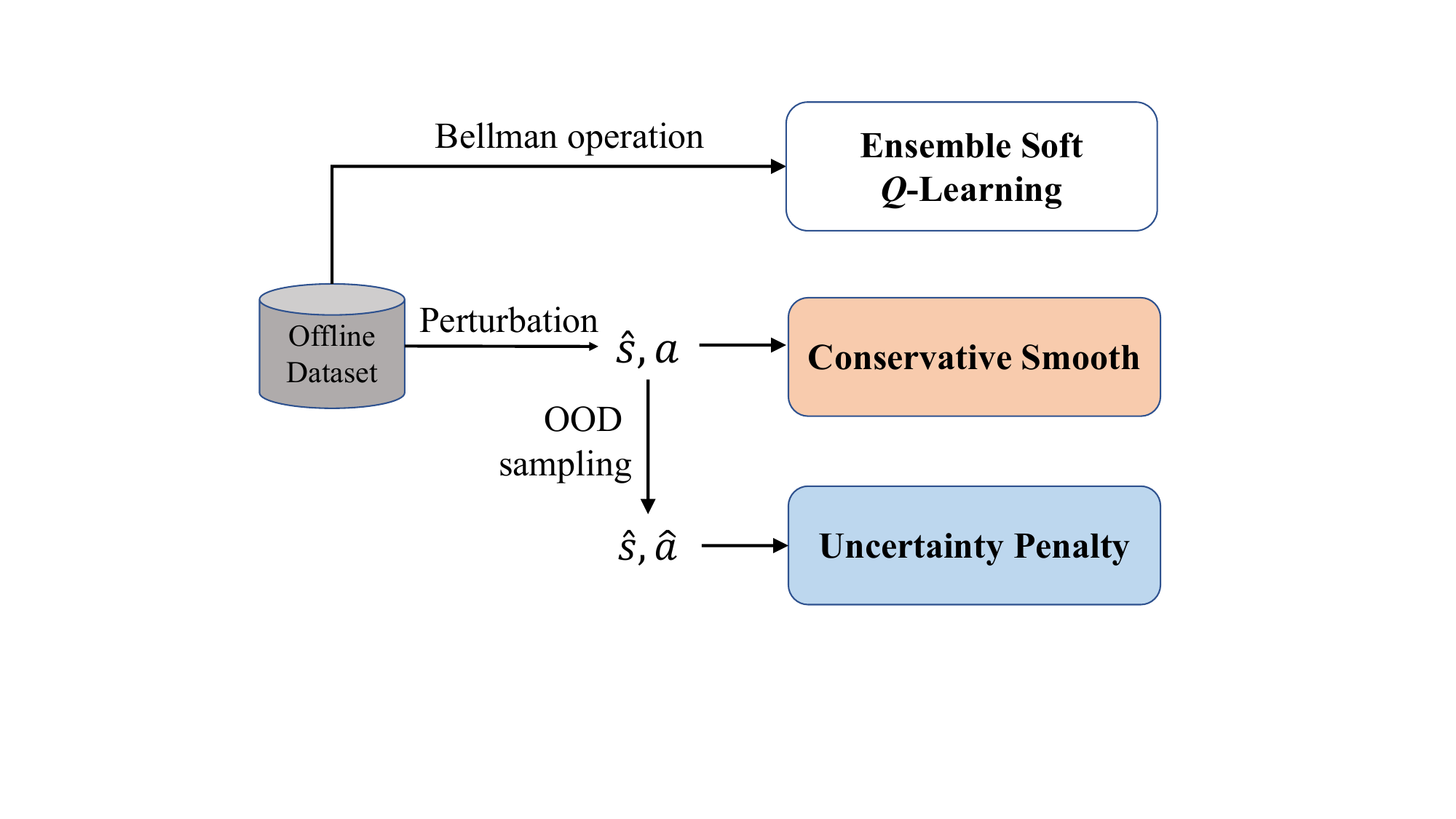}
    \end{minipage} \hfill
    \begin{minipage}{0.57\linewidth}
    
    \begin{algorithm}[H]
    \small
    \DontPrintSemicolon
    \SetAlgoLined
    Initialize policy $\pi_{\theta}$ and $Q$-functions $\{Q_{\phi_1},\ldots,Q_{\phi_K}\}$. \;
    \While{not converged}{
        Sample mini-batch transitions $(s,a,r,s')$ from $\mathcal{D}$.\;
        Sample $\hat{s}$ from $\mathbb{B}_d(s,\epsilon)$ to obtain $(\hat{s},a)$ pairs.\;
        Calculate the $Q$ smooth loss $\mathcal{L}_{\rm smooth}$.\;
        Sample OOD actions $\hat{a} \sim \pi_{\theta}(\hat{s})$. \; Calculate uncertainty $u(\hat{s},\hat{a})$ and the OOD loss $\mathcal{L}_{\rm ood}$. \;
        Train each $Q$ function $Q_{\phi_i}$ with Eq.~\eqref{eq:roal-Q}. \;
        Train the policy $\pi_{\theta}$ with Eq.~\eqref{eq:roal-policy}. \;
    }
    \caption{RORL Algorithm}
    \end{algorithm}
    \end{minipage}
    \caption{\textbf{RORL Algorithm}: RORL trains multiple $Q$-functions for uncertainty quantification. The conservative smoothing loss is calculated for $(\hat{s},a)$ with perturbed states. We perform uncertainty penalization for $(\hat{s},\hat{a})$ with perturbed states and OOD actions. 
\label{fig:model}}
\vspace{-1em}
\end{figure}

In RORL, we develop smooth regularization on both the policy and the value function for states near the dataset. The smooth constraints make the policy and the $Q$-functions robust to observation perturbations. Nevertheless, the smoothness may also lead to value overestimation in areas outside the supported dataset. To address this problem, we adopt bootstrapped $Q$-functions \cite{osband2016deep,bai2021pessimistic} for uncertainty quantification and sample perturbed states and OOD actions for penalization. RORL obtains conservative and smooth value estimation on OOD states, which can improve the generalization ability of offline RL algorithms. The overall architecture of RORL is given in Figure \ref{fig:model}.

\paragraph{Robust $Q$-function} We sample three sets of state-action pairs and apply different loss functions to obtain a conservative and smooth policy. Specifically, for a $(s,a)$ pair sampled from $\mathcal{D}$, we construct a perturbation set $\mathbb{B}_d(s,\epsilon)$ to obtain $(\hat{s},a)$ pairs, where $\hat{s}\in \mathbb{B}_d(s,\epsilon)$ and $\epsilon$ is the perturbation scale. The perturbation set $\mathbb{B}_{d}(s,\epsilon)=\{\hat{s}:d(s,\hat{s})\leq \epsilon\}$ for state $s$ is an $\epsilon$-radius ball measured in metric $d(·,·)$, which is the $\ell_\infty$ norm in our paper.
Then we perform OOD sampling by using the current policy $\pi_\theta$ to obtain $(\hat{s},\hat{a})$ pairs, where $\hat{a}\sim\pi_\theta(\hat{s})$. RORL contains $K$ ensemble $Q$-functions. We denote the parameters of the $i$-th $Q$-function and the target $Q$-function as $\phi_i$ and $\phi'_i$, respectively. In the following, we give different learning targets for $(s,a)$, $(\hat{s},a)$, and $(\hat{s},\hat{a})$ pairs.

First, for a $(s,a)$ pair sampled from $\mathcal{D}$, we apply extended soft $Q$-learning to obtain the target as 
\begin{equation}
    \label{eq:soft_Q_learning}
    \widehat{\mathcal{T}}Q_{\phi_i}(s, a):=r(s,a)+\gamma \widehat {\mathbb{E}}_{a'\sim \pi_\theta(\cdot|s')}\big[\min_{j=1,\ldots,K}Q_{\phi'_j}(s',a')- \alpha \cdot \log\pi_\theta(a'|s')\big],
\end{equation}
where the next-$Q$ function takes minimum value among the target $Q$-functions and $\log\pi_\theta(a'|s')$ is the entropy regularization. Note that Eq.~\eqref{eq:soft_Q_learning} is the same learning target of SAC-$N$ in~\cite{an2021uncertainty}.

Then, for a $(\hat{s},a)$ pair with a perturbed state, we enforce smoothness in each $Q$-function by minimizing the $Q$-value difference between $Q(s,a)$ and $Q(\hat{s},a)$. In particular, we choose an adversarial $\hat{s}\in \mathbb{B}_d(s,\epsilon)$ that maximizes a inner objective $\mathcal{L}(Q(\hat{s},a),Q(s,a))$, and then train each $Q$-function to minimize a loss function $\mathcal{L}_{\rm smooth}$ with the adversarial $\hat{s}$. Intuitively, we want the $Q$-function to be smooth under the most difficult (i.e., adversarial) perturbation in $\mathbb{B}_d(s,\epsilon)$. The smooth loss function for $Q_{\phi_i}$ is as follows:
\begin{equation}
\label{eq:Q_smooth}
    \mathcal{L}_{\rm smooth}(s,a; \phi_i) = \max_{\hat s\in \mathbb{B}_d(s,\epsilon)}  \mathcal{L}\big(Q_{\phi_i}(\hat s,a),Q_{\phi_i}(s,a)\big).
\end{equation}
We denote $\delta(s, \hat s,a)=Q_{\phi_i}(\hat s,a)-Q_{\phi_i}(s,a)$ and remark that if $\delta(s, \hat s,a)>0$, the perturbed state may induce an overestimated $Q$-value that we need to smooth. In contrast, if $\delta(s, \hat s,a)<0$, the perturbed $Q$-function is underestimated, which does not cause a serious problem in offline RL. As a result, we use different weights for $\delta(s, \hat s,a)_+$ and $\delta(s, \hat s,a)_-$, where $x_+=\max(x,0)$ and $x_-=\min(x,0)$. The definition of $\mathcal{L}(\cdot,\cdot)$ is give as follows:
\begin{equation}
\label{eq:conservative_smoothing_Q}
\mathcal{L}\big(Q_{\phi_i}(\hat s,a),Q_{\phi_i}(s,a)\big) = (1-\tau) \delta(s, \hat s,a)_{+}^2 + \tau \delta (s, \hat s,a)_{-}^2,
\end{equation}
where we can choose $\tau \leq 0.5$. In $\mathcal{L}_{\rm smooth}$, we does not introduce OOD action $\hat a$ for smoothing since the actions are desired to be close to the behavior actions for areas near the offline dataset.

Finally, to prevent overestimation of OOD states and actions, we use bootstrapped uncertainty $u(\hat{s},\hat{a})$ as the penalty for $Q(\hat{s},\hat{a})$, where $\hat{a}\sim\pi_{\theta}(\hat{s})$ is an OOD action sampled from the current policy $\pi_{\theta}$. We remark that a similar OOD sampling is also used in PBRL \cite{bai2021pessimistic}. \emph{The difference is that PBRL only penalizes the OOD actions for in-distribution states, while RORL penalizes both the OOD states and OOD actions to provide conservatism for unfamiliar areas.} 
We follow PBRL and use a loss function as:
\begin{equation}
\label{eq:ood_loss}
    \mathcal{L}_{\rm ood}(s; \phi_i) = \mathbb{E}_{\hat s \sim \mathbb{B}_d(s,\epsilon),\hat a \sim \pi_\theta(\hat{s})} \big(\widehat{\mathcal{T}}_{\rm ood}Q_{\phi_i}(\hat s, \hat a) - Q_{\phi_i}(\hat s, \hat a)\big)^2,
\end{equation}
where the pseudo-target for the OOD datapoints is computed as:
$\widehat{\mathcal{T}}_{\rm ood}Q_{\phi_i}(\hat{s},\hat{a}):=Q_{\phi_i}(\hat{s},\hat{a})-u(\hat{s},\hat{a})$, which is detached from gradients similar to the conventional TD target. The bootstrapped uncertainty $u(\hat{s},\hat{a})$ is defined as the standard deviation among the $Q$-ensemble:
\begin{equation*}
    u(\hat{s},\hat{a}):=\sqrt{\frac{1}{K}\sum\nolimits_{k=1}^K \big(Q_{\phi_i}(\hat s,\hat a)-\bar{Q}_{\phi}(\hat s,\hat a)\big)^2}.
\end{equation*}
The ensemble technique~\cite{osband2016deep} forms an estimation of the $Q$-posterior, which yields diverse predictions and large penalty $u(\hat{s},\hat{a})$ on areas with scarce data. 

Combining the loss functions above, RORL has the following loss function for each $Q_{\phi_i}$:
\begin{equation}
\begin{aligned}
\min_{\phi_i}  \mathbb{E}_{s,a,r,s'\sim \mathcal{D}} \Big[ &\big(\widehat{\mathcal{T}} Q_{\phi_i}(s,a) - Q_{\phi_i}(s,a)\big)^2 + \beta_{\text{Q}}   \mathcal{L}_{\rm smooth}(s,a; \phi_i) + \beta_{\rm ood}  \mathcal{L}_{\rm ood}(s; \phi_i) \Big],
\label{eq:roal-Q}
\end{aligned}
\end{equation}

\paragraph{Robust Policy}
We learn a robust policy by using a smooth constraint to make the policy change less under perturbations. Similarly, we choose an adversarial state $\hat{s}\in\mathbb{B}_d(s,\epsilon)$ that maximizes $D_{\rm J} \big(\pi_{\theta}(\cdot|s)\|\pi_{\theta}(\cdot|\hat s)\big)$, and then minimize the policy difference between $\pi_\theta(\cdot|s)$ and $\pi_\theta(\cdot|\hat{s})$. To conclude, we minimize the following loss function for $\pi_{\theta}$:
\begin{equation}
\label{eq:roal-policy}
\min_{\theta} \Big[\mathbb{E}_{s \sim \mathcal{D}, a \sim \pi_{\theta}(\cdot|s)} \big[-\min_{j=1,\ldots,K}Q_{\phi_j}(s,a) + \alpha \log \pi_\theta(a|s) + \beta_{\rm P}   \max_{\hat s\in \mathbb{B}_d(s,\epsilon)} D_{\rm J} \big(\pi_{\theta}(\cdot|s)\|\pi_{\theta}(\cdot|\hat s)\big) \big]\:\Big],
\end{equation}
where the first term aims to maximize the minimum of the ensemble $Q$-functions to obtain a conservative policy, and the second term is the entropy regularization.

\section{Theoretical Analysis} \label{sec:theroy}

We analyze a simplified learning objective of RORL in linear MDPs \cite{lsvi-2020,jin2021pessimism}, where the feature map of the state-action pair takes the form of $\phi:\mathcal{S}\times\mathcal{A}\rightarrow\mathbb{R}^d$, and both the transition function and the reward function are assumed to be linear in $\phi$. The parameter $\widetilde w_t$ of RORL can be solved in closed form following the least squares value iteration (LSVI), which minimizes the following loss function.
\begin{equation}
\label{eq:simplified_problem-main}
\begin{aligned}
    \widetilde w_t^i =  \min_{w\in \mathcal{R}^d} & \Big[\sum_{i=1}^{m} \big(y_t^i-Q_{w}(s_t^i,a_t^i)\big)^2 +   \sum_{i=1}^{m} \frac{1}{|\mathbb{B}_d(s_t^i,\epsilon)|} \sum_{\hat{s}_t^i\in \mathcal{D}_{\text{ood}}(s_t^i)}  \big(Q_{w}(s_t^i,a_t^i) - Q_{w}(\hat{s}_t^i,a_t^i)\big)^2  + \\ & \sum_{(\hat s, \hat a, \hat y) \sim \mathcal{D}_{\text{ood}} } \big(\hat y - Q_{w}(\hat s,\hat a)\big)^2 \Big],
\end{aligned}
\end{equation}
where we have $Q_{w}(s_t^i,a_t^i)=\phi(s_t^i,a_t^i)^\top w$ since the $Q$-function is also linear in $\phi$. The first term in Eq.~\eqref{eq:simplified_problem-main} is the ordinary TD-error, where we consider the setting of $\gamma = 1$ and the $Q$-target is $y_t^i=r(s_t^i,a_t^i)+V_{t+1}(s_{t+1}^i)$. The second term is the proposed conservative smoothing loss. Specifically, $\hat s_t^i\sim \mathcal{D}_{\text{ood}}(s_t^i)$ are sampled from a $l_\infty$ ball of center $s_t^i$ and norm $\epsilon > 0$, which can also be formulated as $\hat s_t^i\sim \mathbb{B}_d(s_t^i,\epsilon)$. The third term is the additional OOD-sampling loss, which enforces conservatism for OOD states and OOD actions. In contrast to PBRL~\cite{bai2021pessimistic}, we use perturbed states sampled from $\mathcal{D}_{\text{ood}}=\bigcup\limits_{i=1}^m \mathcal{D}_{\text{ood}}(s_t^i)$ rather than states from dataset. The OOD action $\hat a$ is sampled from policy $\pi$.
The explicit solution of Eq.~\eqref{eq:simplified_problem-main} takes the following form:
\begin{equation}\label{eq::main_w_ood_solu}
    \widetilde w_t^i =\widetilde \Lambda_t^{-1} \Big( \sum_{i=1}^{m} \phi(s_t^i,a_t^i) y_t^i + \sum_{(\hat s, \hat a, \hat y) \sim \mathcal{D}_{\text{ood}} }\phi(\hat s,\hat a)  \hat y  \Big),
\end{equation}
where the covariance matrix $\widetilde \Lambda_t$ is defined as
\begin{equation}
\begin{aligned}
        \widetilde \Lambda_t = &\sum_{i=1}^{m} \phi(s_t^i,a_t^i)\phi(s_t^i,a_t^i)^\top + \sum_{(\hat s, \hat a) \sim \mathcal{D}_{\text{ood}}} \phi(\hat s_t,\hat a_t)\phi(\hat s_t,\hat a_t)^\top \\
        &+ \sum_{i=1}^{m} \frac{1}{|\mathbb{B}_d(s_t^i,\epsilon)|}  \sum_{\hat s_t^i \sim \mathcal{D}_{\text{ood}}(s_t^i)} [\phi(\hat s_t^i,a_t^i) - \phi(s_t^i,a_t^i)]\big[\phi(\hat s_t^i,a_t^i) - \phi(s_t^i,a_t^i)\big]^\top. 
\label{eq:covariance-rorl-main}
\end{aligned}
\end{equation}

We denote the first term and the second term as $\widetilde \Lambda^{\text{in}}$ and $\widetilde \Lambda_t^{\text{ood}}$, which represent the covariance matrices induced by the offline samples and OOD samples, respectively. Nevertheless, in linear MDPs, it is difficult to ensure the covariance $\widetilde \Lambda^{\text{in}}+\widetilde \Lambda_t^{\text{ood}} \succeq \lambda \cdot \mathrm{I}$, since it requires that the embeddings of the samples are isotropic to make the eigenvalues of the corresponding covariance matrix lower bounded. This condition holds if we can sample embeddings uniformly from the whole embedding space. However, since the offline dataset has limited coverage in the state-action space and the OOD samples come from limited $l_{\infty}$-balls around the offline data, $\widetilde \Lambda^{\text{in}}+\widetilde \Lambda_t^{\text{ood}}$ cannot be guaranteed to be positive definite. PBRL~\cite{bai2021pessimistic} uses the assumption of $\widetilde \Lambda_t^{\text{ood}}\succeq \lambda \cdot \mathrm{I}$, while it is unachievable empirically. In RORL, we solve this problem by introducing an additional conservative smoothing loss, which induces a covariance matrix as $\widetilde \Lambda_t^{\text{ood\_diff}} = \sum_{i=1}^{m} \frac{1}{|\mathbb{B}_d(s_t^i,\epsilon)|}  \sum_{\hat s_t^i \sim \mathcal{D}_{\text{ood}}(s_t^i)} [\phi(\hat s_t^i,a_t^i) - \phi(s_t^i,a_t^i)][\phi(\hat s_t^i,a_t^i) - \phi(s_t^i,a_t^i)]^\top $ (i.e., the third term in Eq.~\eqref{eq:covariance-rorl-main}). The following theorem gives the guarantees of $\widetilde \Lambda_t^{\text{ood\_diff}}\succeq \lambda \cdot  \mathrm{I}$.

\begin{theorem}
\label{tm:PD_matrix}
Assume $\exists i\in[1,m]$ the vector group of all $\hat s_t^i \sim \mathcal{D}_{\text{ood}}(s_t^i)$: $\{\phi(\hat s_t^i,a_t^i) - \phi(s_t^i,a_t^i)\}$ be full rank, 
then the covariance matrix $\widetilde \Lambda_t^{\rm{ood\_diff}}$ is positive-definite: $\widetilde \Lambda_t^{\rm{ood\_diff}} \succeq \lambda \cdot \mathrm{I}$ where $\lambda > 0$.
\end{theorem}

Recall the covariance matrix of PBRL is $\widetilde \Lambda_t^{\text{PBRL}}=\widetilde \Lambda_t^{\text{in}}+\widetilde \Lambda_t^{\text{ood}}$, and RORL has a covariance matrix as $\widetilde \Lambda_t=\widetilde \Lambda_t^{\text{PBRL}}+\widetilde \Lambda_t^{\text{ood\_diff}}$, we have the following corollary based on Theorem \ref{tm:PD_matrix}.

\begin{corollary}\label{coro:positive-def}
Under the linear MDP assumptions and conditions in Theorem \ref{tm:PD_matrix}, we have $\widetilde \Lambda_t \succeq \widetilde \Lambda_t^{\rm PBRL}$. Further, the covariance matrix $\widetilde \Lambda_t$ of RORL is positive-definite: $\widetilde \Lambda_t \succeq \lambda \cdot \mathrm{I}$, where $\lambda > 0$.
\end{corollary}

Recent theoretical analysis shows that an appropriate uncertainty quantification is essential to provable efficiency in offline RL \cite{jin2021pessimism,xie2021bellman,bai2021pessimistic}. Pessimistic Value Iteration \cite{jin2021pessimism} defines a general $\xi$-uncertainty quantifier as the penalty and achieves provable efficient pessimism in offline RL. In linear MDPs, Lower Confidence Bound (LCB)-penalty \cite{bandit-2011,lsvi-2020} is known to be a $\xi$-uncertainty quantifier for appropriately selected $\beta_t$ as $\Gamma^{\rm lcb}(s_t,a_t)=\beta_t\cdot\big[\phi(s_t,a_t)^\top\Lambda_t^{-1}\phi(s_t,a_t)\big]^{\nicefrac{1}{2}}$. Following the analysis of PBRL \cite{bai2021pessimistic}, since the bootstrapped uncertainty is an estimation of the LCB-penalty and the OOD sampling provides a covariance matrix $\widetilde \Lambda_t \succeq \lambda \cdot \mathrm{I}$ given in Corollary \ref{coro:positive-def}, the proposed RORL also forms a valid $\xi$-uncertainty quantifier. This allows us to further characterize the optimality gap based on the pessimistic value iteration \cite{jin2021pessimism,bai2021pessimistic}. We have the following suboptimality gap under linear MDP assumptions.

\begin{corollary}
\label{cor::opt_gap_rorl}
${\rm SubOpt} (\pi^*, \hat \pi) \leq \sum_{t=1}^{T} \mathbb{E}_{\pi^*} \big[ \Gamma_t^{\rm lcb}(s_t,a_t) \big] < \sum_{t=1}^{T} \mathbb{E}_{\pi^*} \big[ \Gamma_t^{\rm lcb\_PBRL}(s_t,a_t) \big]$.
\end{corollary}
Detailed proof can be found in Appendix \ref{appendix-theoretical}. Corollary \ref{cor::opt_gap_rorl} indicates that RORL enjoys a tighter suboptimality bound than PBRL \cite{bai2021pessimistic}.



\section{Experiments}
\label{sec:experiments}

We evaluate our method on the D4RL benchmark \cite{d4rl-2020} with various continuous-control tasks and datasets. We compare RORL with several offline RL algorithms, including 
(\romannumeral1) BC that performs behavior cloning, (\romannumeral 2) CQL \cite{kumar2020conservative} that learns conservative value function for OOD actions, (\romannumeral3) EDAC \cite{an2021uncertainty} that learns a diversified $Q$-ensemble to enforce conservatism, and (\romannumeral4) PBRL \cite{bai2021pessimistic} that performs uncertainty penalization and OOD sampling. We also include a basic SAC-10 algorithm as a baseline \cite{an2021uncertainty}, which is an extension of SAC with 10 $Q$-functions.
Among these methods, EDAC \cite{an2021uncertainty} and PBRL \cite{bai2021pessimistic} are related to RORL since all these methods apply $Q$-ensemble for conservatism. EDAC needs much more $Q$-networks (i.e., 10$\sim$50) for hopper tasks than PBRL and RORL that only use 10 $Q$-networks. For fair comparison, we also report the reproduced results of EDAC-10. To assign uniform adversarial attack budget on each dimension of observations, we normalize the observations for SAC-10, EDAC and RORL. Besides, we use different perturbation scales for the policy smoothing loss, the Q smoothing loss and the OOD loss, namely $\epsilon_{\rm P}$, $\epsilon_{\rm Q}$ and $\epsilon_{\rm ood}$. More hyper-parameters and implementation details are provided in Appendix \ref{ap-implementation}.

\begin{table}[t]
	\centering
	\small
	\caption{Normalized average returns on Gym tasks, averaged over 4 random seeds. Part of the results are reported in the EDAC paper. Top two scores for each task are highlighted.}
	\vspace{0.2em}
	\label{tab:gym_more}
	\begin{adjustbox}{max width=\linewidth}
		\begin{tabular}{l|rrrrrr|r}
			\toprule
			\multirow{2}{*}{\textbf{Task Name}} & \multirow{2}{*}{\textbf{BC}} &  \multirow{2}{*}{\textbf{CQL}} &
			\multirow{2}{*}{\textbf{PBRL}}   &
			\textbf{SAC-$10$} &
			\textbf{EDAC} & \textbf{EDAC-10} &  \textbf{RORL} \\
			& & & & \textbf{(Reproduced)} & \textbf{(Paper)} & \textbf{(Reproduced)} &  \textbf{(Ours)} \\
			\midrule
			halfcheetah-random & 2.2$\pm$0.0  &  \textbf{31.3$\pm$3.5} & 11.0$\pm$5.8 & \textbf{29.0$\pm$1.5} & 28.4$\pm$1.0  &
			13.4 $\pm$ 1.1 & 28.5$\pm$0.8 \\    
			halfcheetah-medium & 43.2$\pm$0.6  & 46.9$\pm$0.4 & 57.9 $\pm$1.5 & 64.9$\pm$1.3 & \textbf{65.9$\pm$0.6}  & 64.1$\pm$1.1  & \textbf{66.8$\pm$0.7}
			\\
			halfcheetah-medium-expert & 44.0$\pm$1.6  & 95.0$\pm$1.4 & 92.3$\pm$1.1 & 107.1$\pm$2.0 & 106.3$\pm$1.9 &
			\textbf{107.2$\pm$1.0} & \textbf{107.8$\pm$1.1}	\\ 
			halfcheetah-medium-replay & 37.6$\pm$2.1  &45.3$\pm$0.3 & 45.1$\pm$8.0 & \textbf{63.2$\pm$0.6} & 61.3$\pm$1.9 &
			60.1$\pm$0.3 & \textbf{61.9$\pm$1.5} \\ 
			halfcheetah-expert & 91.8$\pm$1.5  & 97.3$\pm$1.1 & 92.4$\pm$1.7 & 104.9$\pm$0.9 & \textbf{106.8$\pm$3.4} &
			104.0$\pm$0.8 & \textbf{105.2$\pm$0.7}\\
			\midrule
			hopper-random & 3.7$\pm$0.6 & 5.3$\pm$0.6 & \textbf{26.8$\pm$9.3} & 25.9$\pm$9.6 & 25.3$\pm$10.4 &
			16.9$\pm$10.1 & \textbf{31.4$\pm$0.1}	\\   
			hopper-medium & 54.1$\pm$3.8 & 61.9$\pm$6.4 & 75.3$\pm$31.2 &  0.8$\pm$0.2 & 101.6$\pm$0.6  &
			\textbf{103.6$\pm$0.2} & \textbf{104.8$\pm$0.1}\\      
			hopper-medium-expert & 53.9$\pm$4.7 & 96.9$\pm$15.1 & \textbf{110.8$\pm$0.8} & 6.1$\pm$7.7 & 110.7$\pm$0.1 &
			58.1$\pm$22.3 &  
			\textbf{112.7$\pm$0.2}\\
			hopper-medium-replay & 16.6$\pm$4.8  &  86.3$\pm$7.3 & 100.6$\pm$1.0 & \textbf{102.9$\pm$0.9} & 101.0$\pm$0.5 &
			\textbf{102.8$\pm$0.3} & \textbf{102.8$\pm$0.5}	\\ 
			hopper-expert & 107.7$\pm$9.7 & 106.5$\pm$9.1 & \textbf{110.5$\pm$0.4} & 1.1$\pm$0.5 & 110.1$\pm$0.1 & 77.0$\pm$43.9 &
			\textbf{112.8$\pm$0.2}\\
			\midrule
			walker2d-random & 1.3$\pm$0.1  & 5.4$\pm$1.7 & 8.1$\pm$4.4  & 1.5$\pm$1.1 & \textbf{16.6$\pm$7.0} & 6.7$\pm$8.8 & \textbf{21.4$\pm$0.2}  \\
			walker2d-medium & 70.9$\pm$11.0  & 79.5$\pm$3.2 & 89.6$\pm$0.7 & 46.7$\pm$45.3 & \textbf{92.5$\pm$0.8} & 87.6$\pm$11.0 & \textbf{102.4$\pm$1.4} \\
			walker2d-medium-expert & 90.1$\pm$13.2 &  109.1$\pm$0.2 & 110.1$\pm$0.3 & \textbf{116.7$\pm$1.9} & 114.7$\pm$0.9 &
			115.4$\pm$0.5 &
			\textbf{121.2$\pm$1.5} \\
			walker2d-medium-replay & 20.3$\pm$9.8 & 76.8$\pm$10.0 & 77.7$\pm$14.5 & 89.6$\pm$3.1 & 87.1$\pm$2.3 & \textbf{94.0$\pm$1.2} & \textbf{90.4 $\pm$ 0.5}  \\
			walker2d-expert & 108.7$\pm$0.2 & 109.3$\pm$0.1 & 108.3$\pm$0.3 & 1.2$\pm$0.7 & \textbf{115.1$\pm$1.9} & 57.8$\pm$55.7 & \textbf{115.4 $\pm$ 0.5} \\
			\midrule
			Average & 49.7 & 70.2 &  74.4 & 50.8 & \textbf{82.9} & 71.2 &  \textbf{85.7} \\ 
			Total & 746.1 & 1052.8 &  1116.5 & 761.6 & \textbf{1243.4} & 1068.7  &  \textbf{1285.7}\\ 
			\bottomrule
		\end{tabular}
	\end{adjustbox}
\vspace{-1em}
\end{table}

\subsection{Benchmark Results}

We evaluate each method on Gym domain that includes three environments (HalfCheetah, Hopper, and Walker2d) with five types of datasets (random, medium, medium-replay, medium-expert, and expert) for each environment.  
The medium-replay dataset contains experiences collected in training a medium-level policy. The random/medium/expert dataset is generated by a single random/medium/expert policy. The medium-expert dataset is a mixture of medium and expert datasets. For benchmark experiments, we set small perturbation scales $\epsilon_{\rm P}$, $\epsilon_{\rm Q}$, and $\epsilon_{\rm ood}$ within $\{0.001, 0.005, 0.01\}$ when training RORL and do not include observation perturbation in the testing time. 

Table~\ref{tab:gym_more} reports the performance of the average normalized score with standard deviation. (\romannumeral1) SAC-10 is unstable on several walker2d and hopper tasks since the ensemble number is relatively small to provide reliable uncertainties for SAC-$N$ \cite{an2021uncertainty}. (\romannumeral2) EDAC solves this problem by gradient diversity constraints while still requiring 10$\sim$50 $Q$-networks to obtain reasonable performance. In contrast, RORL only uses 10 ensemble $Q$-networks to achieve better or comparable performance with EDAC. Additionally, we also show that RORL outperforms EDAC-10 by a large margin. (\romannumeral3) PBRL chooses an alternative OOD-sampling technique to reduce the ensemble numbers. According to the result, RORL significantly outperforms PBRL with the same ensemble number. The reason is RORL additionally uses conservative smoothing loss for perturbed states and penalizes values of these states based on uncertainty estimation, which may improve the generalization ability of the learned policy on continuous state space. We remark that RORL significantly improves over the current SOTA results on walker2d and hopper tasks, probably because these two tasks require a more precise balance of conservatism and robustness for better performance.

\begin{figure*}[t]
\centering
\subfigure[Performance under attack on the halfcheetah-medium-v2 dataset]{\includegraphics[width=1\textwidth]{ 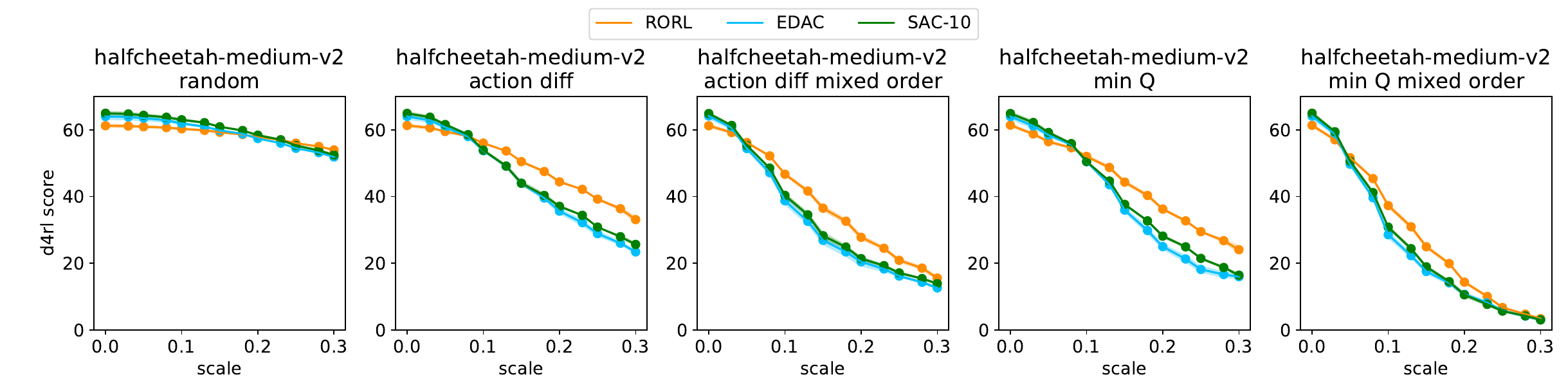}\label{fig:attack_halfcheetah}}
\subfigure[Performance under attack on the walker2d-medium-v2 dataset]{\includegraphics[width=1\textwidth]{ 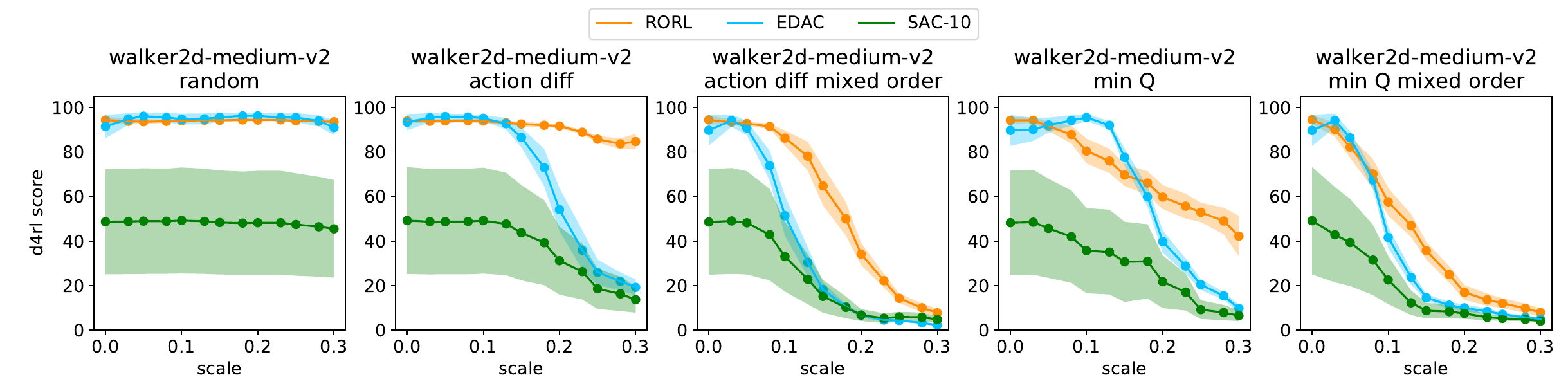}\label{fig:attack_walker2d}}
\subfigure[Performance under attack on the hopper-medium-v2 dataset]{\includegraphics[width=1\textwidth]{ 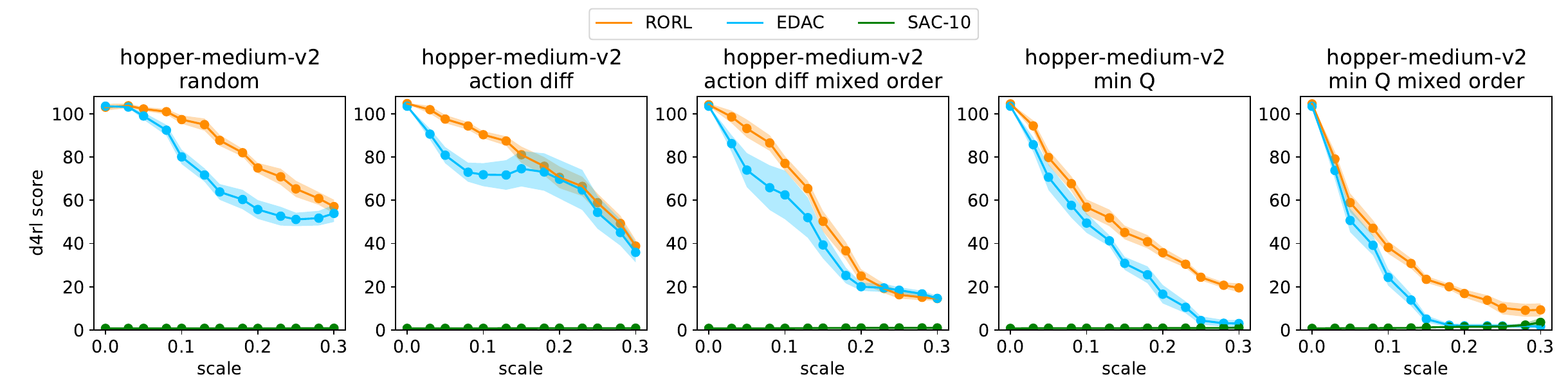}\label{fig:attack_hopper}}
\caption{(a) (b) (c) illustrate the performance of RORL, EDAC and SAC-10 under attack scales range $[0,0.3]$ of different attack types. The curves are averaged over 4 seeds and smoothed with a window size of 3. The shaded region represents half a standard deviation. }
\label{fig:attack_rorl_new}
\vspace{-1em}
\end{figure*}

\subsection{Adversarial Attack}
\label{sec:adver_attak}
We adopt three attack methods, namely \emph{random}, \emph{action diff}, and \emph{min Q} following prior works \cite{zhang2020robust,pinto2017robust}. Given perturbation scale $\epsilon$, the later two methods perform adversarial perturbation on observations and are given access to the agent's policy and value functions. Details about the three attack methods are as follows.
\begin{itemize}
    \item  \emph{random} uniformly samples perturbed states in an $l_{\infty}$ ball of norm $\epsilon$.
    \item \emph{action diff} is an effective attack based on the agent's policy and is proved to be an upper bound on the performance difference between perturbed and unperturbed environments \cite{zhang2020robust}. It directly finds perturbed states in an $l_{\infty}$ ball of norm $\epsilon$ to satisfy:
    $\max_{\hat s\in \mathbb{B}_d(s,\epsilon)} D_{\rm J} \big(\pi_{\theta}(\cdot|s)\|\pi_{\theta}(\cdot|\hat s)\big)$, i.e., $\min_{\hat s\in \mathbb{B}_d(s,\epsilon)} - D_{\rm J} \big(\pi_{\theta}(\cdot|s)\|\pi_{\theta}(\cdot|\hat s)\big)$.
    \item \emph{min Q} requires both the agent's policy and value function to perform a relatively stronger attack. The attacker finds a perturbed state to minimize the expected return of taking an action from that state: $\min_{\hat s\in \mathbb{B}_d(s,\epsilon)} Q(s, \pi_{\theta}(\hat s))$. For ensemble-based algorithms, $Q$ is set as the mean of ensemble $Q$ functions.
\end{itemize}

In our experiments, the two objectives of \emph{action diff} and \emph{min Q} are optimized via two ways. Specifically, we optimize the objectives through:
\begin{itemize}
    \item[(1)] selecting the best perturbed state from uniformly sampled 50 states, which has the advantage of simplicity and little computation cost. For attacks with this type of optimization, we use their original names without specifying.
    \item[(2)] uniformly sampling 20 initial states and performing gradient decent for 10 steps with a step size of$\frac{1}{10} \epsilon$ from each initial state to find the best perturbed state. Note that we need to clip the perturbed states within the $l_{\infty}$ ball at the end of each optimization step. Among the attacks using this optimization, we specifically remark "mixed-order" in their names.
\end{itemize}

We compare RORL with ensemble-based baselines EDAC and SAC-10 on halfcheetah-medium-v2, walker2d-medium-v2, and hopper-medium-v2 datasets. To handle large adversarial noise, we set the perturbation scales $\epsilon_{\rm P}$, $\epsilon_{\rm Q}$ and $\epsilon_{\rm ood}$ within $\{0.01, 0.03, 0.05, 0.07\}$ in RORL's training phase. More detailed description can be found in Appendix \ref{ap-implementation}. The results are shown in Figure \ref{fig:attack_rorl_new}. In the results, RORL exhibits improved robustness than other baselines under five types of adversarial attacks. On the other hand, we find that random attack is not effective for ensemble-based offline RL algorithms, and the ``mixed order'' attack brings more significant performance drop than vanilla zero-order optimization.

\begin{figure}[h]
    \centering
    \includegraphics[width=1\linewidth]{ 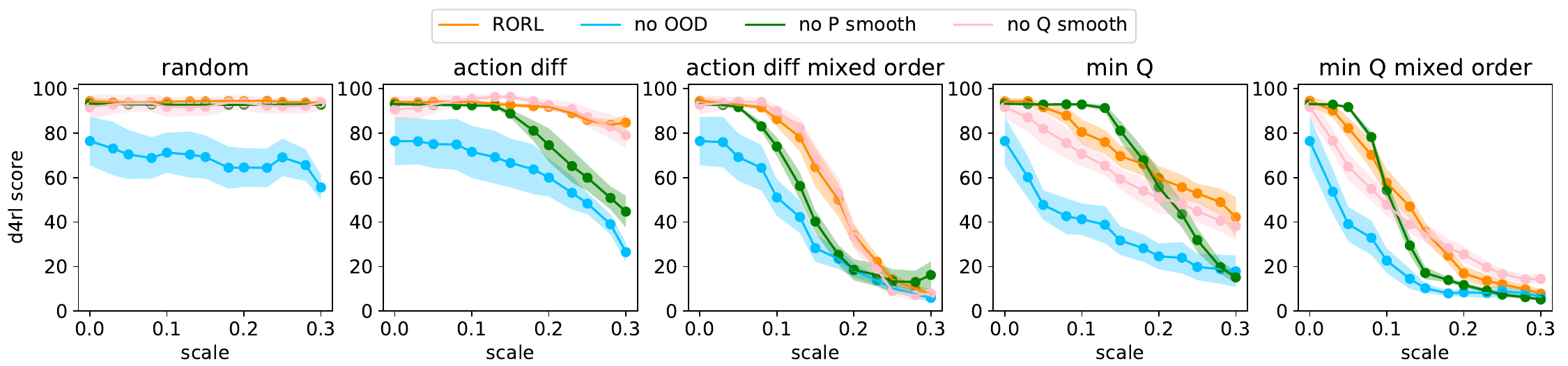}
    \caption{Ablation studies on the walker2d-medium-v2 dataset with varying perturbation scale. The curve is averaged across 4 random seeds and smoothed with a window size of 3. The shaded region represents half a standard deviation.}
    \label{fig:ablation}
\vspace{-1em}
\end{figure}

\subsection{Ablations}
\label{sec:ablation}
We conduct ablation studies on the walker2d-medium-v2 dataset to evaluate the importance of three terms, i.e., the policy smoothing loss, the $Q$ smoothing term and the OOD loss. From the results in Figure \ref{fig:ablation}, we can conclude that each loss contributes to the performance of RORL under adversarial observation attacks. The OOD loss is the most essential term, without which the performance is worse than RORL at almost all perturbation scales and all types of attacks. The policy smoothing loss is also important, especially for perturbation scales larger than 0.2. In addition, $Q$ smooth loss has the minimal impact, which is reasonable since the basic algorithm SAC-10 is based on 10 ensemble $Q$ networks. More ablations on the number of $Q$ networks, the effect of $\epsilon_{\rm ood}$ and $\tau$, and a comparison with more baselines can be found in Appendix \ref{ap:addtional_exp}. 

\begin{wraptable}{r}{5.5cm}
  \vspace{-1.4em}
  \centering
   \caption{Computational costs.}
   \label{tab:compute_cost}
   \begin{adjustbox}{max width=\linewidth}
   \begin{tabular}{ccc}
        \toprule
        & Runtime  & GPU Memory  \\ 
        & (s/epoch) & (GB) \\
        \midrule
        \textbf{CQL} & 32.40 & 1.4 \\
        \textbf{SAC-$10$} & 12.73 & 1.3 \\
        \textbf{PBRL} & 102.96 & 1.8 \\
        \textbf{EDAC} & 17.94 & 1.8 \\
        \textbf{RORL} & 29.56 & 2.1 \\
        \bottomrule
    \end{tabular}
    \end{adjustbox}
\vspace{-1em}
\end{wraptable}

\subsection{Computational Cost Comparison}
We compare the computational cost of RORL with prior works on a single machine with one GPU (Tesla V100 32G). For each method, we measure the average epoch time (i.e., 1$\times 10^3$ training steps) and the GPU memory usage on the hopper-medium-v2 task. More discussions are provided in Appendix \ref{ap:computational_cost}.

As shown in Table \ref{tab:compute_cost}, RORL runs slightly faster than CQL and much faster than PBRL. PBRL is so slow because it uses 10 $Q$ networks and needs OOD action sampling. In RORL, we also include the OOD state-action sampling and the robust training procedure, but we implemented these procedures efficiently based on the parallelization of $Q$ networks. Even so, RORL is still slower than SAC-10 and EDAC. As demonstrated in our experiments, RORL enjoys significantly better robustness than EDAC and SAC-10 under adversarial perturbations. Regarding the GPU memory consumption, RORL uses comparable memory to PBRL and EDAC, with only $16.7\%$ more memory usage.

\section{Related Works}
\paragraph{Offline RL} 
Research related to offline RL has experienced explosive growth in recent years. In model-free domain, offline RL methods focus on correcting the extrapolation error~\cite{fujimoto2019off} in the off-policy algorithms. The natural idea is to regularize the learned policy near the dataset distribution~\cite{wang2018exponentially,wu2019behavior,nair2020accelerating,wang2020critic,yang2021believe,fujimoto2021minimalist,yang2022rethinking}. For example, MARVIL reweights the policy with exponential advantage, which implicitly guarantees the policy within the KL-divergence neighborhood of the behavior policy. Another stream of model-free methods prevents the selection of OOD actions by penalizing their $Q$-value~\cite{kumar2019stabilizing,kumar2020conservative,an2021uncertainty,cheng2022adversarially} or $V$-learning~\cite{ma2022offline,kostrikov2021offline}. With the ensemble $Q$ networks and the additional loss term to diversify their gradients, EDAC~\cite{an2021uncertainty} achieves SOTA performance in the D4RL benchmark. Instead of diversifying gradients, PBRL~\cite{bai2021pessimistic} proposes an explicit value underestimation of OOD actions according to the uncertainty, which requires fewer ensemble networks. Inspired by EDAC and PBRL, we build our work upon ensemble networks, focusing more on the smoothness over the state space.

Besides the surprising empirical results, theoretical analysis of offline reinforcement learning algorithms is of increasing interest~\cite{chen2019information,jin2021pessimism,rashidinejad2021bridging,xie2021bellman,yin2022near}. Though the assumptions for the dataset vary in the different papers, they all suggest that pessimism and conservatism are necessary for offline RL. Our theoretical results can be viewed as robust extensions to previous theoretical results~\cite{jin2021pessimism,bai2021pessimistic}.

\paragraph{Robust RL} The research line of robust RL can be traced back to $H_\infty$-control theory~\cite{xie1990robust,bacsar2008h}, where policies are optimized to be well-performed in the worst possible deterministic environment. Depending on the definition, there are different streams of research on robust RL. As the extension of robust control to MDPs, Robust MDPs (RMDPs)~\cite{nilim2003robustness,iyengar2005robust,roy2017reinforcement,ho2018fast} are proposed to formulate the perturbation of transition probabilities for MDPs.
Though some recent analyses with theoretical guarantees come out under specific assumptions for RMDPs~\cite{zhou2021finite,yang2021towards,li2022policy}, there is currently no practical algorithm to solve RMDPs in a large-scale problem, expect some linear approximation attempt~\cite{tamar2013scaling}.
In online RL, domain randomization~\cite{tobin2017domain,mehta2020active} assumes the model uncertainty can be predefined in data collection by changing the setup of a simulator. However, it is not practical for offline RL.  
Robust Adversarial Reinforcement Learning (RARL)~\cite{pinto2017robust} and Noisy Robust Markov Decision Process (NR-MDP)
\cite{kamalaruban2020robust} study the robust RL with the perturbed actions, showing that the policy robustness to adversarial or noisy actions can also induce robustness for model parameter changes.
The most related work to ours is SR$^2$L~\cite{shen2020deep}, which shows policy smoothing can lead to significant performance improvement in the online setting. In contrast, we focus on the offline setting and tackle the potential overestimation of perturbed states. Another related work is S4RL~\cite{sinha2022s4rl}, where the authors study different data augmentation methods to smooth observations in offline RL. Their result supports the necessity of state smoothing. More related works are discussed in Appendix~\ref{appendix:related-work}.

\section{Conclusion}
We propose Robust Offline Reinforcement Learning (RORL) to trade-off conservatism and robustness for offline RL. To achieve that, we introduce the conservative smoothing technique for the perturbed states while actively underestimating their values based on pessimistic bootstrapping to keep conservative. We show that RORL can achieve comparable or even better performance with fewer ensemble $Q$ networks than previous methods in the offline RL benchmark. In addition, we demonstrate that RORL is considerably robust to adversarial perturbations across different types of attacks. We hope our work can promote the application of offline RL under real-world engineering conditions. 

The main limitation of our method is that the adversarial state sampling slows down the computing process, which may be improved in future work. Also, an interesting direction is to smooth or penalize the policy and $Q$ functions in latent spaces rather than the normalized observation space.

\section*{Acknowledgements}

This work was in part supported by Tencent Robotics X and Shanghai AI Laboratory, and in part by Science and Technology Innovation 2030 – “New Generation Artificial Intelligence” Major Project (No. 2018AAA0100904) and National Natural Science Foundation of China (62176135). The authors would like to thank the anonymous reviewers. Rui Yang thanks Yi Wang and Haoyi Song for valuable discussion.

\bibliographystyle{plain}

\newpage
\section*{Checklist}

The checklist follows the references.  Please
read the checklist guidelines carefully for information on how to answer these
questions.  For each question, change the default \answerTODO{} to \answerYes{},
\answerNo{}, or \answerNA{}.  You are strongly encouraged to include a {\bf
justification to your answer}, either by referencing the appropriate section of
your paper or providing a brief inline description.  For example:
\begin{itemize}
  \item Did you include the license to the code and datasets? \answerYes{See Section \ref{}.}
  \item Did you include the license to the code and datasets? \answerNo{The code and the data are proprietary.}
  \item Did you include the license to the code and datasets? \answerNA{}
\end{itemize}
Please do not modify the questions and only use the provided macros for your
answers.  Note that the Checklist section does not count towards the page
limit.  In your paper, please delete this instructions block and only keep the
Checklist section heading above along with the questions/answers below.

\begin{enumerate}

\item For all authors...
\begin{enumerate}
  \item Do the main claims made in the abstract and introduction accurately reflect the paper's contributions and scope?
    \answerYes{}
  \item Did you describe the limitations of your work?
    \answerYes{}
  \item Did you discuss any potential negative societal impacts of your work?
    \answerNA{}
  \item Have you read the ethics review guidelines and ensured that your paper conforms to them?
    \answerYes{}
\end{enumerate}

\item If you are including theoretical results...
\begin{enumerate}
  \item Did you state the full set of assumptions of all theoretical results?
    \answerYes{}
        \item Did you include complete proofs of all theoretical results?
    \answerYes{See Appendix \ref{appendix-theoretical}.}
\end{enumerate}

\item If you ran experiments...
\begin{enumerate}
  \item Did you include the code, data, and instructions needed to reproduce the main experimental results (either in the supplemental material or as a URL)?
    \answerYes{See Sec \ref{sec:intro} and Appendix \ref{ap-implementation}.}
  \item Did you specify all the training details (e.g., data splits, hyper-parameters, how they were chosen)?
    \answerYes{See Appendix \ref{ap-implementation}.}
        \item Did you report error bars (e.g., with respect to the random seed after running experiments multiple times)?
    \answerYes{See Sec \ref{sec:experiments}.}
        \item Did you include the total amount of compute and the type of resources used (e.g., type of GPUs, internal cluster, or cloud provider)?
    \answerYes{See Appendix \ref{ap-implementation}.}
\end{enumerate}

\item If you are using existing assets (e.g., code, data, models) or curating/releasing new assets...
\begin{enumerate}
  \item If your work uses existing assets, did you cite the creators?
    \answerYes{We cited D4RL \cite{d4rl-2020} and EDAC\cite{an2021uncertainty} for their datasets and code. }
  \item Did you mention the license of the assets?
    \answerYes{}
  \item Did you include any new assets either in the supplemental material or as a URL?
    \answerYes{We included our code in the anonymized link.}
  \item Did you discuss whether and how consent was obtained from people whose data you're using/curating?
    \answerYes{Opensource code and dataset.}
  \item Did you discuss whether the data you are using/curating contains personally identifiable information or offensive content?
    \answerNA{}
\end{enumerate}

\item If you used crowdsourcing or conducted research with human subjects...
\begin{enumerate}
  \item Did you include the full text of instructions given to participants and screenshots, if applicable?
    \answerNA{}
  \item Did you describe any potential participant risks, with links to Institutional Review Board (IRB) approvals, if applicable?
    \answerNA{}{}
  \item Did you include the estimated hourly wage paid to participants and the total amount spent on participant compensation?
    \answerNA{}
\end{enumerate}

\end{enumerate}


\appendix

\section{Theoretical Analysis}
\label{appendix-theoretical}
In this section, we provide detailed theoretical analysis and proofs in linear MDPs \cite{lsvi-2020}.

\subsection{LSVI Solution}

In linear MDPs, we assume that the transition dynamics and reward function take the form of
\begin{equation}
\label{eq::pf_def_linearMDP}
\mathbb{P}_t(s_{t+1} \,|\, s_t, a_t) = \langle \psi(s_{t+1}), \phi(s_t, a_t) \rangle, \quad r(s_t, a_t) = \theta^\top \phi(s_t, a_t), \quad\forall(s_{t+1}, a_t, s_t)\in\mathcal{S}\times\mathcal{A}\times\mathcal{S},
\end{equation}
where the feature embedding $\phi: \cS\times\cA\mapsto \RR^d$ is known. We further assume that the reward function $r:\cS\times\cA\mapsto[0, 1]$ is bounded and the feature is bounded by $\|\phi\|_2 \leq 1$. 

Given the offline dataset $\cD$, the parameter $w_t$ can be solved in the closed-form by following the LSVI algorithm, which minimizes the following loss function,
\begin{equation}
\label{eq::appendix_OOD_LSVI}
\widehat w_t = \min_{w\in \RR^d} \sum^m_{i = 1}\bigl(\phi(s^i_t, a^i_t)^\top w - r(s^i_t, a^i_t)- V_{t+1}(s^i_{t+1})\bigr)^2 
\end{equation}
where $V_{t+1}$ is the estimated value function in the $(t+1)$-th step, and $y_t^i=r(s^i_t, a^i_t)+ V_{t+1}(s^i_{t+1})$ is the target of LSVI. The explicit solution to (\ref{eq::appendix_OOD_LSVI}) takes the form of
\begin{equation}
\label{eq::pf_tilde_w}
\widehat w_t = \Lambda^{-1}_t\sum^m_{i = 1}\phi(s^i_t, a^i_t)y_t^i,\quad {\rm where~}\Lambda_t = \sum^m_{i=1}\phi(s^i_t, a^i_t)\phi(s^i_t, a^i_t)^\top  
\end{equation}

\subsection{RORL Solution}
In RORL, since we introduce the conservative smoothing loss and the OOD loss to learn the $Q$ value function, the parameter $\widetilde w_t$ of RORL can be solved as follows:
\begin{equation}
\label{eq:simplified_problem}
\begin{aligned}
    \widetilde w_t =  \min_{w\in \mathcal{R}^d} & \Big[\sum_{i=1}^{m} \big(y_t^i-Q_{w}(s_t^i,a_t^i)\big)^2 +   \sum_{i=1}^{m} \frac{1}{|\mathbb{B}_d(s_t^i,\epsilon)|} \sum_{\hat{s}_t^i\in \mathcal{D}_{\text{ood}}(s_t^i)}  \big(Q_{w}(s_t^i,a_t^i) - Q_{w}(\hat{s}_t^i,a_t^i)\big)^2  + \\ & \sum_{(\hat s, \hat a, \hat y) \sim \mathcal{D}_{\text{ood}} } \big(\hat y - Q_{w}(\hat s,\hat a)\big)^2 \Big],
\end{aligned}
\end{equation}
which is a simplified learning objective for linear MDPs. The first term is the ordinary TD-error, the second term is the $Q$ value smoothing loss, and the third term is the additional OOD loss. The explicit solution of Eq.~\eqref{eq:simplified_problem} takes the following form by following LSVI:
\begin{equation}\label{eq::appendix_w_ood_solu}
    \widetilde w_t =\widetilde \Lambda_t^{-1} \Big( \sum_{i=1}^{m} \phi(s_t^i,a_t^i) y_t^i + \sum_{(\hat s, \hat a, \hat y) \sim \mathcal{D}_{\text{ood}} }\phi(\hat s,\hat a)  \hat y  \Big),
\end{equation}
where the covariance matrix $\widetilde \Lambda_t$ is defined as
\begin{equation}
\begin{aligned}
        \widetilde \Lambda_t = &\sum_{i=1}^{m} \phi(s_t^i,a_t^i)\phi(s_t^i,a_t^i)^\top + \sum_{(\hat s, \hat a) \sim \mathcal{D}_{\text{ood}}} \phi(\hat s_t,\hat a_t)\phi(\hat s_t,\hat a_t)^\top \\
        &+ \sum_{i=1}^{m} \frac{1}{|\mathbb{B}_d(s_t^i,\epsilon)|}  \sum_{\hat s_t^i \sim \mathcal{D}_{\text{ood}}(s_t^i)} \big[\phi(\hat s_t^i,a_t^i) - \phi(s_t^i,a_t^i)\big]\big[\phi(\hat s_t^i,a_t^i) - \phi(s_t^i,a_t^i)\big]^\top. 
\label{eq:covariance-rorl}
\end{aligned}
\end{equation}
We denote the first term of Eq.~\eqref{eq:covariance-rorl} as $\widetilde \Lambda_t^{\rm in}$, the second term as $\widetilde \Lambda_t^{\rm ood}$, and the third term as $\widetilde \Lambda_t^{\text{ood\_diff}}$.

\subsection{\texorpdfstring{$\xi$}{xi}-Uncertainty Quantifier}

\begin{theorem*}[Theorem 1 restate]
\label{app-tm:PD_matrix}
Assume $\exists i\in[1,m]$ the vector group of all $\hat s_t^i \sim \mathcal{D}_{\text{ood}}(s_t^i)$: $\{\phi(\hat s_t^i,a_t^i) - \phi(s_t^i,a_t^i)\}$ is full rank, then the covariance matrix $\widetilde \Lambda_t^{\rm{ood\_diff}}$ is positive-definite: $\widetilde \Lambda_t^{\rm{ood\_diff}} \succeq \lambda \cdot \mathrm{I}$ where $\lambda > 0$.
\end{theorem*}

\begin{proof}
For the $\widetilde \Lambda_t^{\rm{ood\_diff}}$ matrix (i.e., the third part in Eq.~\eqref{eq:covariance-rorl}), we denote the covariance matrix for a specific $i$ as $\Phi^i_t$. Then we have $\widetilde \Lambda_t^{\rm{ood\_diff}}=\sum_{i=1}^{m} \Phi^i_t$. In the following, we discuss the condition of positive-definiteness of $\Phi^i_t$. For the simplicity of notation, we omit the superscript and subscript of $s_t^i$ and $a_t^i$ for given $i$ and $t$. Specifically, we define
\begin{equation}\nonumber
\Phi^i_t = \frac{1}{|\mathbb{B}_d(s_t^i,\epsilon)|}  \sum_{\hat s_j \sim \mathcal{D}_{\text{ood}}(s)} \big[\phi(\hat s_j,a) - \phi(s,a)\big]\big[\phi(\hat s_j,a) - \phi(s,a)\big]^\top,
\end{equation}
where $j\in\{1,\ldots,N\}$ indicates we sample $|\mathbb{B}_d(s_t^i,\epsilon)|=N$ perturbed states for each $s$. For a nonzero vector $y\in \mathbb{R}^d$, we have
\begin{equation}\label{eq:app-semi-definite}
\begin{aligned}
y^\top \Phi^i_t y &= y^\top\left(\frac{1}{N} \sum_{j=1}^N \big(\phi(\hat s_j,a)-\phi(s,a)\big)\big(\phi(\hat s_j,a)-\phi(s,a)\big)^\top\right) y \\
&= \frac{1}{N} \sum_{j=1}^N y^\top \big(\phi(\hat s_j,a)-\phi(s,a)\big)\big(\phi(\hat s_j,a)-\phi(s,a)\big)^\top y \\
&=\frac{1}{N} \sum_{j=1}^N \left(\big(\phi(\hat s_j,a)-\phi(s,a)\big)^\top y \right)^2 \geq 0,
\end{aligned}
\end{equation}
where the last inequality follows from the observation that $\big(\phi(\hat s_j,a)-\phi(s,a)\big)^\top y$ is a scalar. Then $\Phi^i_t$ is always positive \textbf{semi-definite}.

In the following, we denote $z_j=\phi(\hat s_j,a)-\phi(s,a)$. Then we need to prove that the condition to make $\Phi^i_t$ positive \textbf{definite} is ${\rm rank}[z_1,\ldots,z_N]=d$, where $d$ is the feature dimension. Our proof follows contradiction. 

In Eq.~\eqref{eq:app-semi-definite}, when $y^\top \Phi^i_t y=0$ with a nonzero vector $y$, we have $z_j^\top y=0$ for all $j=1,\ldots,N$. Suppose the set $\{z_1,\ldots,z_N\}$ spans $\mathbb{R}^d$, then there exist real numbers $\{\alpha_1,\ldots,\alpha_N\}$ such that $y=\alpha_1  z_1 +\dots+\alpha_N z_N$. But we have $y^\top y=\alpha_1  z_1^\top y + \dots +\alpha_N z_N^\top y=\alpha_1 \times 0+\ldots+\alpha_N \times 0=0$, yielding that $y=\mathbf{0}$, which forms a contradiction.

Hence, if the set $\{z_1,\ldots,z_N\}$ spans $\mathbb{R}^d$, which is equivalent to ${\rm rank}[z_1,\ldots,z_N]=d$, then $\Phi^i_t$ is positive \textbf{definite}. Under the given conditions, we know that $\exists k\in[1,m]$, for any nonzero vector $y\in \mathbb{R}^d$, $y^\top \Phi^k_t y > 0$. We have $y^\top \widetilde \Lambda_t^{\rm{ood\_diff}} y = \sum_{i=1}^{m} y^\top \Phi^i_t y \geq y^\top \Phi^k_t y > 0$. Therefore, $\widetilde \Lambda_t^{\rm{ood\_diff}}$ is positive definite, which concludes our proof.
\end{proof}

\paragraph{Remark.}
As a special case, when (\romannumeral1) the size of $\mathbb{B}_d(s_t^i,\epsilon)$ is sufficient, (\romannumeral2) the dimension of states is the same as the feature $\phi(s,a)$ and $\phi(s,a)=s$ and (\romannumeral3) each dimension of the state perturbation $\hat s_t^i - s_t^i$ is independent, the matrix $\widetilde \Lambda_t^{\text{ood\_diff}}$ satisfies:
\begin{equation*}
    \widetilde \Lambda_t^{\text{ood\_diff}} = \sum_{i=1}^{m} \frac{1}{|\mathbb{B}_d(s_t^i,\epsilon)|}  \sum_{\hat s_t^i \sim \mathbb{B}_d(s_t^i,\epsilon)} (\hat s_t^i - s_t^i)(\hat s_t^i - s_t^i)^\top  \approx  \frac{m \epsilon^2}{3} \cdot \mathrm{I}.
\end{equation*}

When we use neural networks as the feature extractor, the assumption in the above Theorem needs (\romannumeral1) the size of samples $\mathbb{B}_d(s_t^i,\epsilon)$ is sufficient, and (\romannumeral2) the neural network maintains useful variability for state-action features. To obtain the second constraint, we require that the Jacobian matrix of $\phi(s,a)$ has full rank. Nevertheless, when we use a network as the feature embedding, such a condition can generally be met since the neural network has high randomness and nonlinearity, which results in the feature embedding with sufficient variability. Generally, we only need to
enforce a bi-Lipschitz continuity for the feature embedding. We denote $x_1=(s_1,a)$ and $x_2=(s_2,a)$ as two different inputs. $x^k_1$ is the $k$-th dimension of $x_1$. The bi-Lipschitz constraint can be formed as
\begin{equation}
C_1\|x^k_1-x^k_2\|_\mathcal{X}\leq \|\phi(x_1)-\phi(x_2)\|_\Phi\leq C_2\|x^k_1-x^k_2\|_\mathcal{X}, \quad \forall k\in (1,|\mathcal{X}|),
\end{equation}
where $C_1<C_2$ are two positive constants. The lower-bound $C_1$ ensures the features space has enough variability for perturbed states, and the upper-bound can be obtained by Spectral regularization \cite{gogianu2021spectral} that makes the network easy to coverage. An approach to obtain bi-Lipschitz continuity is to regularize the norm of the gradients by using the gradient penalty as 
\[
\mathcal{L}_{\rm bilip}=\mathbb{E}_x\big[ \big(\min\big(\|
\nabla_{x^k} \phi(x)\|-C_1,0)\big)^2 + \big(\max\big(\|
\nabla_{x^k} \phi(x)\|-C_2,0)\big)^2
\big],\quad \forall k\in (1,|\mathcal{X}|).
\]

In experiments, we do not use explicit constraints (e.g., Spectral regularization) for the upper bound since the state has relatively low dimensions, and we find a small fully connected network does not resulting in a large $C_2$ empirically.

Recall the covariance matrix of PBRL is $\widetilde \Lambda_t^{\text{PBRL}}=\widetilde \Lambda_t^{\text{in}}+\widetilde \Lambda_t^{\text{ood}}$, and RORL has a covariance matrix as $\widetilde \Lambda_t=\widetilde \Lambda_t^{\text{PBRL}}+\widetilde \Lambda_t^{\text{ood\_diff}}$, we have the following corollary based on Theorem \ref{tm:PD_matrix}.

\begin{corollary*}[Corollary 1 restate]
\label{app-coro:positive-def}
Under the linear MDP assumptions and conditions in Theorem \ref{tm:PD_matrix}, we have $\widetilde \Lambda_t \succeq \widetilde \Lambda_t^{\rm PBRL}$. Further, the covariance matrix $\widetilde \Lambda_t$ of RORL is positive-definite: $\widetilde \Lambda_t \succeq \lambda \cdot \mathrm{I}$, where $\lambda > 0$.
\end{corollary*}

Recent theoretical analysis shows that an appropriate uncertainty quantification is essential for provable efficiency in offline RL \cite{jin2021pessimism,xie2021bellman,bai2021pessimistic}. Pessimistic Value Iteration \cite{jin2021pessimism} defines a general $\xi$-uncertainty quantifier as the penalty and achieves provable efficient pessimism in offline RL. We give the definition of a $\xi$-uncertainty quantifier as follows.

\newtheorem{definition}{Definition}
\begin{definition}[$\xi$-Uncertainty Quantifier \cite{jin2021pessimism}]
The set of penalization $\{\Gamma_t\}_{t\in[T]}$ forms a $\xi$-Uncertainty Quantifier if it holds with probability at least $1 - \xi$ that
\begin{equation*}
|\widehat \cT V_{t+1}(s, a) - \cT V_{t+1}(s, a)| \leq \Gamma_t(s, a)
\end{equation*}
for all $(s, a)\in\cS\times\cA$, where $\cT$ is the Bellman operator and $\widehat \cT$ is the empirical Bellman operator that estimates $\cT$ based on the data.
\end{definition}

In linear MDPs, Lower Confidence Bound (LCB)-penalty \cite{bandit-2011,lsvi-2020} is known to be a $\xi$-uncertainty quantifier for appropriately selected $\beta_t$ as $\Gamma^{\rm lcb}(s_t,a_t)=\beta_t\cdot\big[\phi(s_t,a_t)^\top\Lambda_t^{-1}\phi(s_t,a_t)\big]^{\nicefrac{1}{2}}$. Following the analysis of PBRL \cite{bai2021pessimistic}, since the bootstrapped uncertainty is an estimation of the LCB-penalty, the proposed RORL also form a valid $\xi$-uncertainty quantifier with the covariance matrix $\widetilde \Lambda_t \succeq \lambda \cdot \mathrm{I}$ given in Corollary \ref{coro:positive-def}.

\begin{theorem}
\label{app-thm:rorl-lcb}
For all the OOD datapoint $(\hat{s}, \hat{a}, \hat{y})\in \mathcal{D}_{\rm ood}$, if we set $\hat{y} = \mathcal{T} V_{t+1}(s^{\text{\rm ood}}, a^{\text{\rm ood}})$, it then holds for $\beta_t=\mathcal{O}\bigl(T\cdot \sqrt{d} \cdot \text{log}(T/\xi)\bigr)$ that
\begin{equation}
\label{app-eq:lcb-def}
\Gamma^{\rm lcb}_t(s_t,a_t)=\beta_t\big[\phi(s_t,a_t)^\top \widetilde \Lambda_t^{-1}\phi(s_t,a_t)\big]^{\nicefrac{1}{2}}
\end{equation}
forms a valid $\xi$-uncertainty quantifier, where $\widetilde \Lambda_t$ is the covariance matrix of RORL.
\end{theorem}

\begin{proof}
The proof follows that of the analysis of PBRL \cite{bai2021pessimistic} in linear MDPs \cite{jin2021pessimism}. We define the empirical Bellman operator of RORL as $\widetilde\cT$, then
\begin{equation*}
\widetilde\cT V_{t+1}(s_t, a_t) = \phi(s_t, a_t)^\top \widetilde w_t,
\end{equation*}
where $\widetilde w_t$ follows the solution in Eq.~\eqref{eq::appendix_w_ood_solu}. Then it suffices to upper bound the following difference between the empirical Bellman operator and Bellman operator
\begin{equation*}
\cT V_{t+1}(s, a) - \widetilde\cT V_{t+1}(s, a) = \phi(s, a)^\top (w_t - \widetilde w_t).
\end{equation*}
Here we define $w_t$ as follows
\begin{equation}
\label{eq::pf_def_wt_pess}
w_t = \theta + \int_{\cS} V_{t+1}(s_{t+1})\psi(s_{t+1}) \text{d} s_{t+1},
\end{equation}
where $\theta$ and $\psi$ are defined in Eq.~(\ref{eq::pf_def_linearMDP}). 
It then holds that
\begin{align}
\label{eq::pf_ood_0}
\cT V_{t+1}(s, a) - \widetilde \cT V_{t+1}(s, a) &= \phi(s, a)^\top (w_t - \widetilde w_t)\notag\\
&= \phi(s, a)^\top w_t - \phi(s, a)^\top \widetilde \Lambda^{-1}_t\sum^m_{i = 1}\phi(s^i_t, a^i_t)\bigl( r(s^i_t, a^i_t) + V^i_{t+1}(s^i_{t+1})\bigr) \notag\\
&\quad - \phi(s, a)^\top \widetilde \Lambda^{-1}_t\sum_{(\hat{s}, \hat{a}, \hat{y})\in \cD_{\text{ood}}}\phi(\hat{s}, \hat{a})\hat{y}. 
\end{align}
where we plug the solution of $\widetilde w_t$ in Eq.~\eqref{eq::appendix_w_ood_solu}. Meanwhile, by the definitions of $\widetilde \Lambda_t$ and $w_t$ in Eq.~\eqref{eq:covariance-rorl} and Eq.~\eqref{eq::pf_def_wt_pess}, respectively, we have
\begin{equation}
\begin{aligned}
\label{eq::pf_ood_1}
\phi&(s, a)^\top w_t = \phi(s, a)^\top \widetilde \Lambda_t^{-1} \widetilde \Lambda_t w_t\\
&=\phi(s, a)^\top \widetilde \Lambda_t^{-1} \biggl(\sum^m_{i = 1}\phi(s^i_t, a^i_t) \cT V_{t+1}(s_t, a_t) + \sum_{(\hat{s}, \hat{a}, \hat{y})\in \cD_{\text{ood}}}\phi(\hat{s}, \hat{a}) \cT V_{t+1}(\hat{s}, \hat{a})+\\
&\sum_{i=1}^{m} \frac{1}{|\mathbb{B}_d(s_t^i,\epsilon)|}  \sum_{\hat s_t^i \sim \mathcal{D}_{\text{ood}}(s_t^i)} [\phi(\hat s_t^i,a_t^i) - \phi(s_t^i,a_t^i)]\big[\phi(\hat s_t^i,a_t^i) - \phi(s_t^i,a_t^i)\big]^\top w_t \biggr).
\end{aligned}
\end{equation}
Plugging Eq.~\eqref{eq::pf_ood_1} into Eq.~\eqref{eq::pf_ood_0} yields
\begin{equation}
\label{eq::pf_ood_2}
\cT V_{t+1}(s, a) - \widetilde \cT V_{t+1}(s, a) = \text{(i)} + \text{(ii)}+\text{(iii)},
\end{equation}
where we define
\begin{align*}
\text{(i)} &= \phi(s, a)^\top \widetilde \Lambda^{-1}_t\sum^m_{i = 1}\phi(s^i_t, a^i_t)\bigl( \cT V_{t+1}(s^i_t, a^i_t) - r(s^i_t, a^i_t) - V^i_{t+1}(s^i_{t+1})\bigr),\\
\text{(ii)} &= \phi(s, a)^\top \widetilde \Lambda_t^{-1}\sum_{(\hat{s}, \hat{a}, \hat{y})\in \cD_{\text{ood}}}\phi(\hat{s}, \hat{a}) \bigl(\cT V_{t+1}(\hat{s}, \hat{a}) - \hat{y}\bigr),\\
\text{(iii)} &= \phi(s, a)^\top \widetilde \Lambda_t^{-1} \sum_{i=1}^{m} \frac{1}{|\mathbb{B}_d(s_t^i,\epsilon)|}  \sum_{\hat s_t^i \sim \mathcal{D}_{\text{ood}}(s_t^i)} \Big[\Big(\phi(\hat s_t^i,a_t^i) \phi(\hat s_t^i,a_t^i)^\top w_t - \phi(\hat s_t^i,a_t^i) \phi(s_t^i,a_t^i)^\top w_t \Big) \\
& \qquad \qquad \qquad \qquad + \Big(\phi(s_t^i,a_t^i) \phi(s_t^i,a_t^i)^\top w_t - \phi(s_t^i,a_t^i) \phi(\hat s_t^i,a_t^i)^\top w_t \Big)\Big].
\end{align*}
Following the standard analysis based on the concentration of self-normalized process \cite{bandit-2011, azar2017minimax, wang2020reward, lsvi-2020, jin2021pessimism} and the fact that $\Lambda_{\text{ood}} \succeq \lambda\cdot I$, it holds that
\begin{equation}
|\text{(i)}| \leq \beta_t \cdot\big[\phi(s_t,a_t)^\top\Lambda_t^{-1}\phi(s_t,a_t)\big]^{\nicefrac{1}{2}},
\end{equation}
with probability at least $1 - \xi$, where $\beta_t = \mathcal{O}\bigl(T\cdot \sqrt{d} \cdot \text{log}(T/\xi)\bigr)$. Meanwhile, by setting $y = \cT V_{t+1} (s^{\text{ood}}, a^{\text{ood}})$, it holds that $\text{(ii)} = 0$. For $\text{(iii)}$, we have 
\begin{equation}
\begin{aligned}
&\Big(\phi(\hat s_t^i,a_t^i) \phi(\hat s_t^i,a_t^i)^\top w_t - \phi(\hat s_t^i,a_t^i) \phi(s_t^i,a_t^i)^\top w_t \Big) + \Big(\phi(s_t^i,a_t^i) \phi(s_t^i,a_t^i)^\top w_t - \phi(s_t^i,a_t^i) \phi(\hat s_t^i,a_t^i)^\top w_t \Big)\\
& = \phi(\hat s_t^i,a_t^i) \Big(\cT V_{t+1}(\hat s_t^i,a_t^i) - \cT V_{t+1}(s_t^i,a_t^i) \Big) + \phi(s_t^i,a_t^i) \Big(\cT V_{t+1}(s_t^i,a_t^i) - \cT V_{t+1}(\hat s_t^i,a_t^i) \Big) \\
& = \big(\phi(\hat s_t^i,a_t^i) - \phi(s_t^i,a_t^i)\big) \big(\cT V_{t+1}(\hat s_t^i,a_t^i) - \cT V_{t+1}(s_t^i,a_t^i) \big)
\end{aligned}
\end{equation}
Since we enforce smoothness for the value function, we have $\cT V_{t+1}(\hat s_t^i,a_t^i) \approx \cT V_{t+1}(s_t^i,a_t^i)$. Thus $\text{(iii)}\approx 0$. To conclude, we obtain from Eq.~\eqref{eq::pf_ood_2} that 
\begin{equation}
|\cT V_{t+1}(s, a) - \widetilde \cT V_{t+1}(s, a)| \leq \beta_t \cdot\big[\phi(s_t,a_t)^\top\Lambda_t^{-1}\phi(s_t,a_t)\big]^{\nicefrac{1}{2}}
\end{equation}
for all $(s, a)\in\cS\times\cA$ with probability at least $1 - \xi$.
\end{proof}

\subsection{Suboptimality Gap}

Theorem \ref{app-thm:rorl-lcb} allows us to further characterize the optimality gap based on the pessimistic value iteration \cite{jin2021pessimism}. First, we give the following lemma.

\begin{lemma}\label{lemma:1}
Given two positive definite matrix A and B, it holds that: \begin{equation}
\frac{x^\top A^{-1} x}{x^\top (A+B)^{-1} x} > 1.
\end{equation}
\end{lemma}
\begin{proof}
Leveraging the properties of generalized Rayleigh quotient, we have
\begin{equation}
\frac{x^\top A^{-1} x}{x^\top (A+B)^{-1} x} \geq \lambda_{\text{min}}\big((A+B)A^{-1}\big) = \lambda_{\text{min}}\big( \mathrm{I} +BA^{-1}\big) = 1 + \lambda_{\text{min}}\big(BA^{-1}\big).
\end{equation}
Since $B$ and $A^{-1}$ are both positive definite, the eigenvalues of $BA^{-1}$ are all positive: $\lambda_{\text{min}}\big(BA^{-1}\big) > 0$. This ends the proof.
\end{proof}

Then, according to the definition of LCB-penalty in Eq.~\eqref{app-eq:lcb-def}, since $\widetilde \Lambda_t=\widetilde \Lambda_t^{\text{PBRL}}+\widetilde \Lambda_t^{\text{ood\_diff}}$ with $\widetilde \Lambda_t^{\text{ood\_diff}} \succeq \lambda \mathrm{I}$. we have the relationship of the LCB-penalty between RORL and PBRL as follows.

\begin{corollary}
\label{app-coro-lcb-two}
Suppose $\Lambda_t^{\text{PBRL}}$ is positive definite. The RORL-induced LCB-penalty term is less than the PBRL-induced LCB-penalty, as
$\Gamma_t^{\rm lcb}(s_t,a_t) = \beta_t \big[\phi(s_t,a_t)^\top \widetilde \Lambda_t^{-1} \phi(s_t,a_t) \big]^{1/2} < \Gamma_t^{\rm lcb\_PBRL}(s_t,a_t)$.
\end{corollary}
\begin{proof}
Since $\widetilde \Lambda_t=\widetilde \Lambda_t^{\text{PBRL}}+\widetilde \Lambda_t^{\text{ood\_diff}}$ and $\widetilde \Lambda_t^{\text{ood\_diff}}\succeq \lambda I$, we have 
\begin{equation}
\frac{\phi(s_t,a_t)^\top \widetilde \Lambda_t^{-1} \phi(s_t,a_t)}{\phi(s_t,a_t)^\top  (\widetilde \Lambda_t^{\text{PBRL}})^{-1} \phi(s_t,a_t)}=\frac{\phi(s_t,a_t)^\top  (\widetilde \Lambda_t^{\text{PBRL}}+\Lambda_t^{\text{ood\_diff}})^{-1} \phi(s_t,a_t)}{\phi(s_t,a_t)^\top  (\widetilde\Lambda_t^{\text{PBRL}})^{-1} \phi(s_t,a_t)}<1.
\end{equation}
where the inequality directly follows Lemma \ref{lemma:1}. Then we have 
\begin{equation}
\phi(s_t,a_t)^\top \widetilde \Lambda_t^{-1} \phi(s_t,a_t)<\phi(s_t,a_t)^\top  (\widetilde \Lambda_t^{\text{PBRL}})^{-1} \phi(s_t,a_t).
\end{equation}
\end{proof}

Theorem \ref{app-thm:rorl-lcb} and Corollary \ref{app-coro-lcb-two} allow us to further characterize the optimality gap of the pessimistic value iteration. In particular, we have the following suboptimality gap under linear MDP assumptions.
\begin{corollary*}[Corollary \ref{cor::opt_gap_rorl} restate]
\label{app-cor::opt_gap_rorl}
Under the same conditions as Theorem \ref{app-thm:rorl-lcb}, it holds that
${\rm SubOpt} (\pi^*, \hat \pi) \leq \sum_{t=1}^{T} \mathbb{E}_{\pi^*} \big[ \Gamma_t^{\rm lcb}(s_t,a_t) \big] < \sum_{t=1}^{T} \mathbb{E}_{\pi^*} \big[ \Gamma_t^{\rm lcb\_PBRL}(s_t,a_t) \big]$.
\end{corollary*}
We refer to Jin et al \cite{jin2021pessimism} for a detailed proof of the first inequality. The second inequality is directly induced by $\Gamma_t^{\rm lcb}(s_t,a_t) < \Gamma_t^{\rm lcb\_PBRL}(s_t,a_t)$ in Corollary \ref{app-coro-lcb-two}. The optimality gap is information-theoretically optimal under the linear MDP setup with finite horizon \cite{jin2021pessimism}. 
Therefore, RORL enjoys a tighter suboptimality bound than PBRL \cite{bai2021pessimistic} in linear MDPs.

\section{Implementation Details and Experimental Settings}
\label{ap-implementation}
In this section, we provide detailed implementation and experimental settings.

\subsection{Implementation Details}
\label{ap-implementation_details}
\paragraph{SAC-10}
Our SAC-10 implementation is based on~\cite{an2021uncertainty}, which is open-source. We keep the default parameters as EDAC \cite{an2021uncertainty} except for the ensemble size set to 10 in our paper. In addition, we normalize each dimension of observations to a standard normal distribution for consistency with RORL. The hyper-parameters are listed in Table \ref{tab:hyper-SAC10}.

\begin{table}[h!]
\small
  \caption{Hyper-parameters of SAC-10}
  \vspace{3pt}
  \label{tab:hyper-SAC10}
  \centering
  \begin{tabular}{p{0.55\columnwidth}p{0.4\columnwidth}}
    \toprule
    Hyper-parameters & Value\\
    \midrule
    The number of bootstrapped networks $K$          & 10  \\
    Policy network  & FC(256,256,256) with ReLU activations\\
    $Q$-network  & FC(256,256,256) with ReLU activations\\
    Target network smoothing coefficient $\tau$ for every training step  & 5e-3 \\
    Discount factor $\gamma$ & 0.99 \\
    Policy learning rate & 3e-4 \\
    $Q$ network learning rate & 3e-4 \\
	Optimizer & Adam  \\
	Automatic Entropy Tuning & True \\
	batch size & 256 \\
     \bottomrule
  \end{tabular}
\end{table}

\paragraph{EDAC} 
Our EDAC implementation is based on the open-source code of the original paper \cite{an2021uncertainty}. In the benchmark results, we directly report results from the paper which are the previous SOTA performance on the D4RL Mujoco benchmark. As for other experiments, we also normalize the observations and use 10 ensemble $Q$ networks for consistency with RORL, and set the gradient diversity term $\eta=1$ by default.

\paragraph{RORL} 
We implement RORL based on SAC-10 and keep the hyper-parameters the same. The differences are the introduced policy and $Q$ network smoothing techniques and the additional value underestimation on OOD state-action pairs. In Eq.~\eqref{eq:roal-Q}, the coefficient $\beta_{\rm Q}$ for the $Q$ network smoothing loss $\mathcal{L}_{\rm smooth}$ is set to $0.0001$ for all tasks, and the coefficient $\beta_{\rm ood}$ for the OOD loss $\mathcal{L}_{\rm ood}$ is tuned within $\{0.0,0.1,0.5 \}$. Besides, the coefficient $\beta_{\rm P}$ of the policy smoothing loss in Eq.~\eqref{eq:roal-policy} is searched in $\{0.1, 1.0 \}$. When training the policy and value functions in RORL, we randomly sample $n$ perturbed observations from a $l_{\infty}$ ball of norm $\epsilon$ and select the one that maximizes $D_{\rm J}(\pi_{\theta}(\cdot |s)\|\pi_{\theta}(\cdot|\hat s))$ or $\mathcal{L}_{\rm smooth}$, respectively. We denote the perturbation scales for the $Q$ value functions, the policy, and the OOD loss as $\epsilon_{\rm Q}$, $\epsilon_{\rm P}$ and $\epsilon_{\rm ood}$. The number of sampled perturbed observations $n$ is tuned within $\{10, 20\}$. The OOD loss underestimates the values for $n$ perturbed states $\hat s \sim \mathbb{B}_d(s,\epsilon)$ with actions sampled from the current policy $\hat a \sim \pi_\theta(\hat{s})$. For each $\hat s$, we sample a single $\hat a$ for the OOD loss. Regarding the $Q$ smoothing loss in Eq.~\eqref{eq:conservative_smoothing_Q}, the parameter $\tau$ is set to $0.2$ in all tasks for conservative value estimation. All the hyper-parameters used in RORL for the benchmark experiments and adversarial experiments are listed in Table \ref{tab:bench_params} and Table \ref{tab:attack_params} respectively. Note that for halfcheetah tasks, 10 ensemble $Q$ networks already enforce sufficient pessimism for OOD state-action pairs, thus we do not need additional OOD loss for these tasks.

As for the OOD loss $\mathcal{L}_{\rm ood}$ in Eq.~\eqref{eq:ood_loss}, we remark that the pseudo-target $\widehat{\mathcal{T}}_{\rm ood}Q_{\phi_i}(\hat{s},\hat{a})$ for the OOD state-action pairs $(\hat s, \hat a)$ can be implemented in two ways: $\widehat{\mathcal{T}}_{\rm ood}Q_{\phi_i}(\hat{s},\hat{a}):=Q_{\phi_i}(\hat{s},\hat{a})- \lambda u(\hat{s},\hat{a})$ and $\widehat{\mathcal{T}}_{\rm ood}Q_{\phi_i}(\hat{s},\hat{a}):= {\rm min}_{i=1,\ldots, K} Q_{\phi_i}(\hat{s},\hat{a})$. We refer to the two targets as the ``minus target'' and the ``min target'', and compare them in Appendix \ref{ap:min_target}. Intuitively, the ``minus target'' introduces an additional parameter $\lambda$ but is more flexible to tune for different environments and different types of data. In contrast, the ``min target'' requires tuning the number of ensemble $Q$ networks and cannot enforce appropriate conservatism for all tasks given only 10 ensemble $Q$ networks.
Following PBRL \cite{bai2021pessimistic}, we also decay the OOD regularization coefficient $\lambda$ with decay pace $d$ for each training step to stabilize $\mathcal{L}_{\rm ood}$, because we need strong OOD regularization at the beginning of training and need to avoid too large OOD loss that leads the value function to be fully negative. $\lambda$ and $d$ are also listed in the two tables.

\begin{table}[h]
	\centering
	\small
	\caption{Hyper-parameters of RORL for the benchmark results}
	\vspace{0.2em}
	\label{tab:bench_params}
	\begin{adjustbox}{max width=\columnwidth}
		\begin{tabular}{l|c|c|c|c|c|c|c|c|c}
			\toprule
			\textbf{Task Name} & $\beta_{\rm Q}$ & $\beta_{\rm P}$ & $\beta_{\rm ood}$ & $\epsilon_{\rm Q}$  & $\epsilon_{\rm P}$  & $\epsilon_{\rm ood}$ & $\tau$ & $n$ & $\lambda\ (d)$ \\
			\midrule
			halfcheetah-random & \multirow{5}{*}{0.0001} &  \multirow{5}{*}{0.1} & \multirow{5}{*}{0.0} & 0.001 & 0.001 & \multirow{5}{*}{0.00} & \multirow{5}{*}{0.2} & 20 & \multirow{5}{*}{0}  \\
			halfcheetah-medium &    &   &    & 0.001 & 0.001 &   &   & 10 &  \\
			halfcheetah-medium-expert &    &    &    & 0.001 & 0.001 &    &   & 10 &  \\
			halfcheetah-medium-replay &    &   &   & 0.001 & 0.001 &    &   & 10 &  \\
			halfcheetah-expert &    &   &   & 0.005 & 0.005 &     &   & 10 &  \\
			\midrule
			hopper-random & \multirow{5}{*}{0.0001} & \multirow{5}{*}{0.1} & \multirow{5}{*}{0.5} & \multirow{5}{*}{0.005} & \multirow{5}{*}{0.005} & \multirow{5}{*}{0.01} & \multirow{5}{*}{0.2}  & \multirow{5}{*}{20} & 1$\rightarrow 0.5$ ($1e^{-6}$) \\
			hopper-medium &   &   &   &   &   &  &   &  & 2$\rightarrow 0.1$ ($1e^{-6}$) \\
			hopper-medium-expert &   &  &   &   &    &  &   &   & 3$\rightarrow 1.0$ ($1e^{-6}$) \\
			hopper-medium-replay &   &  &   &   &    &  &   &   & 0.1$\rightarrow 0$ ($1e^{-6}$) \\
			hopper-expert &   &  &   &   &    &  &   &   & 4$\rightarrow 1$ ($1e^{-6}$) \\
			\midrule
			walker2d-random & \multirow{5}{*}{0.0001}  & \multirow{5}{*}{1.0} & 0.5 & 0.005 & 0.005 & \multirow{5}{*}{0.01}  & \multirow{5}{*}{0.2} & \multirow{5}{*}{20} & 5.0$\rightarrow 0.5$ ($1e^{-5}$) \\
			walker2d-medium &   &  & 0.1 & 0.01 & 0.01 &    &   &  & 0.1$\rightarrow 0.1$ (0.0) \\
			walker2d-medium-expert &    &  & 0.1 & 0.01 & 0.01 &    &   &  & 0.1$\rightarrow 0.1$ (0.0) \\
			walker2d-medium-replay &   &  & 0.1 & 0.01 & 0.01 &    &   &  & 0.1$\rightarrow 0.1$ ($0.0$) \\
			walker2d-expert &    &   & 0.5 & 0.005 & 0.005 &    &   &  & 1.0$\rightarrow 0.7$ ($1e^{-6}$) \\
			\bottomrule
		\end{tabular}
	\end{adjustbox}
\end{table}

\begin{table}[h]
	\centering
	\small
	\caption{Hyper-parameters of RORL for the adversarial attack results}
	\vspace{0.2em}
	\label{tab:attack_params}
	\begin{adjustbox}{max width=\columnwidth}
		\begin{tabular}{l|c|c|c|c|c|c|c|c|c}
			\toprule
			\textbf{Task Name} & $\beta_{\rm Q}$ & $\beta_{\rm P}$ & $\beta_{\rm ood}$ & $\epsilon_{\rm Q}$  & $\epsilon_{\rm P}$  & $\epsilon_{\rm ood}$ & $\tau$ & $n$ & $\lambda\  (d)$ \\
			\midrule
			halfcheetah-medium &  \multirow{3}{*}{0.0001}  &  1  &  0.0 & 0.03 & 0.05 & 0.00  & \multirow{3}{*}{0.2}  & \multirow{3}{*}{20} & 0 \\
			walker2d-medium &   &  0.5 & 0.5 & 0.03 &  0.07 &  0.03  &   &  & 1$\rightarrow 0.1$ ($1e^{-6}$) \\
			hopper-medium &   &  0.1 & 0.5 & 0.01 & 0.01  &  0.03  &   &  & 2$\rightarrow 0.1$ ($1e^{-6}$) \\
			\bottomrule
		\end{tabular}
	\end{adjustbox}
\end{table}

\subsection{Experimental Settings}
\label{ap:experiment-setting}
For all experiments, we train algorithms for 3000 epochs (1000 training steps per epoch, i.e., 3 million steps in total) following EDAC \cite{an2021uncertainty}. We use small perturbation scales to train the $Q$ networks and the policy network for the benchmark experiments and relatively large scales for the adversarial attack experiments as listed in Table \ref{tab:bench_params} and Table \ref{tab:attack_params}. 

In the benchmark results, we evaluate algorithms for 1000 steps in clean environments (without adversarial attack) at the end of each epoch. The reported results are normalized to d4rl scores that measure how the performance compared with the expert score and the random score: ${\rm the normalized\ score}=100 \times \frac{\rm score - random\ score}{\rm expert\ score - random\ score}$. Besides, the benchmark results are averaged over 4 random seeds. Regarding the adversarial attack experiments, we evaluate algorithms in perturbed environments that performing ``random'', ``action diff'', and ``min Q'' attack with zeroth-order and mixed-order optimizations as discussed in Section \ref{sec:adver_attak}. Similar to prior work \cite{zhang2020robust}, agents receive observations with malicious noise and the environments do not change their internal transition dynamics. We evaluate each algorithm for 10 trajectories (1000 steps per trajectory) and average their returns over 4 random seeds.

\subsection{Visualization Settings of CQL}
\label{ap:visual_cql}
For visualizing the relationship between the $Q$-function and the state space (i.e., Figure ~\ref{fig:cql-attack} and Figure ~\ref{fig:cql-attack-minQ}), we sample 2560 adversarial transitions from the offline dataset for each attack $\epsilon$ and calculate the corresponding $Q$-function. Since the state has relatively high dimensions (i.e., 11 or 17), we perform PCA dimensional reduction to reduce the state to 4 dimensions. We find the $Q$-function generally has a strong correlation to one or two dimensions of the state after dimensional reduction. For other dimensions, the relationship between the $Q$-value and the PCA-reduced state often has one or two peaks, which has less variety in the curve.

\begin{figure*}[ht]
\centering
\subfigure[$Q$-function of CQL]{\includegraphics[width=0.33\textwidth]{ 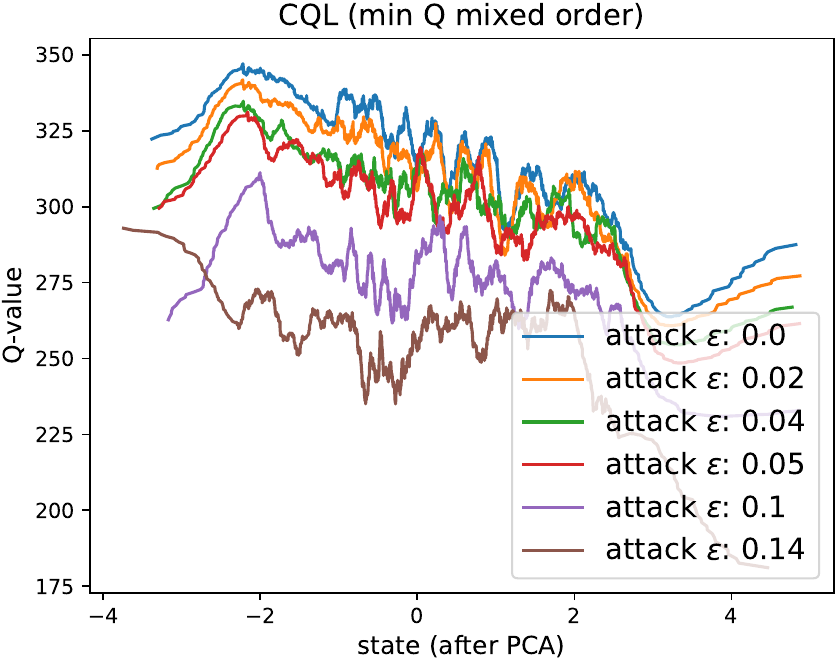}\label{fig:cql-non-smooth-minQ}}
\subfigure[$Q$-function of CQL-smooth]{\includegraphics[width=0.33\textwidth]{ 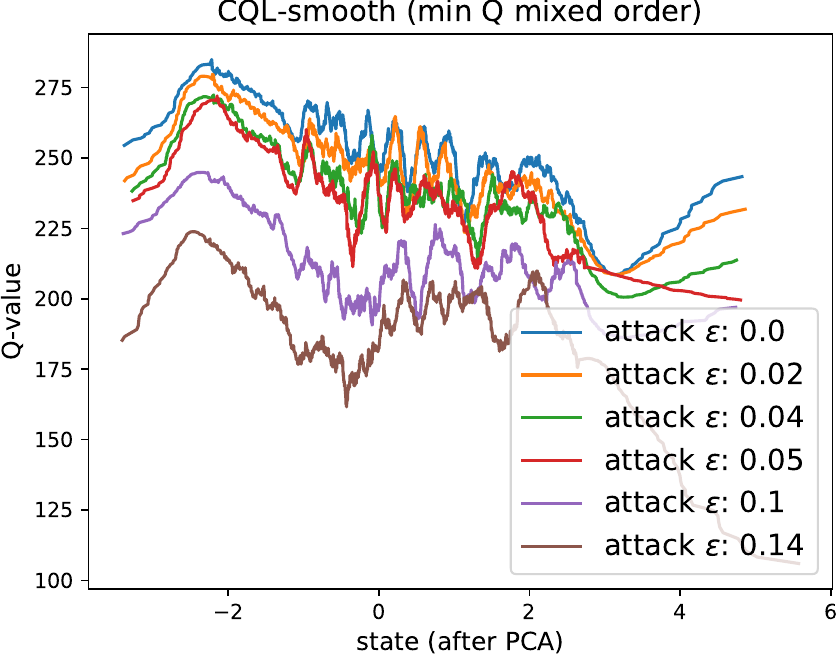}\label{fig:cql-smooth-minQ}}
\subfigure[Final performance]{\includegraphics[width=0.3\textwidth]{ 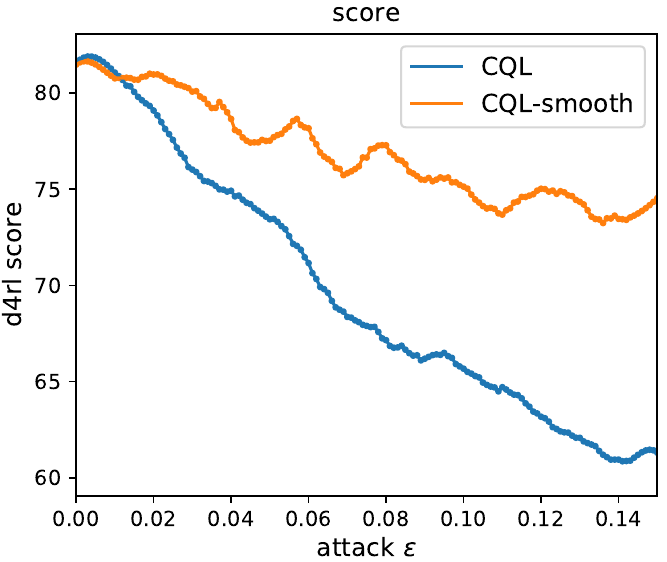}\label{fig:cql-score-minQ}}
\caption{(a)(b) The $Q$-functions of $\hat{s}$ with `min $Q$ mixed order' adversarial noises in CQL and CQL-smooth, respectively. The same moving average factor is used in plotting both figures. (c) The performance evaluation of CQL and CQL-smooth with different perturbation scales. We use 100 different $\epsilon\in[0.0,0.15]$ for the evaluation.}
\label{fig:cql-attack-minQ}
\end{figure*}

\section{Additional Experimental Results}
\label{ap:addtional_exp}
In this section, we present additional ablation studies and adversarial experiments.


\begin{figure*}[h]
\centering
\subfigure[Average epoch time of RORL's components]{\includegraphics[width=0.48\textwidth, trim=20 5 20 20, clip]{ 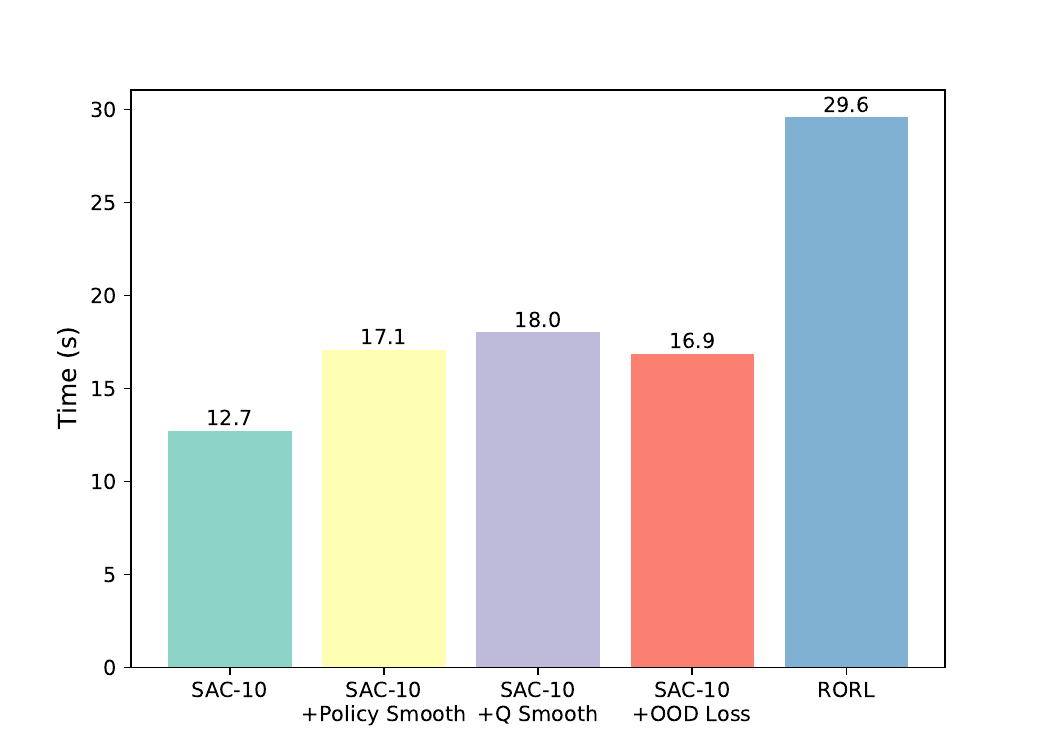}\label{fig:time_rorl}}
\subfigure[Memory usage of RORL's components]{\includegraphics[width=0.48\textwidth, trim=40 40 40 40, clip]{ 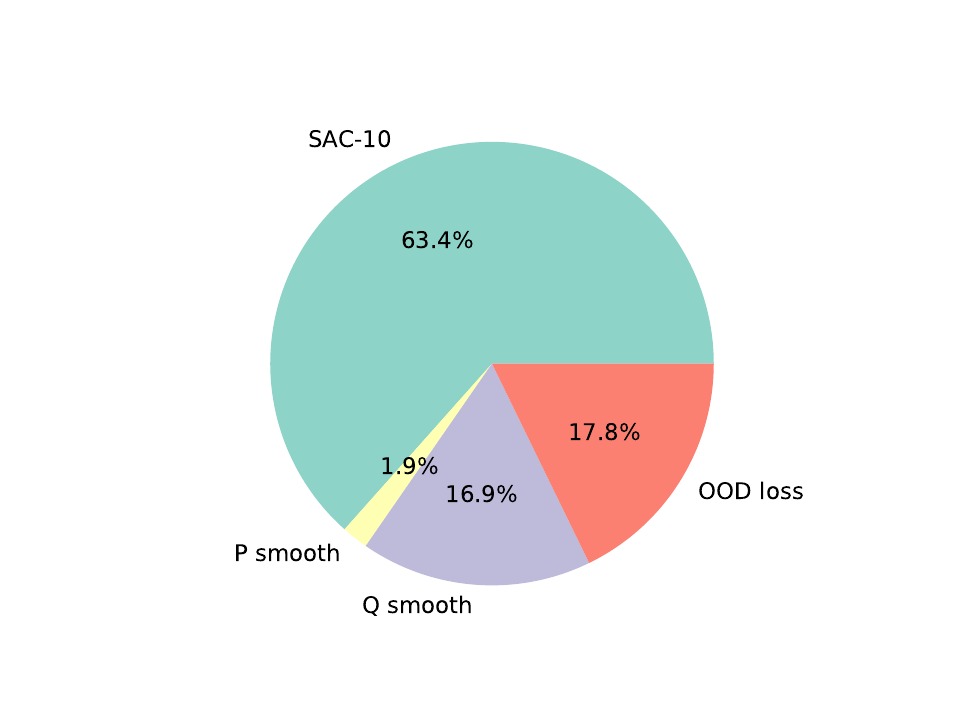}\label{fig:memory_rorl}}
\caption{Visualization of the average epoch time and memory usage for RORL and its components. }
\label{fig:time_memory}
\end{figure*}

\subsection{Computational Cost Comparison}
\label{ap:computational_cost}
In this subsection, we compare the computational cost of RORL with prior works on a single machine with one GPU (Tesla V100 32G) and one CPU (Intel Xeon Platinum 8255C @ 2.50GHz). For each method, we measure the average epoch time (i.e., 1$\times 10^3$ training steps) and the GPU memory usage on the hopper-medium-v2 task. For CQL, PBRL, SAC-$N$, and EDAC, we evaluate the computational cost based on their official code.

As shown in Table \ref{tab:compute_cost}, RORL runs slightly faster than CQL, mainly because CQL needs the OOD action sampling and the logsumexp approximation. For ensemble-based baselines, RORL runs much faster than PBRL, requiring only 28.7$\%$ of PBRL's epoch time. PBRL is so slow because it uses 10 ensemble $Q$ networks for uncertainty measure and needs OOD action sampling for value underestimation. In RORL, we also include the OOD state-action sampling and additional adversarial training procedures, but we implement these procedures efficiently based on GPU operation and parallelism. Even so, RORL is still slower than SAC-10 and EDAC. But as demonstrated in our experiments, RORL enjoys significantly better robustness than EDAC and SAC-10 under different types of perturbations. As for the GPU memory consumption, RORL uses comparable memory
to PBRL and EDAC, with only $16.7\%$ more memory usage.

Furthermore, we analyze the computational cost of RORL's components ($Q$ smoothing, policy smoothing, and the OOD loss). Specifically, we measure the average epoch time of \emph{SAC-$10$+Policy Smooth}, \emph{SAC-$10$+Q Smooth}, \emph{SAC-$10$+OOD Loss} in Figure \ref{fig:time_rorl}, and calculate the corresponding memory usage of each component in Figure \ref{fig:memory_rorl}. For the training time, \emph{SAC-$10$+Q Smooth} runs the slowest and \emph{SAC-$10$+Policy Smooth} runs slightly slower than \emph{SAC-$10$+OOD Loss}. This is mainly because sampling the worst-case perturbation occupies the most time. In addition, since we use an ensemble of 10 $Q$ networks, the memory usage of the $Q$ smoothing loss and the OOD loss (both need to pass $n$ perturbed states to $10$ $Q$ networks) is larger than the policy smoothing loss.

\begin{figure}[htb]
\centering
\subfigure[]{\includegraphics[width=0.47\textwidth]{ 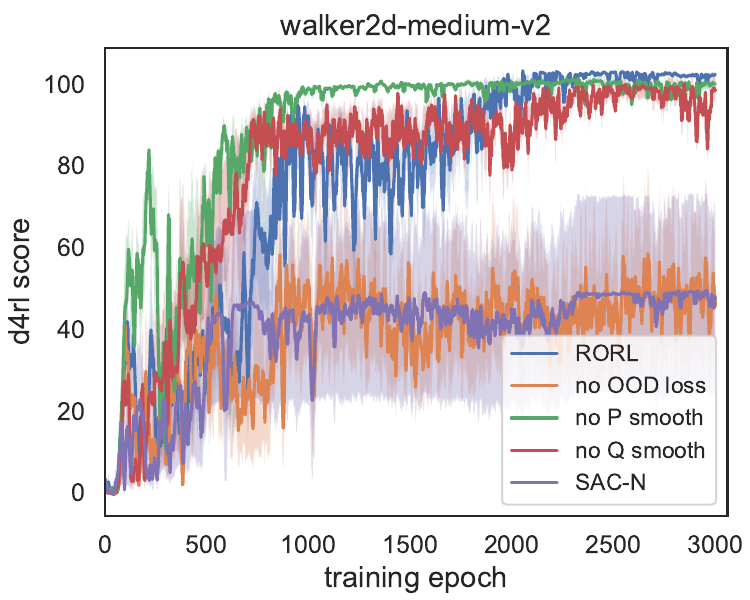}}
\hspace{1.0em}
\subfigure[]{\includegraphics[width=0.47\textwidth]{ 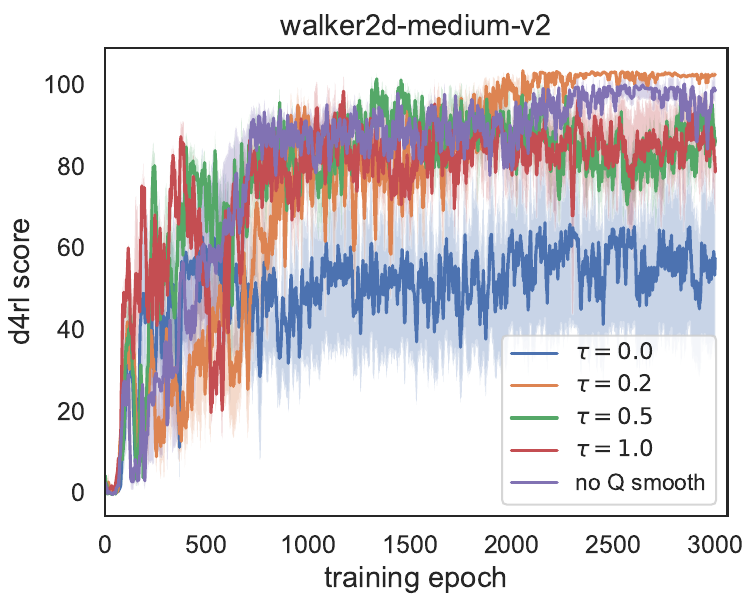}}
\caption{(a) Ablation studies of three introduced loss. The ``P smooth'' and the ``$Q$ smooth'' refer to the policy smoothing loss and the $Q$ network smoothing loss. (b) Ablations studies of the hyper-parameter $\tau$ in the benchmark experiments.}
\vspace{-1em}
\label{fig:ablation_rorl_tau}
\end{figure}

\subsection{Ablations on Benchmark Results}
\label{ap:ablation_bench}
In the benchmark experiments, RORL outperforms other baselines, especially in walker2d tasks. We conduct ablation studies on this task to verify the effectiveness of RORL's components. In Figure \ref{fig:ablation_rorl_tau} (a), we can find that each introduced loss (i.e., the OOD loss, the policy smoothing loss and the $Q$ smoothing loss) influences the performance on the walker2d-medium-v2 task. Specifically, the OOD loss affects the most, without which the performance would drop close to SAC-N's performance. In addition, the $Q$ smoothing loss is helpful for stabilizing the training and final performance in clean environments. 

In Figure \ref{fig:ablation_rorl_tau} (b), we evaluate the performance of RORL with varying $\tau$.
The results suggest that $\tau$ is an important factor that balances the learning of in-distribution and out-of-distribution $Q$ values. In Eq.~\eqref{eq:conservative_smoothing_Q}, we want to assign larger weights ($1-\tau$) on the $\delta(s, \hat s,a)_{+}^2$ and smaller weights ($\tau$) on the $\delta (s, \hat s,a)_{-}^2$ to underestimate the values of OOD states, where $\delta (s, \hat s, a)=Q_{\phi_i}(\hat s,a) - Q_{\phi_i}(s,a)$. On the contrary, a too small $\tau$ can also lead to overestimation of in-distribution state-action pairs. In Figure \ref{fig:ablation_rorl_tau} (b), $\tau=0$ leads to poor performance while larger $\tau=0.5,1.0$ also result in performance worse than RORL without $Q$ smoothing. Empirically, we find $\tau=0.2$ works well across different tasks and set $\tau=0.2$ by default for all experiments.

In the above analysis, we know that the OOD loss is a key component in RORL. We further study the impact of the OOD loss and $\epsilon_{\rm ood}$ on the performance and the value estimation. As shown in Figure \ref{fig:ood_loss_rorl} (a), when $\epsilon_{\rm ood}=0$, the performance of RORL drops significantly, which illustrates the effectiveness of underestimating values of OOD states since the smoothness of RORL may overestimate these values. From Figure \ref{fig:ood_loss_rorl} (b), we can verify that the OOD loss with $\epsilon_{\rm ood}>0$ contributes to the value underestimation.

\begin{figure}[h]
\centering
\subfigure[D4rl scores ]{\includegraphics[width=0.47\textwidth]{ 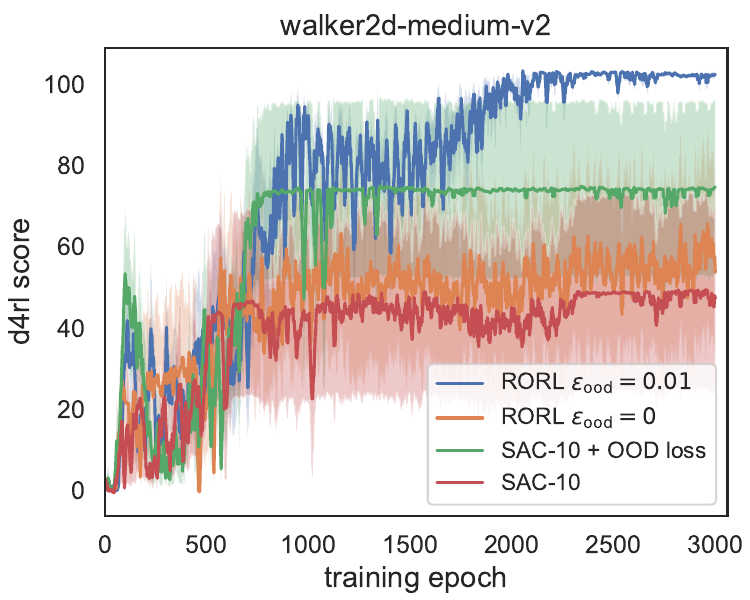}}
\hspace{1.0em}
\subfigure[Estimated values ($\rm log$)]{\includegraphics[width=0.47\textwidth]{ 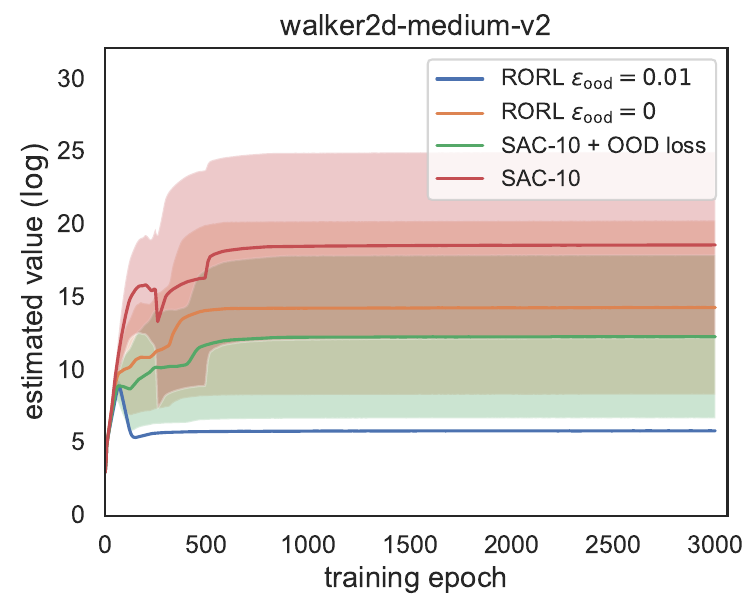}}
\caption{The ablations of the OOD loss $\mathcal{L}_{\rm ood}$ and the hyper-parameter $\epsilon_{\rm ood}$ on the benchmark experiments. }
\vspace{-1em}
\label{fig:ood_loss_rorl}
\end{figure}

\subsection{Robustness Measures}
\label{ap:addition_exp}
In prior works \cite{shen2020deep,zhang2020robust}, the authors only demonstrate the robustness of algorithms via comparing the return curves with different attack scales. To better measure the robustness of RL algorithms, we consider the \emph{robust score} as the areas under the perturbation curve in Figure \ref{fig:attack_rorl_new}. Since the returns in the figure have been normalized as introduced in Appendix \ref{ap:experiment-setting}, we can simply calculate the \emph{robust score} for each attack strategy as:
$$\text{robust score} = \frac{1}{N} \sum_{i\in [1,N]} Rs[i]$$
where $Rs$ is the list of returns under $N$ monotonically increasing attack scales. The introduced \emph{robust score} treats different attack scales equally. However, in many real scenarios, we would pay more attention to larger-scale disturbances. To this end, we also define a \emph{weighted robust score} as:
$$\text{weighted robust score} = \frac{2}{(1+N)\times N} \sum_{i\in [1,N]} i \times Rs[i]$$
where the weights are assigned according to the scale order. In Table \ref{tab:robust_score} and Table \ref{tab:weighted_robust_score}, RORL consistently outperforms EDAC and SAC-$10$ on the two robustness metrics. For walker2d and hopper tasks, RORL surpasses EDAC by more than 10 points on both the \emph{robust score} and the \emph{weighted robust score}.

\begin{table}[tb]
    \caption{Robust scores under attack on halfcheetah-medium-v2, walker2d-medium-v2, and hopper-medium-v2 tasks.}\label{tab:robust_score}
    \centering
    \begin{adjustbox}{max width=\columnwidth}
    \begin{tabular}{llcccccc}
    \toprule
\multirow{2}{*}{Task}   &   &   \multirow{2}{*}{Random} &  \multirow{2}{*}{Action Diff}  &   Action  Diff  &  \multirow{2}{*}{Min Q}  &   Min  $Q$   &  \multirow{2}{*}{Average} \\
   &  {} &   &   &     Mixed  Order &    &   Mixed  Order &   \\
    \midrule
  \multirow{3}{*}{\text{halfcheetah-m}}   & RORL   &     58.6 &     49.5 &   38.0 &      43.5 &      28.2 &      43.6 \\
    &  EDAC   &     59.2 &      44.5 &     33.0 &      38.1     &    25.0  &      40.0 \\
     & SAC-10 &       60.1 &     45.6 &  34.2 &     39.8 &      25.7 &      41.1 \\
     \midrule

   \multirow{3}{*}{\text{walker2d-m}} &  RORL   &    94.1 &     91.0  & 56.9  &     71.0 &    43.3 &   71.2 \\
    & EDAC   &    95.1 &     68.3 &    37.2 &    62.1 &    35.9 &   59.7 \\
    & SAC-10 &   48.2 &    37.0 &     23.0 & 29.2 &      18.5 &      31.2 \\
    \midrule
   \multirow{3}{*}{\text{hopper-m}} & RORL   &   84.8 &    78.4 &    53.9 &    51.5 &  34.7 &   60.7 \\
    & EDAC   &  72.2 &  69.7 &  45.5 &  38.3 &  23.7 &      49.9 \\
    & SAC-10 &   0.79 &   0.82 &   0.89 &   0.88 &               1.36 &     0.95 \\
    \bottomrule
    \end{tabular}
    \end{adjustbox}
\end{table}

\begin{table}[tb]
    \caption{Weighted robust scores under attack on halfcheetah-medium-v2, walker2d-medium-v2, and hopper-medium-v2 tasks.}\label{tab:weighted_robust_score}
    \centering
    \begin{adjustbox}{max width=\columnwidth}
    \begin{tabular}{llcccccc}
    \toprule
\multirow{2}{*}{Task}   &   &   \multirow{2}{*}{Random} &  \multirow{2}{*}{Action Diff}  &   Action  Diff  &  \multirow{2}{*}{Min Q}  &   Min  $Q$   &  \multirow{2}{*}{Average} \\
   &  {} &   &   &     Mixed  Order &    &   Mixed  Order &   \\
    \midrule
  \multirow{3}{*}{\text{halfcheetah-m}}   & RORL   &     57.4 &     44.5 &    29.7 &     37.0 &     17.7 &      37.2 \\
   &  EDAC   &       57.0 &      37.0 &                        23.9 &      28.7 &                  14.4 &      32.2 \\
  &  SAC-10 &  57.9 &      38.3 &         25.1 &      30.8 &       14.9 &      33.4 \\
     \midrule

   \multirow{3}{*}{\text{walker2d-m}} &  RORL   &  94.1 &    89.1 &       39.1 &      61.8 &        26.7 &      62.2 \\
    & EDAC   &    95.1 &    52.9 &        18.7 &      45.7 &      18.8 &      46.2 \\
   &  SAC-10 &       47.7 &  30.1 &      13.8 &      21.3 &   10.7 &  24.7 \\
    \midrule
   \multirow{3}{*}{\text{hopper-m}} & RORL   &   76.0 &   68.0 &       36.1 &      37.4 &     21.1 &      47.7 \\
   & EDAC   &    61.7 &     61.4 &      30.8 &      21.7 &    9.6 &      37.0 \\
   & SAC-10 &        0.80 &    0.84 &                   0.93 &       0.91 &    1.67 &       1.03 \\
    \bottomrule
    \end{tabular}
    \end{adjustbox}
\end{table}

\begin{table}[t]
\caption{The robust scores of ablation studies on the walker2d-medium-v2 task}
\label{tb:ablations_new}
    \centering
    \begin{adjustbox}{max width=\columnwidth}
    \begin{tabular}{lcccccc}
\toprule
{} &   random &   action diff &   action diff mixed order &   min $Q$ &   min $Q$ mixed order &  Average Score \\
\midrule
RORL               &            94.1 &                 90.9 &                             56.9 &           71.0 &                       43.3 &           71.2 \\
no OOD             &            68.4 &                 62.0 &                             37.6 &           35.9 &                       22.0 &           45.2 \\
no P smooth        &            92.8 &                 78.7 &                             48.6 &           67.1 &                       39.3 &           65.3 \\
no $Q$ smooth        &            92.7 &                 91.1 &                             57.3 &           62.2 &                       40.2 &           68.7 \\
$\epsilon_{\rm ood}$=0 &            74.1 &                 70.5 &                             44.1 &           46.3 &                       26.9 &           52.4 \\
\bottomrule
\end{tabular}
    \end{adjustbox}
\end{table}

\subsection{Ablations of Components in the Adversarial Experiments}
In Section \ref{sec:ablation}, we conducted ablations of RORL's major components in the adversarial settings. In this subsection, we provide robust scores of the ablation results over 4 random seeds in Table \ref{tb:ablations_new}. Besides, results of $\epsilon_{\rm ood} = 0$ are also included to demonstrate the effectiveness of penalizing values of OOD states. From Table \ref{tb:ablations_new}, we can conclude that the OOD loss is the most essential component of RORL, and only penalizing in-distribution states is insufficient for adversarial perturbations. To summarize, the order of the importance of each component is: OOD loss $>$  $\epsilon_{\rm ood}$ $>$ policy smoothing loss $>$ $Q$ smoothing loss. The conclusion may be different for different tasks, for example we found that the halfcheetah task does not even need the OOD loss because the SAC-10 framework already provides it with sufficient pessimism.

\subsection{Ablations on the Number of Ensemble $Q$ Networks}
We conduct the adversarial attack experiments with different number of bootstrapped $Q$ networks in RORL. As shown in Figure \ref{fig:attack_numQ_ablation}, the robustness of RORL improves as the ensemble size $K$ increases. For $K=6,8,10$, RORL has similar initial performance but $K=10$ considerably outperforms others as the attack scale increases. Therefore, we set $K=10$ by default in our paper.

\begin{figure}[htb]
    \centering
    \includegraphics[width=1\linewidth]{ 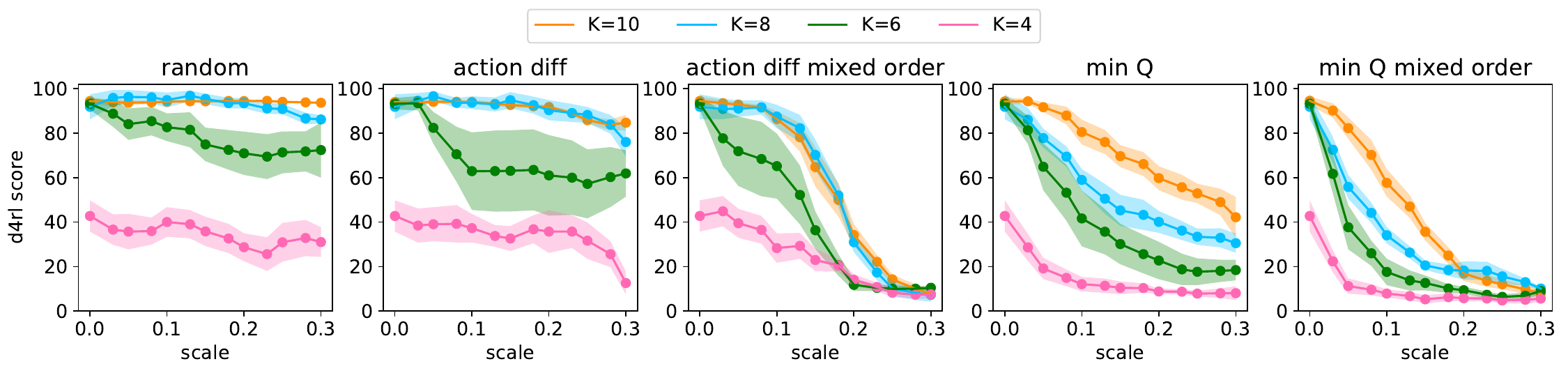}
    \caption{Ablations on the number of $Q$ networks on the walker2d-medium-v2 dataset.}
    \label{fig:attack_numQ_ablation}
\end{figure}

\subsection{Ablations of \texorpdfstring{$\tau$}{tau} for the Adversarial Experiments}
In this subsection, we study the performance under attacks with varying $\tau\in \{0.0, 0.2, 0.5, 1.0\}$. From the results in Figure \ref{fig:attack_tau}, we find $\tau=0.2$ slightly outperforms the others on 4 out of the 5 attack types. The results are also consistent with the ablation studies of the benchmark experiments in Appendix \ref{ap:ablation_bench}. Accordingly, we set $\tau=0.2$ by default for all experiments in our paper.

\begin{figure}[htb]
    \centering
    \includegraphics[width=1\linewidth]{ 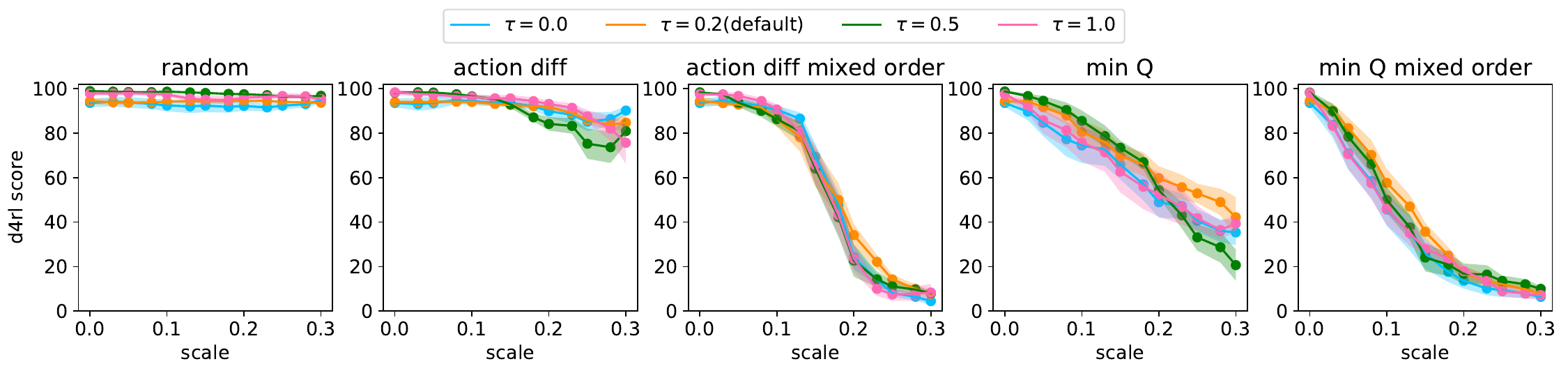}
    \caption{Comparison of different $\tau$ in the adversarial experiments on the walker2d-medium-v2 dataset.}
    \label{fig:attack_tau}
\end{figure}

\subsection{Ablations on the Number of Sampled Perturbed Observations}
\label{ap:num_samples}

We ablate the number of sampled perturbed observations in Figure \ref{fig:ablation_sample}. From the figure, we can conclude that the robustness of RORL improves as the number of samples $n$ increases. At the same time, the computational cost also increases as $n$ increases. Therefore, we can choose $n$ according to the computational budget. Interestingly, RORL with $n=1$ already outperforms SAC-10 by a large margin, which could be an appropriate option when computing resources are limited.

\begin{figure}[htb]
    \centering
    \includegraphics[width=1\textwidth]{ 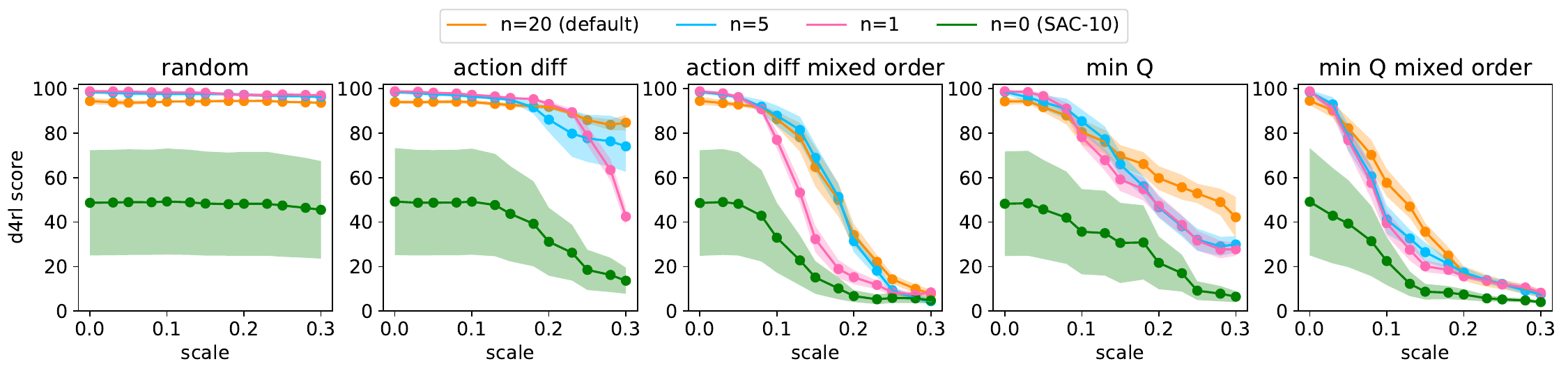}
    \caption{Ablations on the number of sampled perturbed observations. The comparison is made on the walker2d-medium-v2 task. }
    \label{fig:ablation_sample}
\end{figure}

\subsection{Adversarial Attack with Different $Q$ Functions}
In our experiments, it is assumed that the 'min $Q$' and the 'min $Q$ mixed order' attackers have access to the corresponding $Q$ value functions of the attacked agent. Generally, the assumption is strong for many real-world scenarios. In addition, the comparison does not take into account the impact of attacking with different $Q$ functions. Intuitively, conservative and smoothed $Q$ functions make it easier for attackers to find the most impactful perturbation to degrade the performance. To investigate the impact of different $Q$ functions, we swap the attacker's $Q$-function, i.e. \textbf{using RORL's $Q$-functions to attack EDAC and using EDAC's $Q$-functions to attack RORL}. In Figure \ref{fig:change_Q}, we can conclude:
\begin{itemize}
    \item [(1)] RORL outperforms EDAC with a wider margin when using the same $Q$ functions. Surprisingly, the difference of normalized scores increases from 37.3 to 51.2 for walker2d-medium-v2 task with the largest 'min Q' attack.
    
    \item [(2)] The value function of EDAC may still not be smooth and can mislead the attackers. In contrast, RORL successfully learns smooth value functions, which may facilitate further research on stronger attack strategies for robust offline RL.
\end{itemize}

\begin{figure}[h]
\centering
\subfigure['Min $Q$' attack with different $Q$ functions]{\includegraphics[width=0.8\textwidth]{ 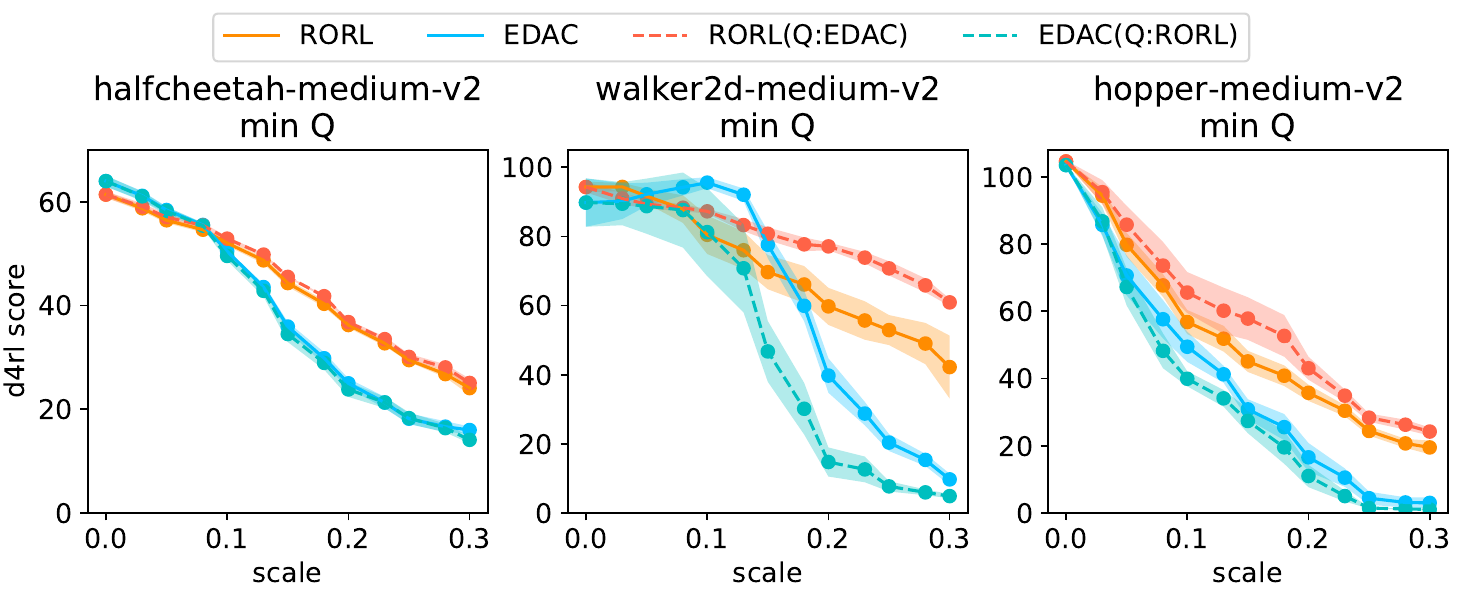}}

\subfigure['Min $Q$ mixed order' attack with different $Q$ functions]{\includegraphics[width=0.8\textwidth]{ 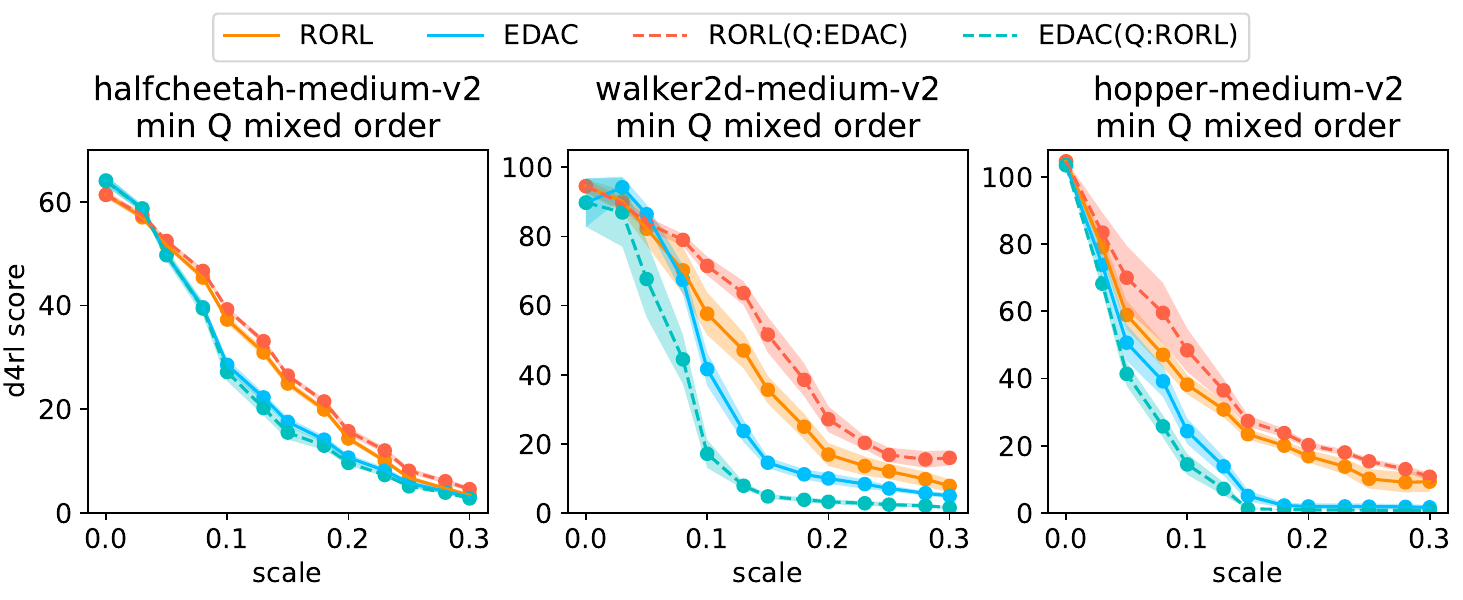}}
\caption{Performance under the 'min $Q$' and the 'min $Q$ mixed order' adversarial attacks with different $Q$ functions. Curves are averaged over 4 random seeds. \emph{RORL(Q:EDAC)} refers to attacking RORL with EDAC's $Q$ functions, and \emph{EDAC(Q:RORL)} refers to attacking EDAC with RORL's $Q$ functions. When the attacker uses the same $Q$ functions, RORL outperforms EDAC with a wider margin.}
\vspace{-1em}
\label{fig:change_Q}
\end{figure}

\subsection{Comparison with EDAC+Smoothing}
We also compare EDAC with both policy smoothing and $Q$ smoothing, which leverages the gradient penalty rather than our OOD loss to enforce pessimism on OOD state-action pairs. The hyper-parameters are kept the same with EDAC and RORL, except $\tau=0.5$ in EDAC+Smoothing. As shown in Figure \ref{fig:edac_smoothing}, the smoothing technique slightly improves the robustness of EDAC under large-scale (0.2$\sim$0.3) adversarial perturbations, but it significantly decreases the overall performance under attack. The results imply that directly using smoothing techniques without explicit OOD penalization can even worsen the robust scores of previous SOTA offline RL algorithm.

\begin{figure}[tb]
    \centering
    \includegraphics[width=1\textwidth]{ 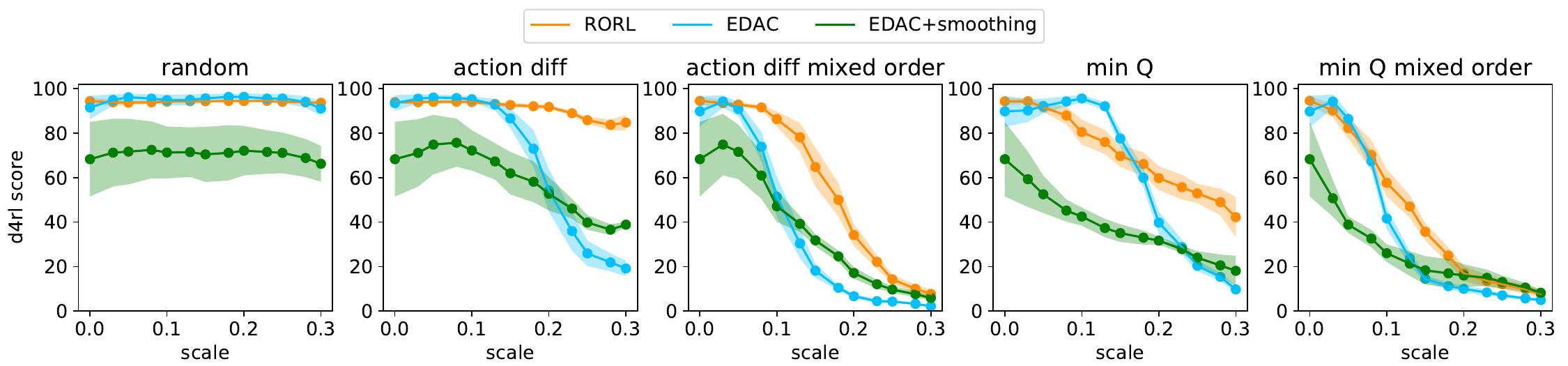}
    \caption{Comparison with EDAC+Smoothing under adversarial attacks on the walker2d-medium-v2 task. The curves are averaged over 4 seeds and smoothed with a window size of 3.}
    \label{fig:edac_smoothing}
\end{figure}

\begin{figure}[tb]
\centering
\subfigure[Performance under attack on halfcheetah-medium-v2 dataset]{\includegraphics[width=1\textwidth]{ 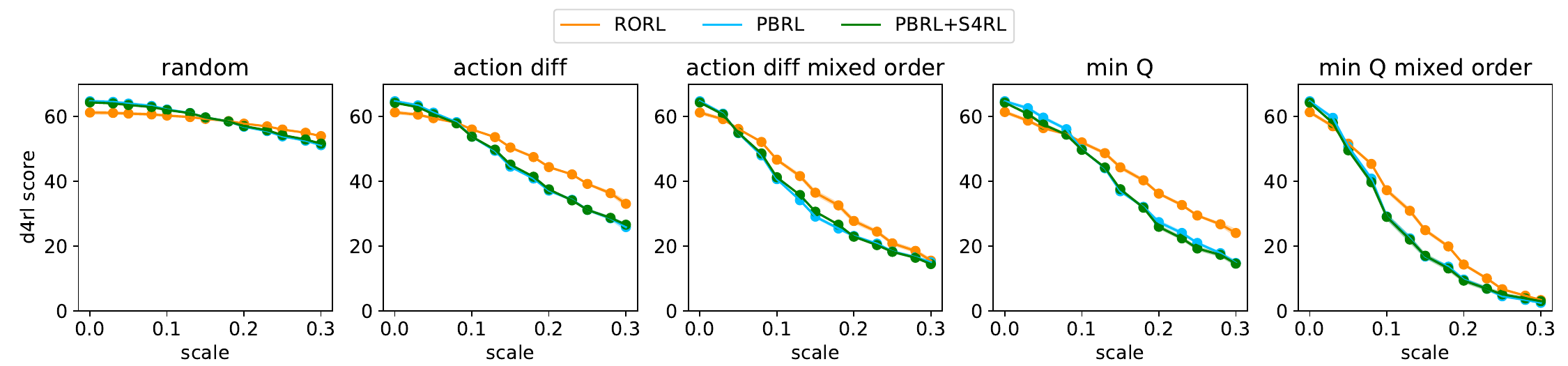}}

\subfigure[Performance under attack on walker2d-medium-v2 dataset]{\includegraphics[width=1\textwidth]{ 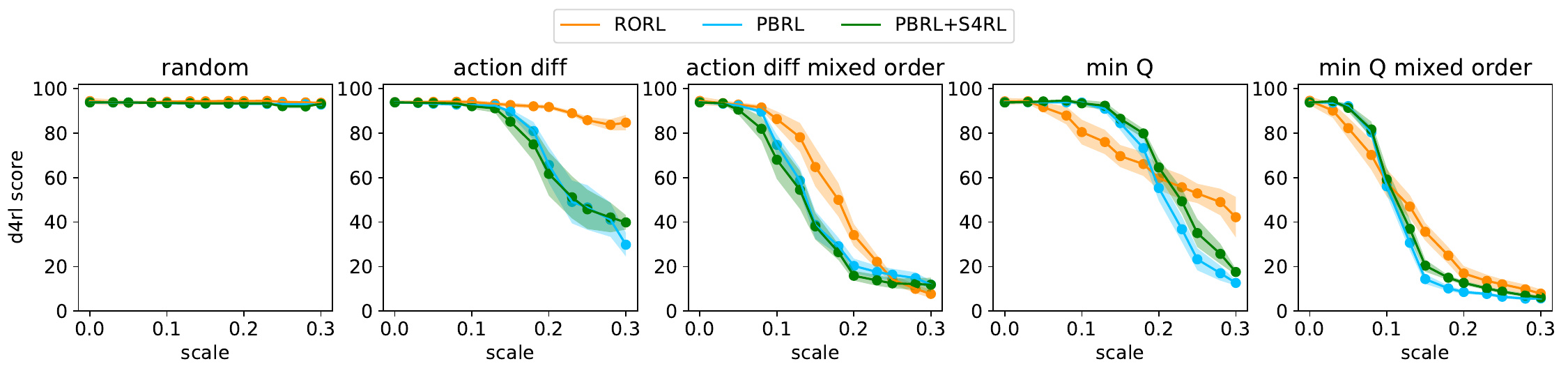}}
\caption{Comparison of PBRL and PBRL+S4RL under attack scales range [0, 0.3] of different types of attack. The curves are averaged over 4 seeds and smoothed
with a window size of 3. The shaded region represents half a standard deviation.}
\vspace{-1em}
\label{fig:PBRL_S4RL}
\end{figure}

\subsection{Comparison with PBRL + S4RL}
We also include comparison with PBRL and PBRL+S4RL to verify if RORL is more robust than data augmentation for offline RL \cite{sinha2022s4rl}. The main differences between RORL and S4RL are three folds:
\begin{itemize}
    \item [(1)] S4RL only implicitly smooths the value functions while RORL explicitly smooths them, which is more efficient and enjoys theoretical guarantees.
    \item [(2)] S4RL does not consider the impact of overestimation on OOD states brought by the data augmentation, which can be harmful for offline RL. In contrast, RORL further underestimates values for OOD states, which essentially alleviates the potential overestimation.
    \item [(3)] In addition, S4RL selects adversarially perturbed states according to the gradient of $Q(s, \pi(s))$, aiming to choose the direction where the $Q$-value deviates the most. Different from S4RL, RORL samples perturbed states to maximize a conservative smoothing loss $\mathcal{L}\big(Q_{\phi_i}(\hat s,a),Q_{\phi_i}(s,a)\big)$ and a policy smoothing loss $\max_{\hat s\in \mathbb{B}_d(s,\epsilon)} D_{\rm J} \big(\pi_{\theta}(\cdot|s)\|\pi_{\theta}(\cdot|\hat s)\big)$ defined in Section \ref{sec:method}.
\end{itemize}

The empirical results on halfcheetah-medium-v2 and walker2d-medium-v2 are shown in Figure \ref{fig:PBRL_S4RL}. We can observe that S4RL only slightly improves the robustness of PBRL on the walker2d-medium-v2 task and has little impact on the halfcheetah-medium-v2 task. In contrast, RORL exhibits higher robustness across different tasks and attack types.

\begin{figure*}[t]
\centering
\subfigure[Performance under attack on halfcheetah-medium-v2 dataset]{\includegraphics[width=1\textwidth]{ 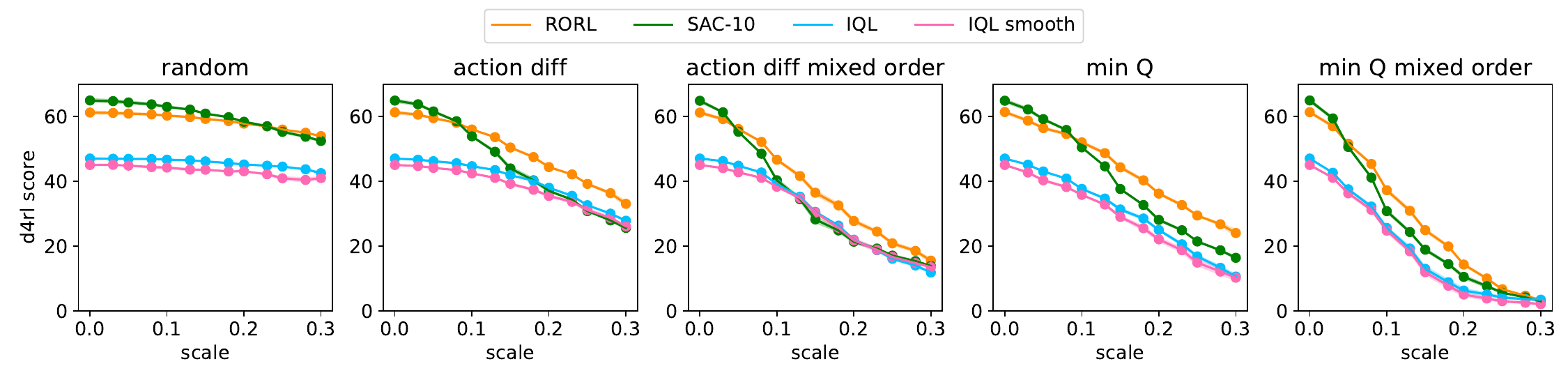}\label{fig:attack_halfcheetah_iql}}
\subfigure[Performance under attack on walker2d-medium-v2 dataset]{\includegraphics[width=1\textwidth]{ 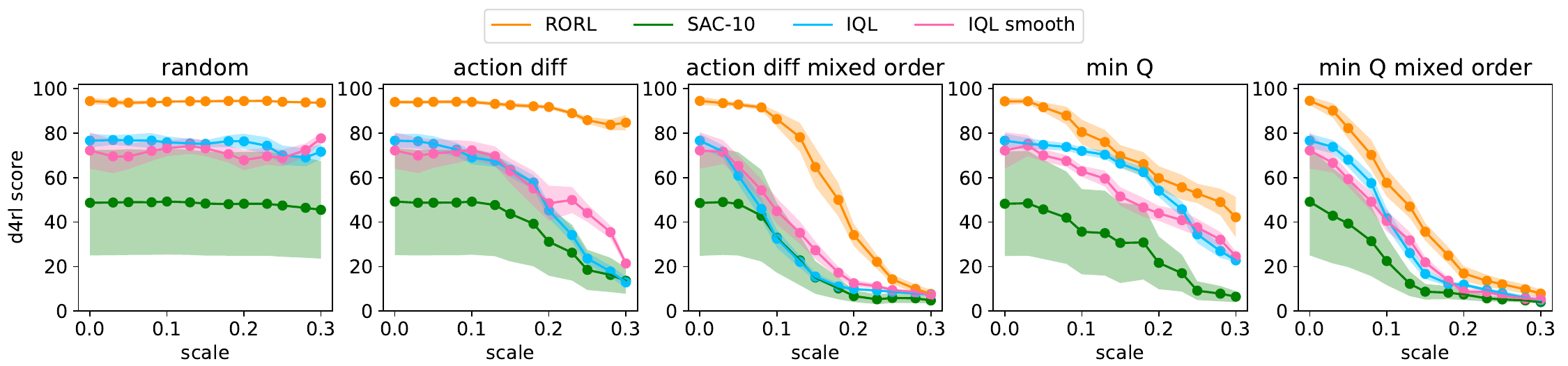}\label{fig:attack_walker2d_iql}}
\subfigure[Performance under attack on hopper-medium-v2 dataset]{\includegraphics[width=1\textwidth]{ 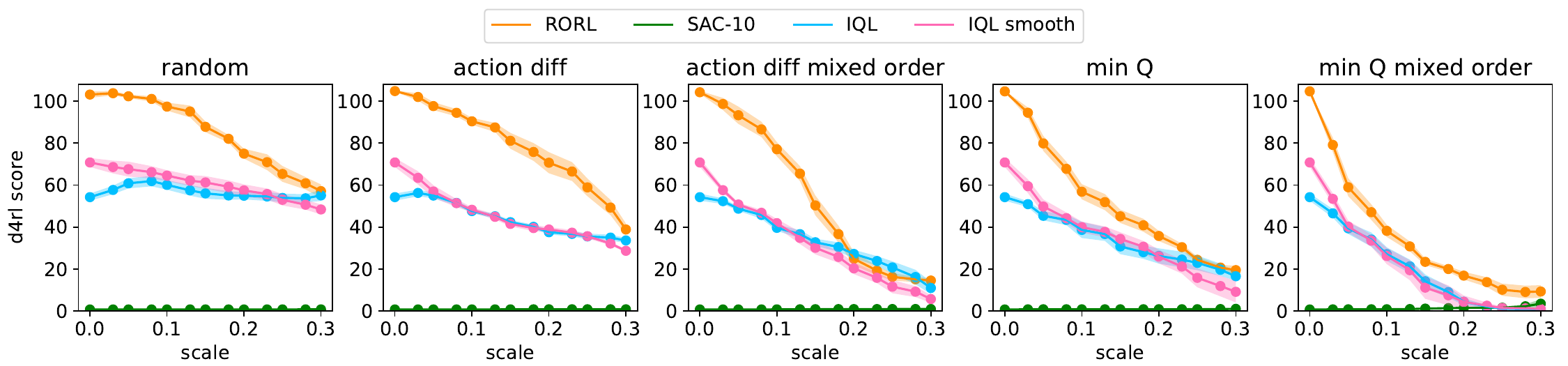}\label{fig:attack_hopper_iql}}
\caption{Comparison of IQL and IQL smooth. Figures (a) (b) (c) illustrate the performance under attack scales range $[0,0.3]$ of different types of attack. The curves are averaged over 4 seeds and smoothed with a window size of 3. The shaded region represents half a standard deviation. }
\label{fig:attack_iql}
\vspace{-1em}
\end{figure*}

\subsection{Combining Smoothing with IQL}
We combine the policy smoothing and $Q$ function smoothing techniques in RORL with IQL \cite{kostrikov2021offline}, a SOTA offline RL algorithm without ensemble $Q$ networks. We use the default hyper-parameters of IQL and set the hyper-parameters for smoothing the same as in Table \ref{tab:attack_params}. The training and evaluation settings keep the same as the adversarial experiments in our paper. As shown in Figure \ref{fig:attack_iql}, we can observe that IQL with the smoothing technique (short for 'IQL smooth') slightly improves the robustness on the walker2d-medium-v2 and hopper-medium-v2 tasks, but it has little effect on the halfcheetah-medium-v2 task. This suggests that simply adopting the smoothing technique does not consistently improve the performance in the offline setting. In contrast, RORL introduces additional OOD underestimation based on uncertainty measure, which helps to obtain conservatively smoothed policy and value functions.

\begin{figure}[htb]
    \centering
    \includegraphics[width=1\textwidth]{ 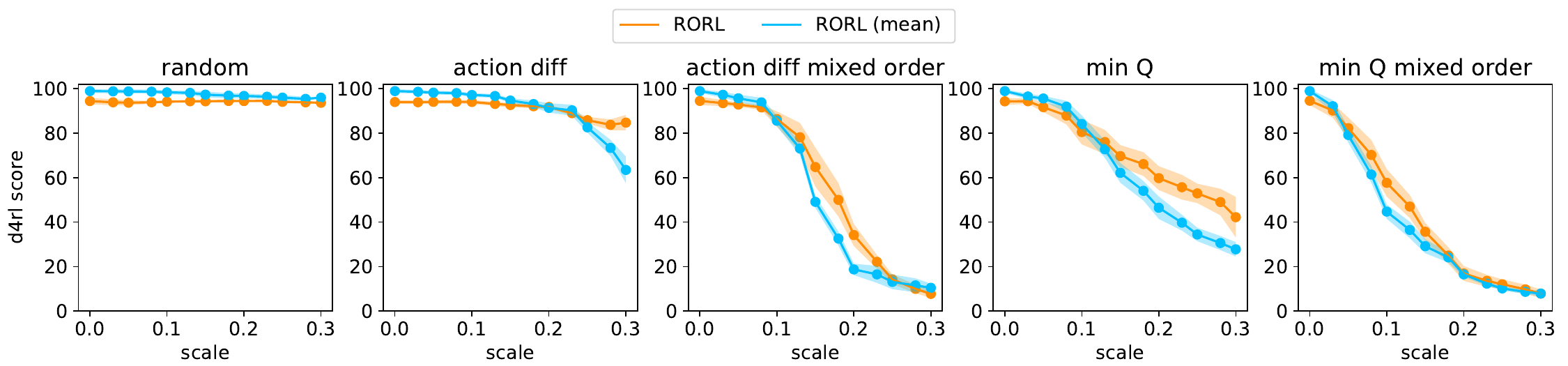}
    \caption{Comparing the 'max' with the 'mean' operators in our smoothing techniques. The comparison is made on the walker2d-medium-v2 task.}
    \label{fig:ablation_mean}
\end{figure}

\subsection{Comparing the 'max' and the 'mean' Operators in Smoothing}

In our implementation, we first sample $n$ perturbed states and select the one that maximizes the smoothing losses in Eq.~\eqref{eq:Q_smooth} and Eq.~\eqref{eq:roal-policy}. It is interesting to see if the 'max' operator is useful, as we can also use the 'mean' operator as an alternative, i.e., $\mathcal{L}_{\rm smooth}^{mean}(s,a; \phi_i) = \mathbb{E}_{\hat s\in \mathbb{B}_d(s,\epsilon)}  \mathcal{L}\big(Q_{\phi_i}(\hat s,a),Q_{\phi_i}(s,a)\big)$ and $\mathbb{E}_{\hat s\in \mathbb{B}_d(s,\epsilon)} D_{\rm J} \big(\pi_{\theta}(\cdot|s)\|\pi_{\theta}(\cdot|\hat s)\big)$.

The results are demonstrated in Figure \ref{fig:ablation_mean}. We can find that RORL with the 'max' operator obtains a more conservative policy under small-scale perturbations and achieves higher robustness under large-scale perturbations. Since the 'max' operator has the same complexity as the 'mean' operator, we use the 'max' operator by default, which is also a zeroth-order approximation to an inner optimization problem.

\begin{figure}[t]
\centering
\subfigure[Performance under attack on halfcheetah-medium-v2 dataset]{\includegraphics[width=1\textwidth]{ 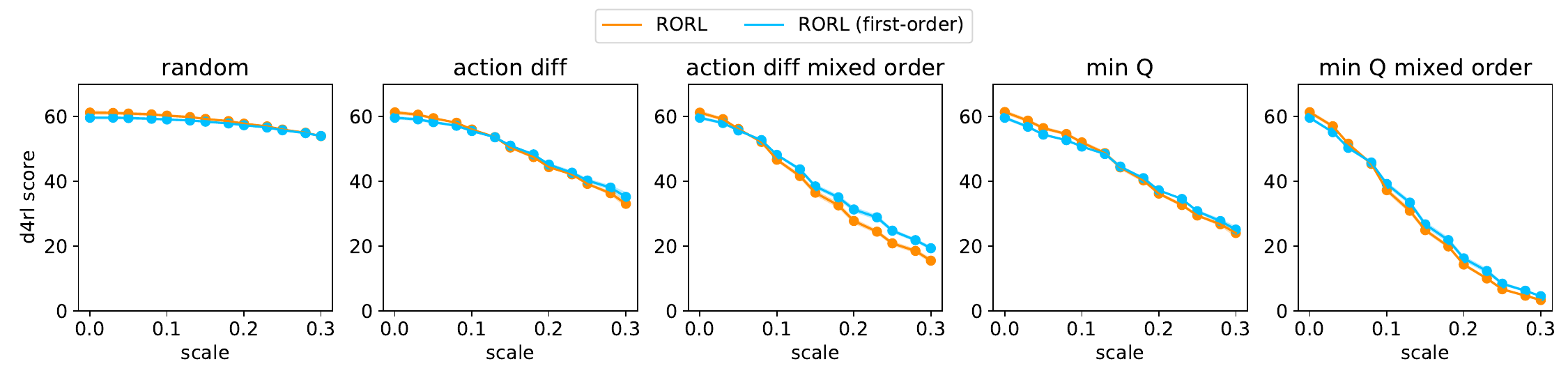}}

\subfigure[Performance under attack on walker2d-medium-v2 dataset]{\includegraphics[width=1\textwidth]{ 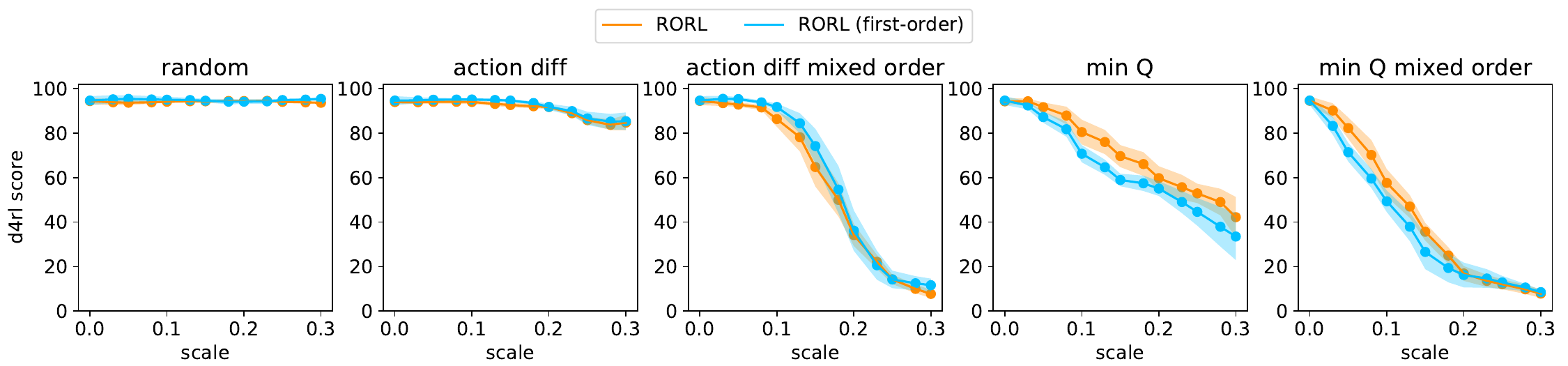}}
\caption{Comparison of zeroth-order and first-order optimization in the training period. The curves are averaged over 4 seeds and smoothed with a window size of 3. The shaded region represents half a standard deviation.}
\vspace{-1em}
\label{fig:first_order}
\end{figure}

\subsection{Comparing Different Optimization for Perturbation Generation during Training}
In the training period, we use zeroth-order optimization to approximately optimize the $Q$ smoothing loss in Eq.~\eqref{eq:Q_smooth} and the policy smoothing loss: $\max_{\hat s\in \mathbb{B}_d(s,\epsilon)} D_{\rm J} \big(\pi_{\theta}(\cdot|s)\|\pi_{\theta}(\cdot|\hat s)\big)$. In this way, we can accelerate training the robust policy and obtain similar performance. Besides, zeroth-order optimization is commonly applied in black-box attack where we can only access the input and output of neural networks without explicit gradient information. Black-box attack for reinforcement learning might be a promising direction in the future.

We also implemented a first-order version of RORL, which requires an average epoch time of 72.7s on a V100 GPU (while the average epoch time of the zeroth-order method is 29.6s). Since the perturbation generation for each training step is independent, we use the first-order optimization for a probability of 0.5 to alleviate the computational cost. In Figure \ref{fig:first_order}, we compare the trained policies with zeroth-order and first-order optimization. We can conclude that the two types of optimization for perturbation generation have very similar performance. On halfcheetah-medium task, the first-order version performs slightly better than the zeroth-order version, while the zeroth-order version works slightly better on the walker2d-medium task. We think this might be because we train the policy and value networks for 3$\times 10^6$ training steps, which may narrow the gap of the two optimization methods. On the contrary, the mixed-order attackers ('action diff mixed order' and 'min $Q$ mixed order') work better than zeroth-order attackers ('action diff' and 'min $Q$') in the evaluation period, as demonstrated in Figure \ref{fig:attack_rorl_new}.

\begin{figure}[tb]
\centering
\subfigure[D4rl scores ]{\includegraphics[width=0.245\textwidth]{ 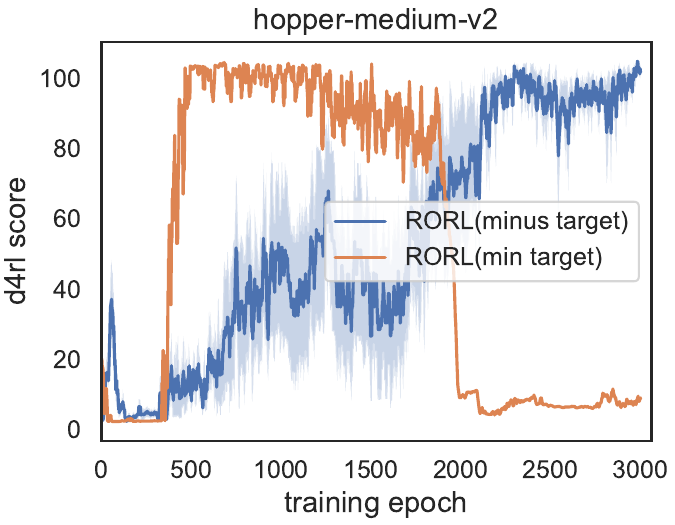}}
\subfigure[Estimated values ($\rm log$)]{\includegraphics[width=0.245\textwidth]{ 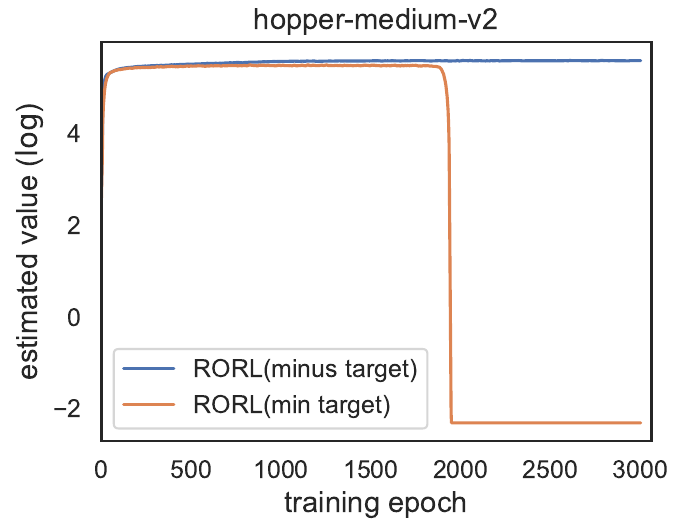}}
\subfigure[D4rl scores]{\includegraphics[width=0.245\textwidth]{ 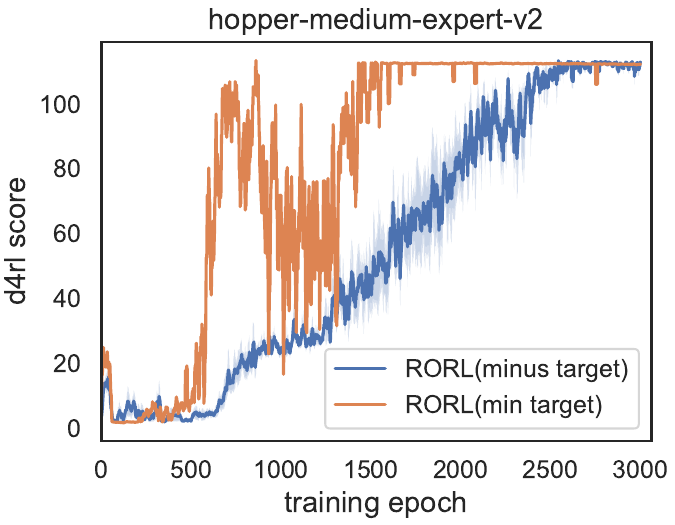}}
\subfigure[Estimated values ($\rm log$)]{\includegraphics[width=0.245\textwidth]{ 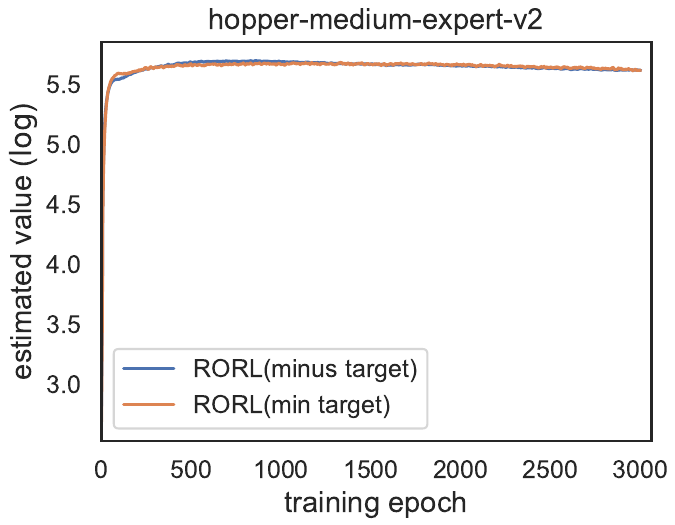}}

\subfigure[D4rl scores ]{\includegraphics[width=0.245\textwidth]{ 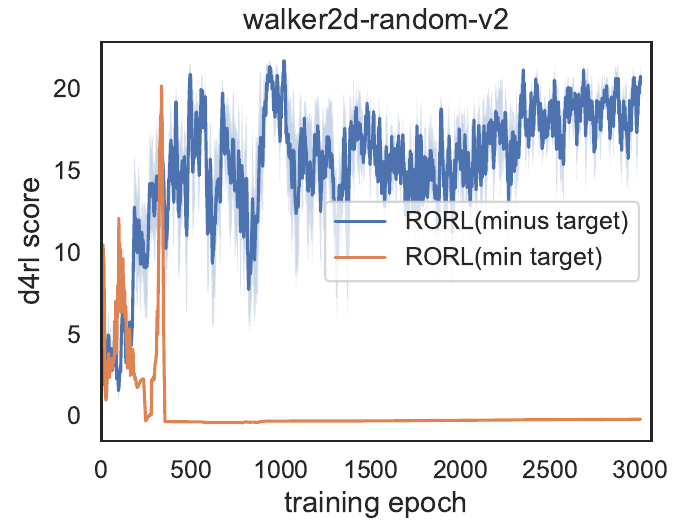}}
\subfigure[Estimated values ($\rm log$) ]{\includegraphics[width=0.245\textwidth]{ 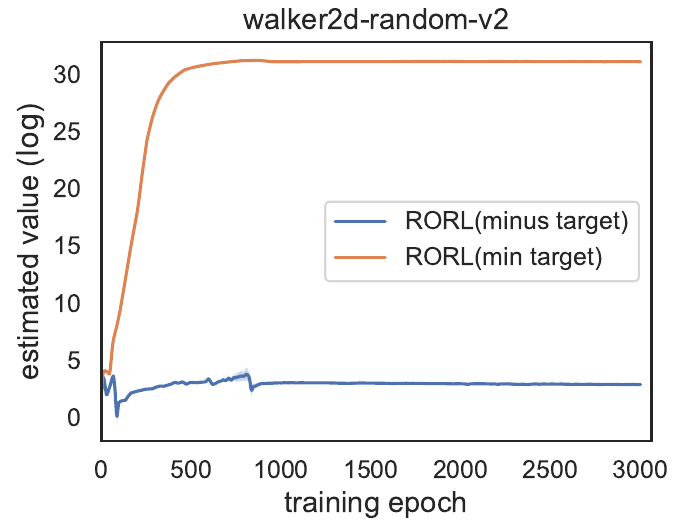}}
\subfigure[D4rl scores]{\includegraphics[width=0.245\textwidth]{ 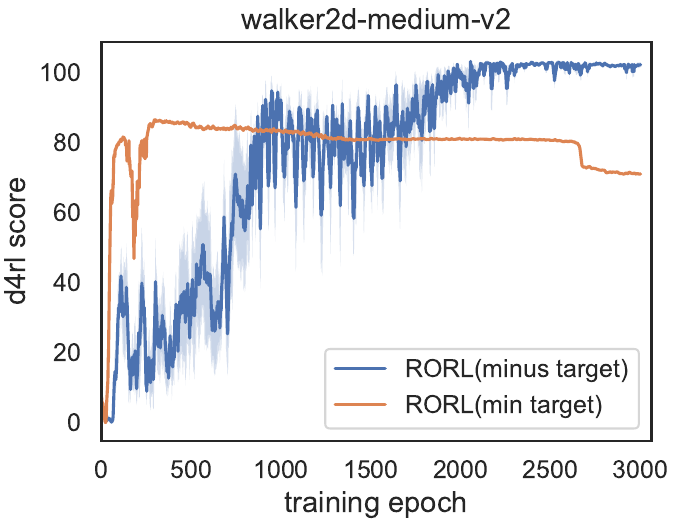}}
\subfigure[Estimated values ($\rm log$)]{\includegraphics[width=0.245\textwidth]{ 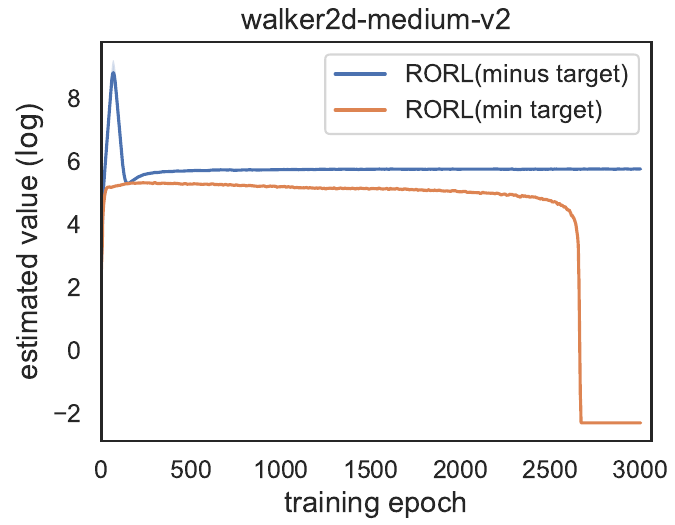}}

\caption{Comparison of the ``minus target'' and the ``min target'' in the OOD loss $\mathcal{L}_{\rm ood}$ on four tasks.}
\vspace{-1em}
\label{fig:comparison_min_target}
\end{figure}

\subsection{Comparison of the ``minus target'' and the ``min target''}
\label{ap:min_target}
In the OOD loss $\mathcal{L}_{\rm ood}$ (Eq.~\eqref{eq:ood_loss}), the pseudo-target $\widehat{\mathcal{T}}_{\rm ood}Q_{\phi_i}(\hat{s},\hat{a})$ for the OOD state-action pairs $(\hat s, \hat a)$ can be implemented in two ways to underestimate the values of $(\hat s, \hat a)$: $\widehat{\mathcal{T}}_{\rm ood}Q_{\phi_i}(\hat{s},\hat{a}):=Q_{\phi_i}(\hat{s},\hat{a})- \lambda u(\hat{s},\hat{a})$ or $\widehat{\mathcal{T}}_{\rm ood}Q_{\phi_i}(\hat{s},\hat{a}):= {\rm min}_{i=1,\ldots, K} Q_{\phi_i}(\hat{s},\hat{a})$ ($K=10$). The two targets are referred to as the ``minus target'' and the ``min target'' respectively. In Figure \ref{fig:comparison_min_target}, we compare the two targets' D4RL scores in clean environments, and the hyper-parameters are the same as Table \ref{tab:bench_params}. Although the ``min target'' has less hyper-parameters and achieves comparable performance on the hopper-medium-expert-v2 task, it is unstable and not flexible across different tasks, e.g., significantly overestimating values for the walker2d-random-v2 task and underestimating values for hopper-medium-v2 and walker2d-medium-v2 tasks. Therefore, we choose the ``minus target'' by default in our paper.

\begin{table}[b]
	\centering
	\small
	\caption{Hyper-parameters of RORL for the Adroit domains.}
	\vspace{0.2em}
	\label{tab:adroit_params}
	\begin{adjustbox}{max width=\columnwidth}
		\begin{tabular}{l|c|c|c|c|c|c|c|c|c}
			\toprule
			\textbf{Task Name} & $\beta_{\rm Q}$ & $\beta_{\rm P}$ & $\beta_{\rm ood}$ & $\epsilon_{\rm Q}$  & $\epsilon_{\rm P}$  & $\epsilon_{\rm ood}$ & $\tau$ & $n$ & $\lambda\ (d)$ \\
			\midrule
			Pen-human  & \multirow{4}{*}{0.0001} &  \multirow{4}{*}{0.01} & \multirow{4}{*}{0.5} & \multirow{4}{*}{0.001} & \multirow{4}{*}{0.001} &  0.001 & \multirow{4}{*}{0.2} & \multirow{4}{*}{20} &  0.2$\rightarrow 0.1$ ($1e^{-6}$) \\
			Hammer-human &    &   &    &  &  & 0.01  &   &  & 2$\rightarrow 0.5$ ($2e^{-6}$) \\
			Door-human  &    &    &   &  &  &   0.01 &   &   & 1$\rightarrow 0.5$ ($1e^{-6}$)  \\
			Relocate-human  &    &   &     &  &  &  0.01  &   &  &   1$\rightarrow 0.5$ ($1e^{-6}$) \\
			\midrule
			Pen-cloned   & \multirow{4}{*}{0.0001} & \multirow{4}{*}{0.1} & \multirow{4}{*}{0.5} & 0.001 & 0.001 & 0.001 & \multirow{4}{*}{0.2}  & \multirow{4}{*}{20} & 1$\rightarrow 0.2$ ($2e^{-6}$) \\
			Hammer-cloned  &   &   &   & 0.005  & 0.005  & 0.01 &   &  & 2$\rightarrow 0.5$ ($2e^{-6}$) \\
			Door-cloned &   &  &   & 0.001  &  0.001  & 0.01 &   &   & 1$\rightarrow 0.5$ ($1e^{-6}$) \\
			Relocate-cloned  &   &  &   & 0.005  &  0.005  & 0.01 &   &   & 2$\rightarrow 1.0$ ($1e^{-6}$) \\
			\midrule
			Pen-expert & \multirow{4}{*}{0.0001}  & \multirow{4}{*}{1.0} & \multirow{4}{*}{0.5} & 0.005 & 0.005 & \multirow{4}{*}{0.01}  & \multirow{4}{*}{0.2} & \multirow{4}{*}{20} & 2.0$\rightarrow 2.0$ ($0.0$) \\
			Hammer-expert &   &  &  & 0.005 & 0.005 &    &   &  & 1$\rightarrow 0.5$ ($1e^{-6}$) \\
			Door-expert &    &  &  & 0.005 & 0.005 &    &   &  & 1.5$\rightarrow 1.5$ (0.0) \\
			Relocate-expert  &   &  &  & 0.001 & 0.001 &    &   &  & 3$\rightarrow 2.0$ ($2e^{-6}$) \\
			\bottomrule
		\end{tabular}
	\end{adjustbox}
\end{table}

\subsection{Experiments in Adroit Domains}
We also evaluate RORL in the challenging Adroit domains which control a 24-DoF robotic hand to manipulate a pen, a hammer, a door and a ball. These domains contain three types of data, namely `Expert', `Cloned', and `Human', for each task. The hyper-parameters are listed in Table \ref{tab:adroit_params}. We set $\beta_{\rm Q}=0.0001$, $\beta_{\rm ood}=0.5$, $\tau=0.2$, $n=20$, and search $\beta_{\rm P}$ within $\{0.01, 0.1, 1.0\}$, $\epsilon_{\rm Q}$/$\epsilon_{\rm P}$/$\epsilon_{\rm ood}$ within $\{0.001, 0.005, 0.01 \}$. For Door/Relocate-human/cloned datasets, the policy learning rate is set to $1e^{-4}$. The other hyper-parameters are the same as in Table \ref{tab:hyper-SAC10}. On four expert datasets, we train RORL for 1000 epochs (1000 gradient steps per epoch). As for other datasets, we train RORL for 300 epochs because the `Cloned' and `Human' datasets are much smaller.

In Table \ref{tab:adroit}, we compare the performance of RORL with other baselines, such as EDAC, PBRL, TD3+BC, CQL, UWAC, BEAR, and BC. We can observe that RORL achieves the top two highest score in 6 out of 12 tasks, which further verifies the effectiveness of RORL.

\begin{table}[h]
\small
\centering
\caption{Average normalized score over 3 seeds in Adroit domain. Top two highest scores are highlighted.}
\begin{adjustbox}{max width=\columnwidth}
\begin{tabular}{llcccccccc}
\toprule
& & BC & BEAR & UWAC & CQL & TD3+BC & PBRL & EDAC & RORL \\ 
\midrule
\multirow{4}{*}{\rotatebox[origin=c]{90}{Human}} & Pen  & 34.4 & -1.0 &   10.1  $\pm$3.2 &   \textbf{37.5} &  0.0 &   {35.4  $\pm$3.3} &  \textbf{52.1$\pm$8.6} & 33.7 $\pm$ 7.6  \\
& Hammer & 1.5 & 0.3 &    1.2 $\pm$0.7 &   \textbf{4.4} &  0.0 &   0.4  $\pm$ 0.3 & 0.8$\pm$0.4 &  \textbf{2.3 $\pm$ 1.9} \\
& Door & 0.5 & -0.3 &   0.4  $\pm$0.2 &   \textbf{9.9} &  0.0 &   0.1  $\pm$0.0  & \textbf{10.7$\pm$6.8}  & 3.78 $\pm$ 0.7 \\
& Relocate & 0.0 & -0.3 &   0.0  $\pm$0.0 &   \textbf{0.2} &  0.0 &   0.0  $\pm$0.0 & \textbf{0.1$\pm$0.1}   &  0.0 $\pm$ 0.0 \\

\hline
\multirow{4}{*}{\rotatebox[origin=c]{90}{Cloned}} & Pen  &   {56.9} & 26.5 &   {23.0  $\pm$6.9} & 39.2 &  0.0 &   \textbf{74.9 $\pm$9.8} & \textbf{68.2$\pm$7.3} & 35.7$\pm$ 3.1 \\
& Hammer & 0.8 & 0.3 &   {0.4 $\pm$0.0} &   \textbf{2.1} &  0.0 &   {0.8  $\pm$0.5}  &  0.3$\pm$0.0    & \textbf{ 1.7 $\pm$0.5} \\
& Door & -0.1 & -0.1 &   {0.0 $\pm$0.0} & 0.4 &  0.0 &   \textbf{4.6   $\pm$4.8} &   \textbf{ 9.6$\pm$8.3 }    & -0.1 $\pm$ 0.1\\
& Relocate & -0.1 & -0.3 &   {-0.3  $\pm$0.0} & -0.1 &  0.0 &   {-0.1   $\pm$0.0}   &  \textbf{ 0.0$\pm$0.0}    & \textbf{0.0 $\pm$ 0.0} \\
\hline
\multirow{4}{*}{\rotatebox[origin=c]{90}{Expert}} & Pen  & 85.1 & 105.9 &   {98.2   $\pm$9.1} & 107.0 &  0.3 &   \textbf{137.7   $\pm$3.4}  & 122.8 $\pm$ 14.1  & \textbf{130.3 $\pm$ 4.2}\\
& Hammer &   {125.6} &   {127.3} &   {107.7   $\pm$21.7} & 86.7 &  0.0 &   \textbf{127.5  $\pm$0.2} & 0.2 $\pm$ 0.0 & \textbf{132.2 $\pm$ 0.7} \\
& Door & 34.9 & 103.4 &   \textbf{104.7  $\pm$0.4} & 101.5 &  0.0 &   {95.7  $\pm$12.2} & -0.3 $\pm$ 0.1   & \textbf{104.9 $\pm$ 0.9} \\
& Relocate &   \textbf{101.3} &   {98.6} &   \textbf{105.5  $\pm$3.2} & 95.0 &  0.0 &   {84.5  $\pm$12.2} & -0.3 $\pm$ 0.0 & 47.8 $\pm$ 13.5 \\
\bottomrule
\end{tabular}
\end{adjustbox}
\label{tab:adroit}
\end{table}

\subsection{AntMaze Tasks}
The AntMaze domain is a challenging navigation domain with an 8-DoF Ant quadruped robot and three types of datasets, namely `umaze', `medium', and `large'. In this domain, the agent receives a sparse reward of 0/1, where reward 1 is given only when the ant reaches the desired goal. The challenges for the AntMaze domain are sparse rewards and multitask data, which might be beyond the scope of our study. To the best of our knowledge, very few ensemble-based offline RL algorithms can work in this domain, probably because estimating uncertainty in a sparse reward setting is difficult. A recent work~\cite{ghasemipour2022so} conducted in-depth research on this problem and found that the independent target is crucial for the uncertainty estimation in ensemble-based offline RL. We adopt the techniques used in~\cite{ghasemipour2022so} for RORL and reported the results in Table \ref{tab:antmaze}. We only use the OOD loss and policy smoothing loss for RORL, and replace the shared min target in Eq. \eqref{eq:soft_Q_learning} with the independent target to train $Q$ functions:
\begin{equation}
    \widehat{\mathcal{T}}Q_{\phi_i}(s, a):=r(s,a)+\gamma \widehat {\mathbb{E}}_{a'\sim \pi_\theta(\cdot|s')}\big[ Q_{\phi'_i}(s',a')- \alpha \cdot \log\pi_\theta(a'|s')\big],
\end{equation}
In Eq \eqref{eq:roal-policy}, we can train the policy with the `LCB' objective (i.e., $\text{mean}_{j=1,\ldots,K}Q_{\phi_j}(s,a) - c \cdot \text{std}_{j=1,\ldots,K}Q_{\phi_j}(s,a) $, where $c=4$ is used in our experiments) or the `Min' target (i.e., $\min_{j=1,\ldots,K}Q_{\phi_j}(s,a)$) to enforce pessimism. Following~\cite{ghasemipour2022so}, we train RORL for $2\times 10^5$ training steps and evaluate the final performance for 100 episodes. Instead of changing the 0/1 reward to -2/2, we adopt reward shifting~\cite{sun2022exploiting} to change the 0/1 reward to 0.001/10. We also find that adding the BC loss to the policy loss is helpful for antmaze-umaze tasks. Therefore, we add the BC loss to the policy loss for $5\times 10^4$ training steps for all tasks, except for antmaze-umaze-diverse, where we add the BC loss for $2\times 10^5$ training steps. Other hyper-parameters such as the coefficient $\beta_{\rm BC}$ of BC loss are listed in Table \ref{tab:antmaze_params}. 

We also apply our policy and value function smoothing techniques on top of IQL (short for `IQL+smoothing'). For the hyper-parameters, we use $\epsilon_{\rm Q}=0.01$, $\epsilon_{\rm P}=0.03$, $\tau=0.2$ for all six types of datasets, and search $\beta_{\rm Q} \in \{0.1, 0.01\}$, $\beta_{\rm P}\in \{0.1, 0.5\}$, $n=20$. Other hyper-parameters keep the default hyper-parameters of IQL~\cite{kostrikov2021offline}.

In Table \ref{tab:antmaze}, we compare RORL and `IQL+smoothing' with both model-free (AWAC \cite{nair2020accelerating}, TD3+BC \cite{fujimoto2021minimalist}, CQL \cite{kumar2020conservative}, and IQL \cite{kostrikov2021offline}) and model-based (ROMI \cite{wang2021offline}) baselines. RORL achieves the highest average score on the 6 tasks. Besides, on 4 out of 6 tasks, 'IQL+smoothing' improves the performance of IQL. Intuitively, for sparse reward tasks, smoothing the value functions of nearby states could help with the value propagation, and smoothing the policy can enhance the robustness of learned policies. But we can still notice that RORL does not perform well on the antmaze-large task, which may be a future improvement work.


\begin{table}[h]
	\centering
	\small
	\caption{Hyper-parameters of RORL for the AntMaze domains.}
	\vspace{0.2em}
	\label{tab:antmaze_params}
	\begin{adjustbox}{max width=\columnwidth}
		\begin{tabular}{l|c|c|c|c|c|c|c|c|c}
			\toprule
			\textbf{Task Name}  & $\beta_{\rm P}$ & $\beta_{\rm ood}$  & $\epsilon_{\rm P}$  & $\epsilon_{\rm ood}$ &  $n$ & policy objective & $\beta_{\rm BC}$ & $\lambda\ (d)$ \\
			\midrule
			umaze  & \multirow{6}{*}{1.0} & 0.3 & \multirow{6}{*}{0.005} & \multirow{6}{*}{0.01} & \multirow{6}{*}{20} &  LCB &  \multirow{6}{*}{10} &  1.0$\rightarrow 1.0$ ($0$) \\
			umaze-diverse &    & 0.3  &    &    &   &  LCB &  & 2.0$\rightarrow 2.0$ ($0$) \\
			medium-play  &    &  0.3 &     &    &    & LCB  &  &   1.0$\rightarrow 1.0$ ($0$) \\
			medium-diverse &   &   0.3 &   &    &  & LCB &  & 2.0$\rightarrow 1.0$ ($1e^{-6}$) \\
			large-play  &   & 0.5  &   &   &     &  Min &  & 2.0$\rightarrow 1.0$ ($1e^{-6}$) \\
			large-diverse &   & 0.3 &   &   &      & Min  &   & 1.0$\rightarrow 1.0$ ($0$) \\
			\bottomrule
		\end{tabular}
	\end{adjustbox}
\end{table}

\begin{table}[hb]
    \centering
    \caption{Comparison of final performance on AntMaze tasks. The results are average over 3 random seeds. Top two scores for each task are highlighted.}
    \label{tab:antmaze}
    \begin{adjustbox}{max width=\columnwidth}
    \begin{tabular}{lcccccccc}
        \toprule
        & BC & AWAC & TD3+BC & CQL & ROMI+BCQ &  IQL &   IQL+smoothing & RORL \\ 
        \midrule
       antmaze-umaze  &   54.6 &    56.7 &   78.6 &     74.0 &  68.7$\pm$2.7 &  87.5 &   \textbf{92.3$\pm$4.6}  &  \textbf{96.7 $\pm$ 1.9}  \\
        antmaze-umaze-diverse   &  45.6 &   49.3 &     71.4 &   \textbf{84.0}  & 61.2 $\pm$ 3.3  &   62.2 &  64.0 $\pm$ 5.6 & \textbf{90.7$\pm$2.9} \\
         antmaze-medium-play  &  0.0 &     0.0 &       10.6 &   61.2 &    35.3 $\pm$1.3 &           71.2 &  \textbf{75.3$\pm$2.5} & \textbf{76.3$\pm$2.5} \\
         antmaze-medium-diverse  & 0.0 &  0.7 &   3.0&    53.7 &  27.3 $\pm$3.9   &  \textbf{70.0} & \textbf{74.3 $\pm$ 3.7} & 69.3$\pm$3.3   \\
       antmaze-large-play  &  0.0 &   0.0 &    0.2& 15.8 &  20.2 $\pm$ 14.8  &   \textbf{39.6} &    \textbf{38.3 $\pm$ 4.8} & 16.3$\pm$11.1 \\
      antmaze-large-diverse &  0.0 &   1.0 &       0.0&  14.9 &     \textbf{41.2 $\pm$4.2}   &   \textbf{47.5} & 40.0 $\pm$ 7.8 & 41.0$\pm$10.7  \\
      \midrule
      Average & 16.7 & 17.95 & 27.3 & 50.6 & 42.3 & 63.0 & \textbf{64.0} & \textbf{65.1} \\
        \bottomrule
    \end{tabular}
    \end{adjustbox}
\end{table}

\subsection{Robustness of the Benchmark Results}
In Figure \ref{fig:benchmark_attack}, we evaluate the robustness of the benchmark results, i.e., how robust each algorithm is to maintain the performance listed in Table \ref{tab:gym_more}. We compare RORL with EDAC, SAC-10 on six tasks. EDAC is reproduced with 10 ensemble $Q$ networks as RORL and SAC-10, and uses $\eta=1$ for all six tasks. Note that in the benchmark experiments, RORL is only trained with small smoothing scales within $\{0.001, 0.005, 0.01\}$. The evaluation perturbation scales are within range $[0.00,0.05]$ and the results are averaged over 4 random seeds. From the results, we can conclude that RORL can successfully keep the highest performance within a certain perturbation scale and the performance of EDAC and SAC-10 decreases faster than RORL for most tasks and attack methods. The results imply that RORL has better practicability in real-world scenarios.








\begin{figure}[htb]
\centering
\subfigure[Halfcheetah-medium-v2]{\includegraphics[width=1\textwidth]{  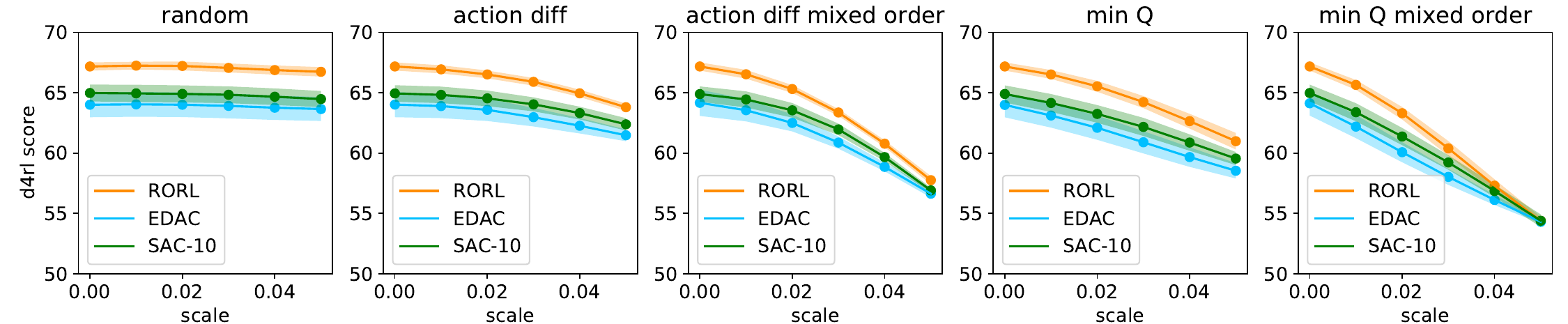}}

\subfigure[Halfcheetah-expert-v2]{\includegraphics[width=1\textwidth]{  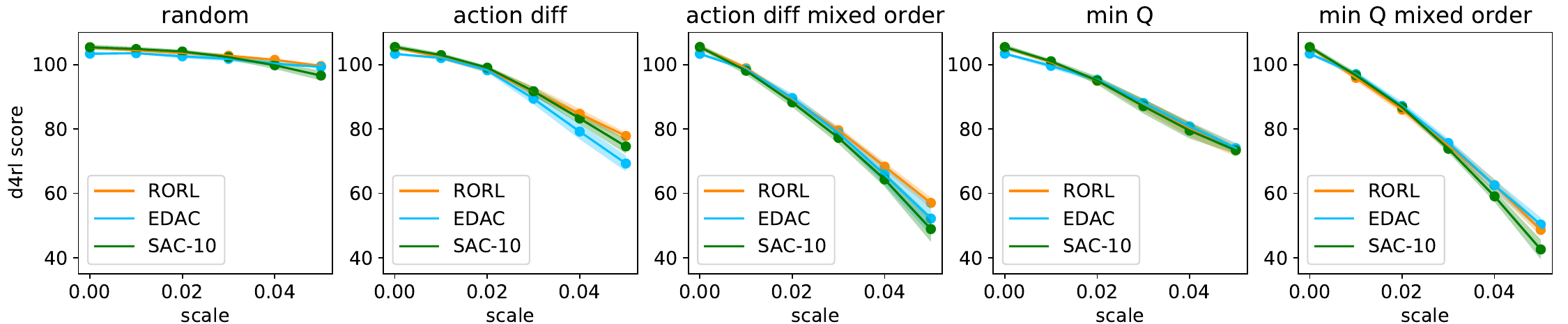}}

\subfigure[Hopper-medium-v2]{\includegraphics[width=1\textwidth]{  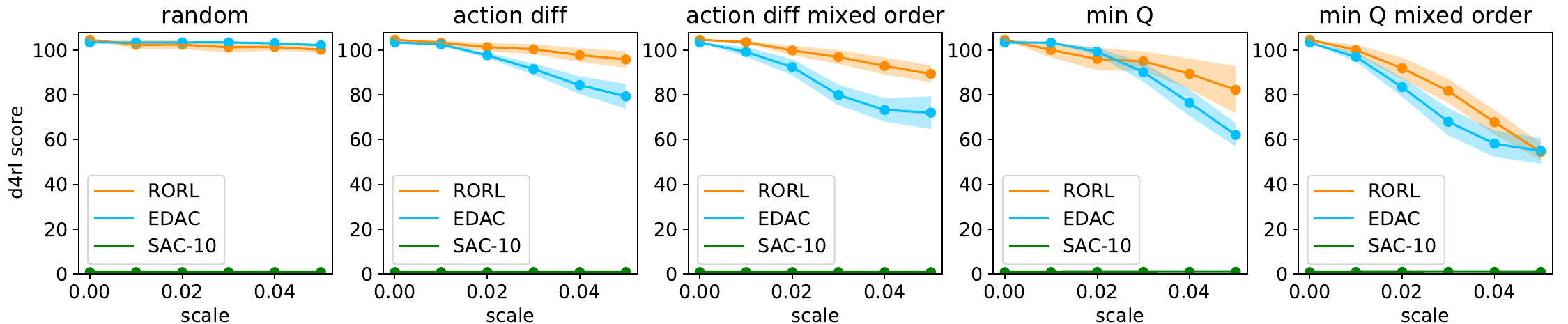}}

\subfigure[Hopper-expert-v2]{\includegraphics[width=1\textwidth]{  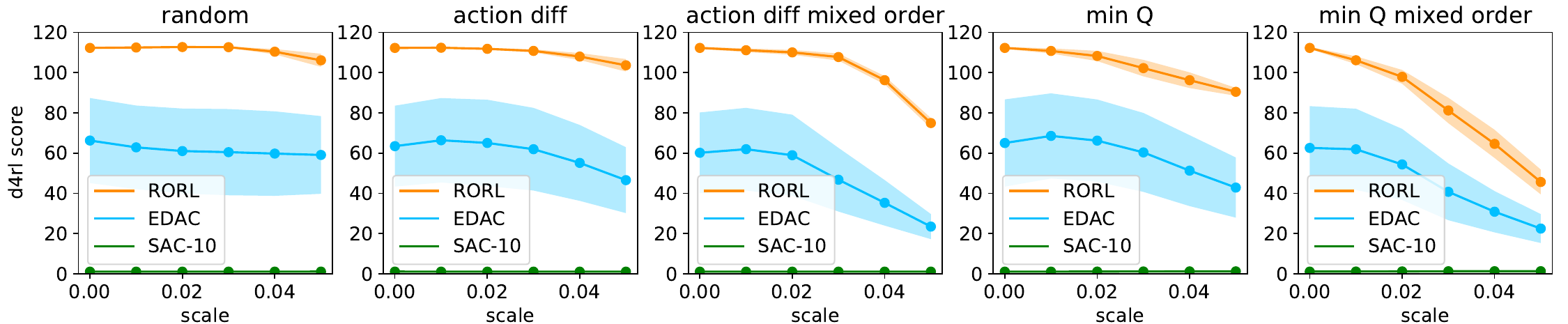}}

\subfigure[Walker2d-random-v2]{\includegraphics[width=1\textwidth]{  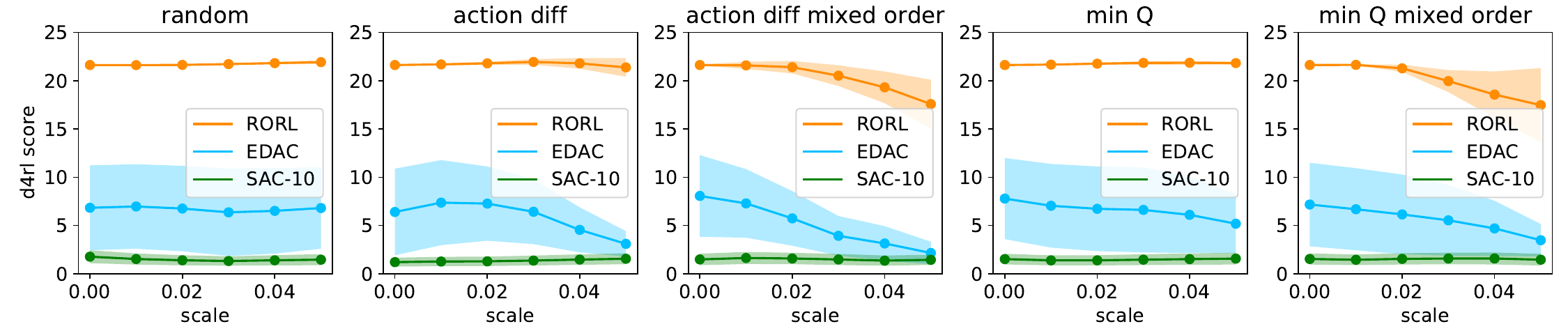}}

\subfigure[Walker2d-medium-v2]{\includegraphics[width=1\textwidth]{  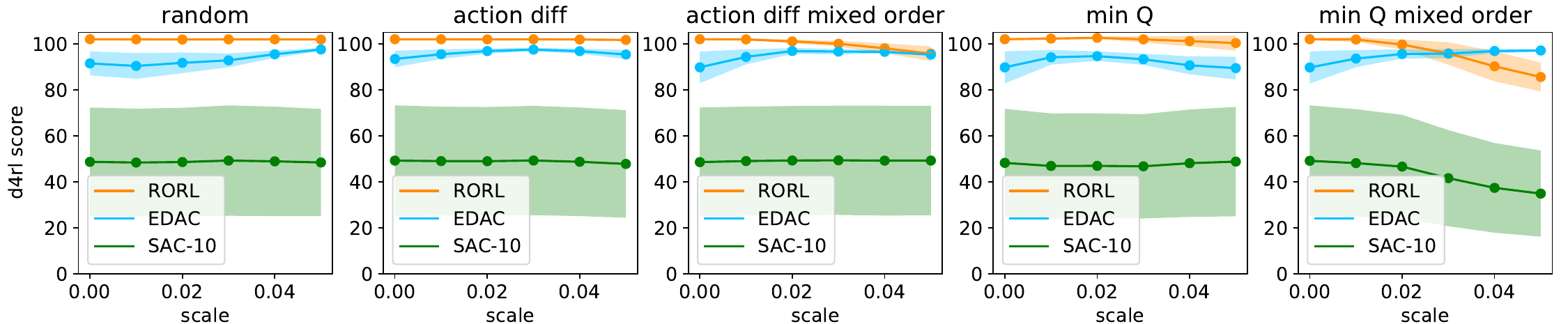}}

\caption{Performance under adversarial attack on six datasets. RORL can maintain the best performance in the benchmark experiments for small-scale perturbations.}
\label{fig:benchmark_attack}
\end{figure}

\section{Tips for Customizing RORL}
According to our ablation study result in Appendix \ref{ap:addtional_exp}, we summarize some tips for adapting RORL for customized use below.
\begin{itemize}
    \item \textbf{Hyper-parameter Tuning:} Since RORL is proposed to solve a challenging problem, it has many hyper-parameters. Our first suggestion is to use our hyper-parameter search range in Appendix \ref{ap-implementation_details}. You can tune them according to the importance of each component, where the general order is : OOD loss $>$ policy smoothing loss $>$ $Q$ smoothing loss. 

    \item \textbf{Computation Cost:} If you want less GPU memory usage and less training time, you can (1) set $\beta_{\rm Q}=0$ and $\epsilon_{\rm Q}=0$ because the $Q$ smoothing loss contributes the least but consumes a large computational cost, and (2) use a small number $n$ of sampled perturbed states to reduce the GPU memory usage. 
\end{itemize}

\newpage

\section{More Related Works}
\label{appendix:related-work}

\paragraph{Model-Based Offline RL} In offline RL, model-based methods use an empirical model learned from the offline dataset to enhance the generalization ability. The model can be used as the virtual environment for data collection~\cite{yu2020mopo,kidambi2020morel}, or to augment the dataset for an existing model-free algorithm~\cite{yu2021combo,wang2021offline}. The main challenges of model-based algorithms are how to learn the accurate empirical model and how to construct the uncertainty measure. A recent work~\cite{janner2021offline} demonstrates that the transformer model can generate realistic trajectories, which is beneficial for policy learning. In contrast, we focus on the model-free methods in this paper and leave the robustness of model-based methods in future work.

\paragraph{Adversarial Attack} Inspired by adversarial examples in deep learning~\cite{goodfellow2014explaining,papernot2016practical}, adversarial attack and policy  poisoning~\cite{behzadan2017vulnerability,huang2017adversarial,pattanaik2018robust} are studied to avoid adversarial manipulations on the network policies. Gleave et al.~\cite{gleave2019adversarial} study adversarial policy in the behavior level~\cite{gleave2019adversarial}. Data corruption~\cite{zhang2021robust,ma2019policy,wu2022copa} considers the case where an attacker can arbitrarily modify the dataset under a specific budget before training. While adversarial attack in RL is highly related to robust RL, they focus more on adversarial attacks compared to our robustness setting. More effective attack strategies for offline RL can facilitate learning more robust policies.

\end{document}